\documentclass{article}

\usepackage{arxiv_style}
\usepackage{times}
\usepackage{natbib}
\usepackage{mathtools}

\usepackage{comment}

\usepackage[utf8]{inputenc} 
\usepackage[T1]{fontenc}    
\usepackage{hyperref}       
\usepackage[capitalise]{cleveref}
\usepackage{url}            
\usepackage{booktabs}       
\usepackage{amsfonts}       
\usepackage{nicefrac}       
\usepackage{microtype}      

\usepackage{authblk}

\usepackage{xcolor}
\usepackage{comment}

\usepackage{amsmath,amssymb,amsthm}
\usepackage{subcaption}
\usepackage{graphicx}

\newtheorem{theorem}{Theorem}
\newtheorem{lemma}{Lemma}
\newtheorem{proposition}{Proposition}
\newtheorem{corollary}{Corollary}
\newtheorem{definition}{Definition}
\newtheorem{remark}[]{Remark}
\newtheorem{assumption}[]{Assumption}

\newcommand{\abs}[1]{|#1|}
\newcommand{\norm}[1]{||#1||}
\newcommand{\set}[1]{\{#1\}}
\newcommand{\N}{\mathbb{N}}
\newcommand{\Var}{\operatorname{Var}}
\newcommand{\Ent}{\operatorname{Ent}}
\newcommand{\Renyi}{\mathcal{R}}
\newcommand{\R}{\mathbb{R}}
\newcommand{\D}{\mathcal{D}}
\newcommand{\KL}{\operatorname{KL}}
\newcommand{\E}{\mathbb{E}}

\renewcommand{\P}{\mathbb{P}}

\title{Sampling Multimodal Distributions with the Vanilla Score: Benefits of Data-Based Initialization}

\author[1]{Frederic Koehler}
\author[1]{Thuy-Duong Vuong}
\affil[1]{Stanford University, \url{{fkoehler,tdvuong}@stanford.edu}}
\begin{document}

\maketitle

\begin{abstract}
There is a long history, as well as a recent explosion of interest, in statistical and generative modeling approaches based on \emph{score functions} --- derivatives of the log-likelihood of a distribution. In seminal works, Hyv\"arinen proposed vanilla score matching as a way to learn distributions from data by computing an estimate of the score function of the underlying ground truth, and established connections between this method and established techniques like Contrastive Divergence and Pseudolikelihood estimation. It is by now well-known that vanilla score matching has significant difficulties learning multimodal distributions. Although there are various ways to overcome this difficulty, the following question has remained unanswered --- is there a natural way to sample multimodal distributions using just the vanilla score? Inspired by a long line of related experimental works, we prove that the Langevin diffusion with early stopping, initialized at the empirical distribution, and run on a score function estimated from data successfully generates natural multimodal distributions (mixtures of log-concave distributions). 
\end{abstract}

\section{Introduction}

Score matching is a fundamental approach to generative modeling which proceeds by attempting to learn the gradient of the log-likelihood of the ground truth distribution from samples (``score function'') \cite{hyvarinen2005estimation}.
This is an elegant approach to learning \emph{energy-based models} from data, since it circumvents the need to compute the (potentially intractable) partition function which arises in  Maximum Likelihood Estimation (MLE). 
Besides the original version of the score matching method (often referred to as \emph{vanilla score matching}), many variants have been proposed and have seen dramatic experimental success in generative modeling, especially in the visual domain (see e.g.\ \cite{song2019generative,song2020score,rombach2022high}). 

In this work, we revisit the vanilla score matching approach. It is known that learning a distribution via vanilla score matching generally fails in the multimodal setting \citep{wenliang2019learning,song2019generative,koehler2022statistical}. However, there are also many positive aspects of modeling a distribution with the vanilla score. To name a few:
\begin{enumerate}
    \item Simplicity to fit: computing the best estimate to the vanilla score is easy in many situations. For example, there is a simple closed form solution the class of models being fit is an exponential family \citep{hyvarinen2007some}, and this in turn lets us compute the best fit in a kernel exponential family  (see e.g.\ \cite{sriperumbudur2017density,wenliang2019learning}).
    \item Compatibility with energy-based models: for a distribution $p(x) \propto \exp(E(x))$, the vanilla score function is $\nabla E(x)$ so it is straightforward to go between the energy and the score function. This is related to the previous point (why exponential families are simple to score match), and also why it is easy to implement the Langevin chain for sampling an energy-based model. 
    \item Statistical inference: in cases where vanilla score matching does work well, it comes with attractive statistical features like $\sqrt{n}$-consistency, asymptotic normality, relative efficiency guarantees compared to the MLE, etc. --- see e.g. \cite{barp2019minimum,forbes2015linear,koehler2022statistical,song2020sliced}.
\end{enumerate}
In addition, score matching is also closely related to other celebrated methods for fitting distributions which have been successfully used for a long time in statistics and machine learning --- pseudolikelihood estimation \citep{besag1975statistical} and contrastive divergence training \citep{hinton2002training}. (See e.g.\ \cite{hyvarinen2007connections,koehler2022statistical}.) 

For these reasons, we would like to better understand the apparent failure of score matching in the multimodal setting. In this work, we study score matching in the context of the most canonical family of multimodal distributions --- mixtures of log-concave distributions. (As a reminder, any distribution can be approximated by a sufficiently large mixture, see e.g.\ \cite{wasserman2006all}.) While vanilla score matching itself does not correctly estimate these distributions, we show that the trick of using ``data-based initialization'' when sampling, which is well-known in the context of CD/MLE training of energy based models (see e.g.\ \cite{hinton2012practical,xie2016theory} and further references below), provably corrects the bias of any model which accurately score matches the ground truth distribution. 

\subsection{Our Results} 
We now state our results in full detail. We are interested in the question of generative modeling using the vanilla score function. Generally speaking, there is some ground truth distribution $\mu$, which for us we will assume is a mixture of log-concave distributions, and we are interested in outputing a good estimate $\hat \mu$ of it. We show that this is possible provided access to:
\begin{enumerate}
    \item A good estimate of the score function of $\nabla \log \mu$. (In many applications, this would be learned from data using a procedure like score matching.)
    \item A small number of additional samples from $\mu$, which are used for data-based initialization. 
\end{enumerate}
To make the above points precise, the following is our model assumption on $\mu$:
\begin{assumption}\label{ass:mixture-log-concave}
We assume probability distribution $\mu$ is a mixture of $K$  log-concave components: explicitly, $\mu =\sum_{i = 1}^K p_i \mu_i$ for some weights $p_1,\dots, p_{K} $ s.t. $p_i > 0$ and $\sum_i p_i = 1.$ 
Furthermore, we suppose the density of each component $ \mu_i$ is $\alpha$ strongly-log-concave and $\beta$-smooth with $\beta \geq 1$\footnote{We can always re-scale the domain so that $\beta \geq 1.$} 
i.e. $ \alpha I \preceq -\nabla^2 \log \mu_i(x) \preceq \beta I$ for all $x$.     
We define the notation $p_* = \min_i p_i$ and $\kappa = \beta/\alpha \geq 1.$
\end{assumption}
\begin{remark}
The assumption that $\mu_i$ is $\alpha$-strongly log-concave and $\beta$-smooth is the most standard setting where the Langevin dynamics are guaranteed to mix rapidly (see e.g.\ \cite{dalalyan2017theoretical}).     
\end{remark}
and the following captures formally what we mean by a ``good estimate'' of the score function:
\begin{definition}\label{def:eps-score}
For $\mu$ a probability distribution with smooth density $\mu(x)$, an $\epsilon_{\text{score}}$-accurate estimate of the score in $L_2(\mu)$ is a function $s$ such that
\begin{equation}\label{eqn:L_2}
\mathbb{E}_{x\sim\mu}[||s(x) - \nabla \log \mu(x)||^2 ]\leq \epsilon_{\text{score}}^2.
\end{equation}
\end{definition}
As discussed in the below remark, this is the standard and appropriate assumption to make when score functions are learned from data. There are also other settings of interest where the ground truth score function is known exactly (e.g. $\mu$ is an explicit energy-based model which we have access to, and we want to generate more samples from it\footnote{For example, one use case of generative modeling is when we have the ground truth and want to accelerate an existing sampler which is expensive to run, see e.g. \cite{albergo2021introduction,lawrence2021normalizing}.}) in which case we can simply take $\epsilon_{\text{score}} = 0$. 
\begin{remark}
Assumption \eqref{eqn:L_2} says that on average over a fresh sample from the distribution, $s(x)$ is a good estimate of the true score function $\nabla \log \mu(x)$. This is the right assumption when score functions are estimated from data, because it is generally impossible to learn the score function far from the support of the true distribution. See the previous work e.g.\ \cite{chen2023sampling,lee2022convergence,lee2022convergence2,block2020generative} where the same distinction is  discussed in more detail. 

Given a class of functions which contains a good model for the true score function and has a small Rademacher complexity compared to the number of samples, the function output by vanilla score matching will achieve small $L_2$ error (see proof of Theorem 1 of \cite{koehler2022statistical}). In particular, this can be straightforwardly applied to parametric families of distributions like mixtures of Gaussians. We would also generally expect this assumption to be satisfied when the distribution is successfully learned via other learning procedures, such as MLE/contrastive divergence. (See related simulation in \cref{sec:additional-details}.) 
\end{remark}

We show the distribution output by Langevin dynamics on an approximate score function will be close to the ground truth provided (1) we initialize the Langevin diffusion from the empirical distribution of samples, and (2) we perform early stopping of the diffusion, so that it does not reach its stationary distribution. 
Formally, 
let the \emph{Langevin Monte Carlo} (LMC, a.k.a. discrete-time Langevin dynamics) chain with initial state $X_0$, score function $s$, and step size $h > 0$  be defined by the recursion
\[ X_{h(i + 1)} = X_{hi} + h\, s(X_{hi}) + \sqrt{2h}\,\Delta_{hi}  \]
where each noise variable $\Delta_{hi} \sim N(0,I)$ is independent of the previous ones. 
Our main result gives a guarantee for samplling with LMC started from a small set of samples and run for time $T$:
\begin{theorem}\label{thm:main}
Let $\epsilon_{TV} \in (0,1/2). $  Suppose $\mu$ is a mixture of strongly log-concave measures as in Assumption~\ref{ass:mixture-log-concave} and $s$ is a function which estimates the score of $\mu$ within $L_2$ error $\epsilon_{\text{score}}$ in the sense of Definition~\ref{def:eps-score}. 
Let
\begin{align*}
 T = \Tilde{\Theta} \left(\left(\frac{\exp(K) d \kappa}{p_* \epsilon_{TV}}\right)^{O_K(1)} \right), \qquad\qquad h = \tilde{\Theta}\left(\frac{\epsilon_{TV}^4}{(\beta\kappa^2 K \exp(K) )^4  d^3 T  }\right).
\end{align*}

Let $U_{\text{sample}}$ be a set of $M$ i.i.d. samples from $\mu$ and $ \nu_{\text{sample}}$ be the uniform distribution over $U_{\text{sample}}.$ 
Suppose that $M = \Omega(p_*^{-2} \epsilon_{TV}^{-4} K^4 \log (K /\epsilon_{TV})  \log(K/\tau))  ,$
 and that \[ \epsilon_{\text{score}} \leq \frac{p_*^{1/2} \sqrt{h} \epsilon_{TV}^2}{ 7 T }= \tilde{\Theta}\left(\frac{p_*^{1/2} \epsilon_{TV}^4 }{(\beta \kappa^2 K \exp(K) )^2 d^{3/2} T^{3/2}  }\right).   \]
 Let $ (X_{nh}^{\nu_{\text{sample}}})_{n \in \N}$ be the LMC chain with score $s$ and step size $h$ initialized at $ \nu_{\text{sample}}.$ 
 Then with probability at least $1 - \tau$ over the randomness of ${U_{\text{sample}}},$ the conditional law $\hat \mu = \mathcal{L}(X_{T}^{\nu_{\text{sample}}} \mid U_{\text{sample}}) $ 
 satisfies
 \begin{equation}\label{eqn:learned-distribution} d_{TV} (\hat \mu, \mu) \leq \epsilon_{TV}. 
 \end{equation}
\end{theorem}

We now make a few comments to discuss the meaning of the result.
 Conclusion \eqref{eqn:learned-distribution} says that we have successfully found an $\epsilon_{TV}$-close approximation of the ground truth distribution $\mu$. Unpacking the definitions, it says that with high probability over the sample set: (1) picking a uniform sample from the training set, and (2) running the Langevin chain for time $T$ will generate an $\epsilon_{TV}$-approximate sample from the distribution $\mu$. Note in particular that we have can draw as many samples as we like from the distribution without needing new training data. The fact that this is \emph{conditional on the dataset} is a key distinction: the \emph{marginal} law of any element of the training set would be $\mu$, but its \emph{conditional} law is a delta-distribution at that training sample, and the conditional law is what is relevant for generative modeling (being able to draw new samples from the right distribution). See also Figure~\ref{fig:one-dimensional} for a simulation which helps illustrate this distinction. 

\begin{remark}
Provided the number of components in the mixture is $O(1)$, i.e. upper bounded by a constant, the dependence on all other parameters is polynomial or logarithmic. It is possible to remove the dependence on the minimum weight $p_*$ completely --- see \cref{cor:mixing of discrete chain with score error modified} in Appendix~\ref{sec:remove minimum weight assumption}. 
\end{remark}

\begin{remark}
It turns out Theorem~\ref{thm:main} is a new result even in the very special case that the ground truth is unimodal. The closest prior work is Theorem 2.1 of \cite{lee2022convergence}, where it was proved that the Langevin diffusion computed using an approximate score function succeeds to approximately sample from the correct distribution given a (polynomially-)warm start in the $\chi_2^2$-divergence. However, while the empirical distribution of samples is a natural candidate for a warm start, in high dimensions it will not be anywhere close to the ground truth distribution unless we have an exponentially large (in the dimension) number of samples, due to the ``curse of dimensionality'', see e.g.\ \cite{wasserman2006all}.
\end{remark}

\subsection{Further Discussion}

\paragraph{One motivation: computing score functions at substantial noise levels can be computationally difficult.} 
In some cases, computing/learning the vanilla score may be a substantially easier task than alternatives; for example, compared to learning the score function for all noised versions of the ground truth (as used in diffusion models like \cite{song2019generative}).
As a reminder, denoising diffusion models are based on the observation that the score function of a noised distribution $N(0,\sigma^2 I) \star p$ exactly corresponds to a Bayesian denoising problem: computing the posterior mean on $X \sim p$ given a noisy observation $Y \sim N(x,\sigma^2 I)$ \cite{vincent2011connection,block2020generative}, via the equation
\[ y + \sigma^2 \nabla \log (N(0,\sigma^2 I) \star p)(y) = \mathbb E[X \mid Y = y]. \]
Unlike the vanilla score function this will not be closed form for most energy-based models; the optimal denoiser might be complex when the signal is immersed in a substantive amount of noise.

For example, results in the area of computational-statistical gaps tell us that for certain values of the noise level $\sigma$ and relatively simple distributions $p$, approximate denoising can be average-case computationally hard under widely-believed conjectures. For example, let $p$ be a distribution over matrices of the form $N(rr^T, \epsilon^2)$ with $r$ a random sparse vector and $\epsilon > 0$ small. 
Then the denoising problem for this distribution will be  ``estimation in the sparse spiked Wigner model''. In this model, for a certain range of noise levels $\sigma$ performing optimal denoising is as hard as the (conjecturally intractible) ``Planted Clique'' problem \citep{brennan2018reducibility}; in fact, even distinguishing this model from a pure noise model with $r = 0$ is computationally hard despite the fact it is statistically possible --- see the reference for
details. 
So unless the Planted Clique conjecture is false, there is no hope of approximately computing the score function of $p \star N(0,\sigma^2)$ for these values of $\sigma$. On the other hand, there is no computational obstacle to computing the score of $p$ itself provided $\epsilon > 0$ is small --- denoising is only tricky once the noise level becomes sufficiently large. 


\paragraph{Related Experimental Work.} As mentioned before, many experimental works have found success generating samples, especially of images, by running the Langevin diffusion (or other Markov chain) for a small amount of time. One aspect which varies in these works is how the diffusion is initialized. To use the terminology of \cite{nijkamp2020anatomy}, the method we study uses an \emph{informative/data-based initialization} similar to contrastive divergence \cite{hinton2012practical,gao2018learning,xie2016theory}. While in CD the early stopping of the dynamics is usually motivated as a way to save computational resources, the idea that stopping the sampler early can improve the quality of samples is consistent with experimental findings in the literature on energy-based models. As the authors of \cite{nijkamp2020anatomy} say, ``it is much harder to train
a ConvNet potential to learn a steady-state over realistic images. To our knowledge, long-run MCMC samples of all previous models lose the realism of short-run samples.'' One possible intuition for the benefit of early stopping, consistent with our analysis and simulations, is that it reduces the risk of stepping into low-probability regions where the score function may be poorly estimated. Some  works have also found success using random/uninformative initializations with appropriate tweaks \citep{nijkamp2019learning,nijkamp2020anatomy}, although they still found informative initialization to have some advantages --- for example in terms of output quality after larger numbers of MCMC steps.  


\paragraph{Related Theoretical Work.} 
The works \cite{block2020generative,lee2022convergence} established results for learning unimodal distributions (in the sense of being strongly log-concave or satisfying a log-Sobolev inequality) via score matching, provided the score functions are estimated in an $L_2$ sense. The work \cite{koehler2022statistical} showed that the sample complexity of  vanilla score matching is related to the size of a restricted version of the log-Sobolev constant of the distribution, and in particular proved negative results for vanilla score matching in many multimodal settings. The works \cite{lee2022convergence2,chen2023sampling} proved that even for multimodal distributions, \emph{annealed} score matching will successfully learn the distribution provided all of the annealed score functions can be successfully estimated in $L_2$. In our work we only assume access to a good estimate of the vanilla score function, but still successfully learn the ground truth distribution in a multimodal setting. 

In the sampling literature, our result can be thought of establishing a type of \emph{metastability} statement, where the dynamics become trapped in local minima for moderate amounts of time --- see e.g. \cite{tzen2018local} for further background. Also in the sampling context,
the works \cite{lee2018,ge2018simulated} studied a related  problem, where the goal is to sample a mixture of isotropic Gaussians given black-box access to the score function (which they do via simulated tempering). 
This problem ends up to be different to the ones arising in score matching: they need exact knowledge of the true score function (far away from the support of the distribution), but they do not have access to training data from the true distribution. As a consequence of the differing setup, they prove an impossibility result \cite[Theorem F.1]{ge2018simulated} for a mixture of two Gaussians with covariances $I$ and $2I$   (it will not be possible to find both components), but our result proves this is not an issue in our setting.

\paragraph{Questions for future work.} 
In our result, we proved the first bound for sampling with the vanilla score, estimated from data, which succeeds in the multimodal setting, but it is an open question if the dependence on the number of components is optimal; it seems likely that the dependence can be improved, at least in many cases. 
Finally, it is interesting to ask what the largest class of distributions our result can generalize to --- with data-based initialization, multimodality itself is no longer an obstruction to sampling with Langevin from estimated gradients, but are there other possible obstructions?
\section{Technical Overview}
We first review some background and notation which is helpful for discussing the proof sketch. We leave complete proofs of all results to the appendices.

\paragraph{Notation.} We use standard big-Oh notation and use tildes, e.g. $\tilde{O}(\cdot)$, to denote inequality up to log factors and $O_B(\cdot)$ to denote an inequality with a constant allowed to depend on $B$. 
We let $d_{TV}(\mu,\nu) = \sup_{A} |\mu(A) - \nu(A)|$ be the usual total variation distance between probability measures $\mu$ and $\nu$ defined on the same space, where the supremum ranges over measurable sets. Given a random variable $X$, we write $\mathcal L(X)$ to denote its law.

\paragraph{Log-Sobolev inequality.} 
We say probability distribution $\pi$ satisfies a log-Sobolev inequality (LSI) with constant $C_{LS}$ if  for all smooth functions $f$,
 $\E_{\pi} [f^2 \log(f^2/\E_{\pi}[f^2])] \leq 2 C_{LS} \E_{\pi}[ \norm{\nabla f}^2 ]$.
Due to the Bakry-Emery criterion, if $\pi$ is $\alpha$-strongly log-concave then $ \pi $ satisfies LSI with constant $C_{LS}=1/\alpha.$ LSI is equivalent to a statement about mixing of the Langevin dynamics --- if we let $\pi_t$ denote the law of the diffusion at time $t$ then an LSI is equivalent to the inequality
\[ \D_{\KL} (\pi_t || \pi) \leq \exp(-2t/C_{LS}) \D_{\KL}(\pi_0 || \pi) \]
holding for an arbitrary initial distribution $\pi_0$.
Here $\D_{KL}(P,Q) = \E_P[\log \frac{dP}{dQ}]$ is the Kullback-Liebler divergence.
See \cite{bakry2014analysis,van2014probability} for more background.

\paragraph{Stochastic calculus.} We will need to use stochastic calculus to compare the behavior of similar diffusion processes ---  see \cite{karatzas1991brownian} for formal background.
Let $ (X_t)_{t\geq 0}$ and $(Y_t)_{t\geq 0}$ be two Ito processes defined by SDEs: $dX_t = s_1(X_t)dt + dB_t $ and $d Y_t = s_2(X_t) dt + dB_t.$
Let $P_T, Q_T$ be the laws of the paths $(X_t)_{t\in [0,T]}$ and $(Y_t)_{t\in [0,T]}$ respectively. The following follows by Girsanov's theorem
(see \cite[Eq. (5.5) and Theorem 9]{chen2023sampling})
\[d_{TV} (Y_{T}, X_T)^2 \leq d_{TV}(Q_T, P_T)^2 \leq \frac{1}{2} \E_{Q_T} \left[\int_0^T \norm{s_2(Y_t) - s_1(Y_t) }^2 dt\right] \]
In particular, this is useful to compare continuous and discrete time Langevin diffusions. If $ (Y_t)$ be the continuous Langevin diffusion with score function $s$, and $ (X_t)$ is a linearly interpolated version of the discrete-time Langevin dynamics defined by $ dX_t = s(X_{\lceil t/h \rceil h}) dt + dB_t,$  then
\begin{equation}d_{TV} (Y_{T}, X_T)^2 \leq  \frac{1}{2} \E_{Q_T} \left[\int_0^T \norm{s(Y_t) - s(Y_{\lceil t/h \rceil h}) }^2 dt\right] \label{eqn:girsanov-consequence}
\end{equation}
\subsection{Proof sketch}
\paragraph{High-level discussion.} 
At a high level, our argument proceeds by (1) group the components of the mixture into larger ``well-connected'' pieces, and (2) showing that the process mixes well within each of these pieces, while preserving the correct relative weight of each piece. 
One of the challenges in proving our result is that, contrary to the usual situation in the analysis of Markov chains (as in e.g. \cite{bakry2014analysis,levin2017markov}), we \emph{do not} want to run the Langevin diffusion until it mixes to its stationary distributions. If we ran the process until mixing, then we would be performing the vanilla score matching procedure which provably fails in most multimodal settings because it incorrectly weights the different components \citep{koehler2022statistical}. 
So what we want to do is prove the process succeeds at some intermediate time $T$ (See Figure~\ref{fig:one-dimensional} for a simulation illustrating this.)

To build intuition, consider the special case where
all of the components in the mixture distributions are very far from each other. In this case, one might guess that taking $T$ to be the maximum of the mixing times of each of the individual components will work. Provided there are enough samples in the dataset, the initialization distribution will accurately model the relative weights of the different clusters in the data, and running the process up to time $T$ will approximately sample from the cluster that the initialization is drawn from. We could hope to prove the result by arguing that the dynamics on the mixture is close to the dynamics on one of the mixture components. 

\paragraph{Some challenges to overcome in the analysis.} This is the right intuition, but for the general case the behavior of the dynamics is more complicated. When components are close, the score function of the mixture distribution may not be close to the score function of either component in the region of overlap; relatedly, particles may cross over between components. 
Also, the following remark shows that natural variants of our main theorem are actually false. 

\begin{remark}
We might think that initializing from the \emph{center} of each mixture component would work just as well as initializing from samples. This is fine if the clusters are all very far from each other, but wrong in general.
If the underlying mixture distribution is $\frac{1}{2} N(0,I_d) + \frac{1}{2} N(0,2I_d)$ and the dimension $d$ is large, then the first component will have almost all of its mass within distance $O(1)$ of a sphere of radius $\sqrt{d}$ and the second component will similarly concentrate about a sphere of radius $\sqrt{2d}$. (See Theorem 3.1.1 of \cite{vershynin2018high}.) As a consequence, the dynamics initialized at the origin will mix within the shell of radius $\sqrt{d}$ but take $\exp(\Omega(d))$ time to cross to the larger $\sqrt{2d}$ shell. 
(This can be proved by observing that the gap between the two spheres forms a ``bottleneck'' for the dynamics, see \cite{levin2017markov}.) 
In contrast, if we initialize from samples then approximately half of them will lie on the outer shell and, as we prove, the dynamics mix correctly. 
\end{remark}

We now proceed to explain in more detail how we prove our result. We start with the analysis of an idealized diffusion process, and then through several comparison arguments establish the result for the real LMC algorithm. 

\paragraph{Analysis of idealized diffusion.}
To start out, we analyze an idealized process in which:
\begin{enumerate}
    \item The score function $\nabla \log \mu$ is known exactly. (Our result is still new  in this case.)
    \item The dynamics is the \emph{continous-time} Langevin diffusion given by the Ito process
    \[ d\bar X_t = \nabla \log \mu(\bar X_t)\, dt + \sqrt{2}\, dB_t. \]   
    This is the scaling limit of the discrete-time LMC chain as we take the step size $h \to 0$, where $dB_t$ is the differential of a Brownian motion $B_t$.
    \item For purposes of exposition, we make the fictitious assumption that the ground truth distribution $\mu$ is supported in a ball of radius $R$. This will not be literally true, but for sufficiently large $R$ $\mu$ will be almost entirely contained within a radius $R$ ball. (In the supplement, we handle this rigorously using concentration, see e.g. 
    proof of \cref{lem:well separated cluster} of \cref{sec:continuous}).
\end{enumerate}

Additionally, for the purpose of illustration, in this proof sketch we assume the target distance in TV is $0.01$ and consider the case where there are two $\alpha$-strongly log concave and $\beta$-smooth components $\mu_1$ and $\mu_2$, and $\mu = \frac{1}{2}\mu_1 + \frac{1}{2}\mu_2.$ After we complete the proof sketch for this setting, we will go back and explain how to generalize the analysis to arbitrary mixtures, handle the error induced by discretization, 
and finally make the analysis work with an $L_2$ estimate of the true score function.

\emph{Overlap parameter.} We define
\[ \delta_{12} := 1 - d_{TV} (\mu_1, \mu_2) =\int\min \set{\mu_1(x), \mu_2(x)} dx  \]
as a quantitative measure of how much components $1$ and $2$ overlap; for example, $\delta_{12} = 1$ iff $\mu_1$ and $\mu_2$ are identical. The analysis splits into cases depending on whether $\delta_{12}$ is large; we let $\delta > 0$ be a parameter which determines this split and which will be optimized at the end.

\emph{High overlap case (\cref{sec:mixture-sobolev}).} 
If $\mu_1$ and $\mu_2$ has high overlap, 
in the sense that $\delta_{1 2} \geq \delta$,
then we show that $\mu$ satisfies a log Sobolev inequality with constant at most $O(1/(\alpha \delta))$, by applying our Theorem \ref{thm:log sobolev and poincare for mixture}, an important technical ingredient which is discussed in more detail below. 
Thus for a typical sample $x$ from $\mu$, the continuous Langevin diffusion $(X_{t}^{\delta_x})_{t\geq 0}$ with score function $\nabla \log \mu$ initialized at $x$ converges to $ \mu$ i.e. $d_{TV}(\mathcal{L}(\bar{X}_t^{\delta_x}), \mu) \leq \epsilon$ for $T \geq \Omega(\frac{1}{\alpha \delta} \log (d\epsilon^{-1}))$.\footnote{This follows as LSI yields exponential convergence in KL-divergence. While the KL-divergence of the initialization $\delta_x$ with respect to $\mu$ is unbounded, we can bound the KL-divergence of $\bar{X}_h^{\delta_x}$ for some small $h.$} 
 
\emph{Low overlap case (\cref{sec:continuous}, \cref{lem:well separated cluster}).} 
When $\mu_1$ and $\mu_2$ have small overlap i.e. $\delta_{12} \leq \delta$, we will show that for $x\sim\mu,$  with high probability, 
the gradient of the log-likelihood of the mixture distribution $\mu$ at $x$ is close to that of one of the components $\mu_1, \mu_2$ \emph{(\cref{subsec:gradient error bound continuous})}. This is because, supposing that $\norm{x}\leq R$,
for $i\in 
\set{1,2}$ we can upper bound \[\norm{\nabla \log \mu (x) -\nabla \log \mu_i(x) }\leq 2\beta R \left(1-\frac{\mu_i(x)
}{\mu_1(x)+\mu_2(x)}\right),\] and low overlap implies that $\min_i \left(1-\frac{\mu_i(x)
}{\mu_1(x)+\mu_2(x)}\right)$ is small for \emph{typical} $x\sim \mu$. 

Consider the continuous Langevin diffusion $(\bar{X}_t^{\delta_x})$ initialized at $\delta_x$ i.e. $\bar{X}_0=x.$ Observe that the \emph{marginal} law of $\bar{X}_t^{\delta_x} $ where $x\sim \mu$ is 
exactly $\mu$, since $\mu$ is the stationary distribution of the Langevin diffusion. 
Let $H > 0$ be a parameter to be tuned later.
The above discussion and Markov's inequality allows us to argue that for a typical sample $x,$ the gradient of the log-likelihood of $\mu$ at $\bar{X}_{nH}^{\delta_x}$ is close to that of either components $\mu_1, \mu_2$ with high probability. 

Next, we perform a union bound over $n\in \set{0, \cdots, N-1}$ and bound the drift $ \norm{\nabla \log \mu (x) -\nabla \log \mu_i(x) }$ in each small time interval $[nH, (n+1) H]$. By doing so,
we can argue that for a typical sample $x \sim \mu$, with probability at least $ 1- \epsilon^{-1} \beta RN\delta_{12}$ over the randomness of the Brownian motion driving the Langevin diffusion,
the gradient of the log-likelihood at $\bar{X}_t^{\delta_x}$ for $t\in [0, NH]$ is close to that of the component distribution $\mu_i$ closest to the initial point $x$ (see \cref{prop:error bound for continuous process} of \cref{sec:continuous}). 

In other words, assuming that the initial point $x$ satisfies
$\mu_1(x)\geq \mu_2(x)$ and letting $T=NH$, we can show that with high probability,
\[ \sup_{t\in [0,T]} \norm{\nabla \log \mu (\bar{X}_t^{\delta_x} ) - \nabla \log \mu_1(\bar{X}_t^{\delta_x})} \leq 1.1 \epsilon.  \]
This allows us, using \eqref{eqn:girsanov-consequence}, to compare our Langevin diffusion with the one with score function $\nabla \log \mu_1$ and show the output at time $T$ is approximately a sample from $\mu_1$.  

In a typical set $U_{\text{sample}}$ of i.i.d. samples from $\mu,$ roughly $50\%$ of the samples $x\in U_{\text{sample}}$ satisfy $\mu_1(x)\geq \mu_2(x)$ and the other $50\%$ samples  satisfy $\mu_2(x)\geq \mu_1(x),$ thus the Langevin dynamics  $(\bar{X}_t^{\nu_{\text{sample}}})_{t\geq 0} $ initialized at the uniform distribution $\nu_{\text{sample}}$ over $U_{\text{sample}}$ will be close to $\frac{\mu_1 + \mu_2}{2} = \mu$ after time $T$ provided we set
 $H, T,\epsilon, \delta$ appropriately.

\emph{Concluding the idealized analysis.}
Either $\delta_{12} \ge \delta$ in which case the high-overlap analysis above based on the log-Sobolev constant succeeds, or $\delta_{12} < \delta$ in which case the low-overlap analysis succeeds. 
Optimizing over $\delta$, we find that in either case, with high probability over the set $U_{\text{sample}}$ of samples from $\mu$,  for $ t\geq \tilde{\Omega}(\frac{(\beta R)^3 }{\alpha^{5/2} })$  
we have
\[ d_{TV}(\mathcal{L}(\bar{X}_t^{\nu_{\text{sample}}} \mid U_{\text{sample}}), \mu)\leq 0.01 \]
as desired.

\paragraph{Generalizing idealized analysis to arbitrary mixtures.} (\emph{\cref{sec:continuous}, \cref{thm:continuous mixing}})
When there are more than two components, we can generalize this analysis --- the key technical difficulty, alluded to earlier, is analyzing the overlap between different mixture components. We do this by defining, for each $\delta > 0$, a graph $\mathbb{G}^{\delta}$ where there is an edge between $i,j\in [K]$ when $ \delta_{ij} := 1 - d_{TV}(\mu_i,\mu_j) \leq \delta.$ As long as the minimum of the weights $p_*:=\min_i p_i$  is not too small, each connected component $C$ of $ \mathbb{G}^{\delta}$ is associated with a probability distribution $\mu_C = \frac{\sum_{i\in C } p_i \mu_i}{\sum_{i\in C} p_i}$ that has log Sobolev constant on the order of $ O_{K, p_*^{-1} } (1/\alpha \delta).$ 

Suppose for a moment that the connected components are well separated compared to the magnitude of $\delta$. More precisely, suppose that for $i,j$ in different connected components and some $\delta > 0$ we have
\begin{equation} \label{eq:delta induction} \delta_{ij}\leq  f(\delta) := \Theta\left(\frac{(\alpha \delta)^{3/2} }{(\beta R)^3 }\right).
\end{equation} 
Then, a direct generalization of the argument for two components shows that for a typical set $U_{\text{sample}}$ of i.i.d. samples from $\mu$, the continuous Langevin diffusion $(\bar{X}_t^{\nu_{\text{sample}}})_{t\geq 0}$ initialized at the uniform distribution over $U_{\text{sample}}$ converges to $\mu$ after time $T_{\delta}= (\alpha \delta)^{-1}.$ 

It remains to discuss how we select $\delta$ so that \eqref{eq:delta induction} is satisfied. 
We consider a decreasing sequence $1=\delta_0 > \delta_1 >\cdots > \delta_{K-1}$ where $\delta_{r+1} =f(\delta_r)$ as in Eq.~\eqref{eq:delta induction}.  
Let $\mathbb{G}^{r}:=\mathbb{G}^{\delta_r}.$ If any two vertices from different connected components of $\mathbb{G}^{r}$ have overlap at most $\delta_{r+1},$ then the above argument applies. Otherwise, $\mathbb{G}^{r+1}$ must have one less connected component than $\mathbb{G}^{r},$ and since $\mathbb{G}^{0}$ has at most $K$ connected components,  $\mathbb{G}^{K-1}$ must have 1 connected component and the above argument applies to it.
Thus, in all cases, the distribution of $\bar{X}_{T_{\delta_{K-1}}}^{\nu_{\text{sample}}}$ is close to $\mu$ in total variation distance.
\paragraph{Discretization analysis.} \emph{(\cref{sec:discrete}, \cref{lem:lmc with linfity error})} We now move from a continuous-time to discrete-time process. 
Let $(X_{nh})_{n\in \N}$ and $(\bar{X}_t)_{t\geq 0}$ be respectively the LMC with step size $h$ and the continuous Langevin diffusion. Both are with score function $\nabla \log \mu$  and have the same initialization.
By an explicit calculation,
we can bound $\norm{\nabla^2 \log \mu(x)}_{OP}$ along the trajectory of the continuous process. This combined with
the consequence of Girsanov's theorem \eqref{eqn:girsanov-consequence} allows us to bound the total variation distance between the continuous ($\bar{X}_t$) and discretized ($X_{nh}$) processes. 
 For appropriate choices of step size $h$ and time $T = Nh$, 
 using triangle inequality and the bound $d_{TV}(\bar{X}_T, \mu)$, we conclude that the discretized process $X_{Nh}$ is close to $\mu.$ 

\paragraph{Sampling with an $L_2$-approximate score function.} (\emph{\cref{sec:discrete}})
In many cases, score functions are learned from data, so we only have access to an $L_2$-estimate $s$ of the score such that 
$\mathbb{E}_{\mu} [\norm{s(x)- \nabla \log \mu(x)}^2] \leq \epsilon_{\text{score}}^2$.
We now describe how to make the analysis work in this setting. 
Using Girsanov's theorem, we can bound the total variation distance between the LMC $(X_{nh}^{s,\mu})_{n\in \N}$ initialized at $\mu$ with score estimate $s$ and the continuous Langevin diffusion $(\bar{Z}_{nh}^{\mu})_{n\in \N}$ with true score function $\nabla \log \mu,$ thus we can bound the probability that the LMC $(X_{nh}^{s,\mu})_{n=\set{0,\cdots, N-1}}$ hits the bad set \[ B_{\text{score}}:= \set{x:\norm{s(x)-\log \mu(x)}\geq \epsilon_{\text{score},1}}. \]
(The idea of defining a ``bad set'' is inspired by  the analysis of \cite{lee2022convergence}.)
Similar to the argument for the continuous process, let $X_{nh}^{s,\nu_{\text{sample}} }  $ denote the LMC with score function $s$ and step size $h$ initialized at the empirical distribution $\nu_{\text{sample}}$.
Since we know that $X_{nh}^{s,\mu}$ avoids the bad set and that $\mathcal{L} (X_{nh}^{s,\mu}) =\mathbb{E}_{U_{\text{sample}}\sim \mu^{\otimes M}} [ \mathcal{L} (X_{nh}^{s,\nu_{\text{sample}}}) ) ]$, we have by Markov's inequality that for a typical $U_{\text{sample}},$ with high probability over the randomness of the Brownian motion, $X_{nh}^{s,\nu_{\text{sample}}}$ also avoids the bad set $B_{\text{score}}$ for all $0 \le n < N.$ Thus, we can compare $X_{nh}^{s, \nu_{\text{sample}}} $ with the LMC with true score function $\nabla \log \mu,$ 
and conclude that $\mathcal{L}( X_{Nh}^{s,\nu_{\text{sample}}})$ is close to $\mu$ in total variation distance.  

\subsection{Technical ingredient: log-Sobolev constant of well-connected mixtures}
The following theorem, which we prove in the appendix, is used in the above argument to bound the log-Sobolev constant of mixture distributions where the components have significant overlap. 
\begin{theorem} \label{thm:log sobolev and poincare for mixture}
Let $I$ be a set, and consider probability measures $\set{\mu_i}_{i\in I}$, nonnegative weights $(p_i)_{i\in I}$ summing to one, and mixture distribution $ \mu =\sum_i p_i \mu_i. $
Let $G$ be the graph on vertex set $I $ where 
there is an edge between $i, j $ if 
$ \mu_i, \mu_j$ have high overlap i.e.
\[ \delta_{ij} :=\int \min \set{\mu_i(x), \mu_j(x)} dx\geq \delta.\]
Suppose $G$ is connected and let $p_*=\min p_i.$
The mixture distribution $\mu = \sum_{i\in I} p_i \mu_i$ has 
log-Sobolev constant
\[ C_{LS} (\mu) \leq \frac{C_{\abs{I},p_*}}{\delta} \max_{i} C_{LS} (\mu_i) \]
where $C_{\abs{I},p_*} = 4\abs{I} (1+ \log(p_*^{-1}))  p_*^{-1}$ only depends on $ \abs{I}$ and $p_*.$
\end{theorem}
A version of Theorem which bounds the (weaker) Poincar\'e constant instead 
appeared before as Theorem 1.2 of \cite{Madras2002MarkovCD},
but the result for the log-Sobolev constant 
is new to the best of our knowledge. Compared to \cite{chen2021dimension}, our assumption is milder than their assumption that the chi-square divergence between any two components is bounded.
(For example, two non-isotropic Gaussians might have infinite chi-square divergence (see e.g. \cite[Section 4.3]{Schlichting_2019}), so in that case their result doesn't imply a finite bound on the LSI of their mixture.) \cite{Schlichting_2019} bounds LSI of $\mu= p\mu_1 + (1-p)\mu_2$ when either $\chi^2(\mu_1 ||\mu_2) $ or $\chi^2(\mu_2 ||\mu_1) $ are bounded; our bound applies to mixtures of more than two components. 
    
\section{Simulations}
\begin{figure}
    \centering
    \begin{subfigure}{0.32\textwidth}
    \includegraphics[width=\textwidth]{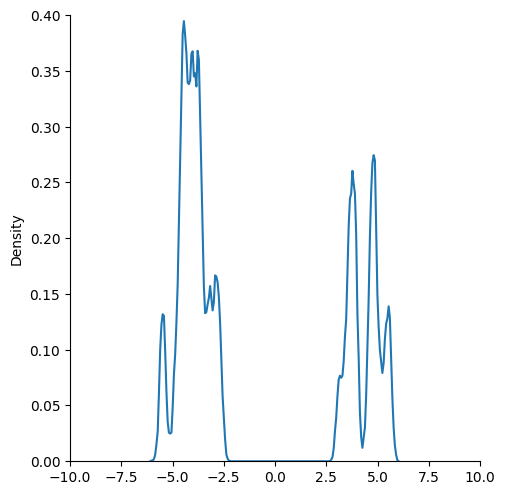}
    \caption{$T = 0$}
    \end{subfigure}
    \hfill
    \begin{subfigure}{0.32\textwidth}
    \includegraphics[width=\textwidth]{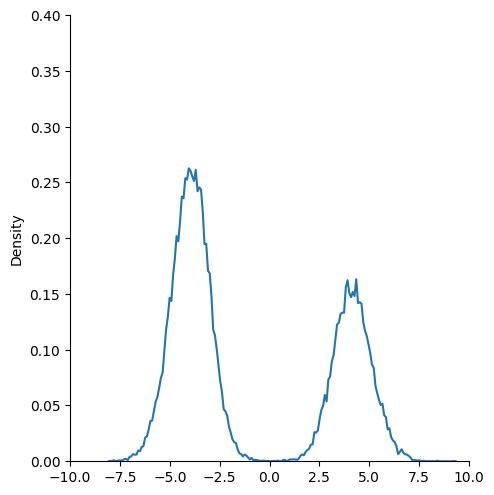}
    \caption{$T = 200$}
    \end{subfigure}
    \hfill
    \begin{subfigure}{0.32\textwidth}
    \includegraphics[width=\textwidth]{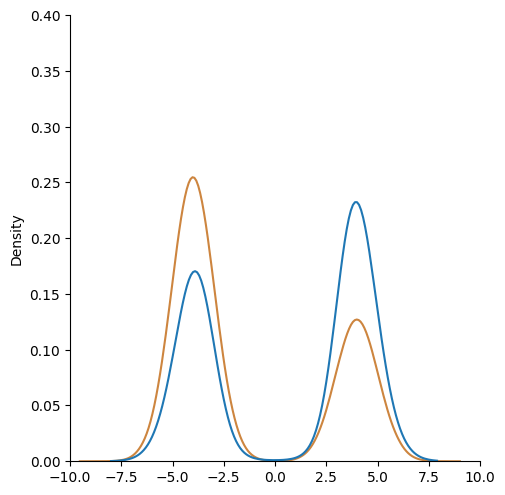}
    \caption{$T = \infty$ \& truth (orange)}
    \end{subfigure}

    \caption{Visualization of the distribution of the Langevin dynamics after $T$ iterations when initialized at the empirical distribution and run with an approximate score function estimated from data. Orange density (rightmost figure) is the ground truth mixture of two Gaussians; the empirical distribution (leftmost figure, $T = 0$) consists of 40 iid samples from the ground truth. Langevin dynamics with step size $0.01$ is run with an estimated score function, which was fit using vanilla score matching with a one hidden-layer neural network trained on fresh samples; densities (blue) are visualized using a Gaussian Kernel Density Estimate (KDE).
    Matching our theory, we see that the ground truth is accurately estimated at time $T = 200$ even though it is not at $T = 0$ or $\infty$. 
    }
\label{fig:one-dimensional}
\end{figure}
In Figure~\ref{fig:one-dimensional}, we simulated the behavior of the Langevin dynamics with step size $0.01$ and an estimated score function initialized at the ground truth distribution on a simple 1-dimensional example, a mixture of two Gaussians. If the Langevin dynamics are run until mixing, this corresponds to exactly performing the standard vanilla score matching procedure and 
this will fail to estimate the ground truth distribution well, which we see in the rightmost subfigure. The empirical distribution (time zero for the dynamics) is also not a good fit to the ground truth, but as our theory predicts the early-stopped Langevin diffusion (subfigure (b)) is indeed a good estimate for the ground truth. 

In Figure~\ref{fig:32dimexample} we simulated the trajectories of Langevin dynamics with step size $0.001$, again with initialization from samples and a learned score function, in a 32-dimensional mixture of Gaussians. Similar to the one-dimensional example, we can see that at moderate times the trajectories have mixed well within their  component, and at large times the trajectories sometimes pass through the region in between the components where the true density is very small. Additional simulations (including an experiment with Contrastive Divergence training) and information is in \cref{sec:additional-details}.

\begin{figure}
    \centering
    \begin{subfigure}{0.32\textwidth}
    \includegraphics[width=\textwidth]{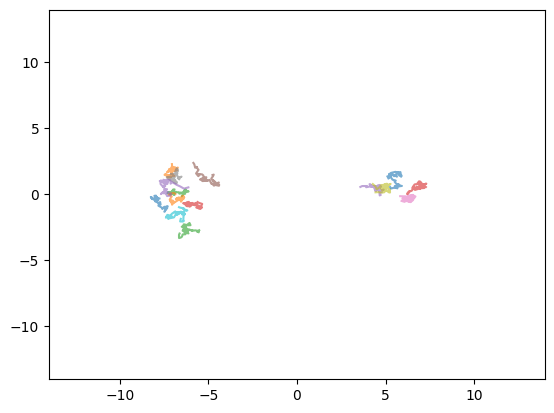}
    \caption{$T = 300$}
    \end{subfigure}
    \hfill
    \begin{subfigure}{0.32\textwidth}
    \includegraphics[width=\textwidth]{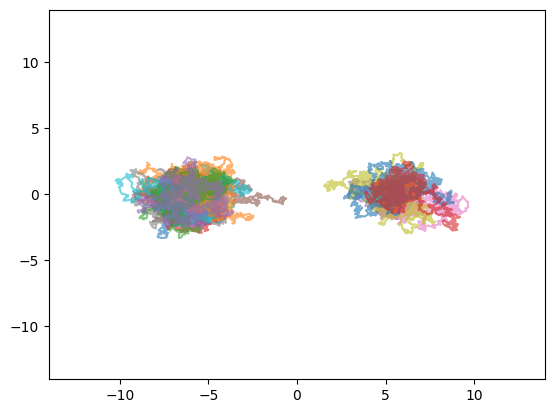}
    \caption{$T = 12000$}
    \end{subfigure}
    \hfill
    \begin{subfigure}{0.32\textwidth}
    \includegraphics[width=\textwidth]{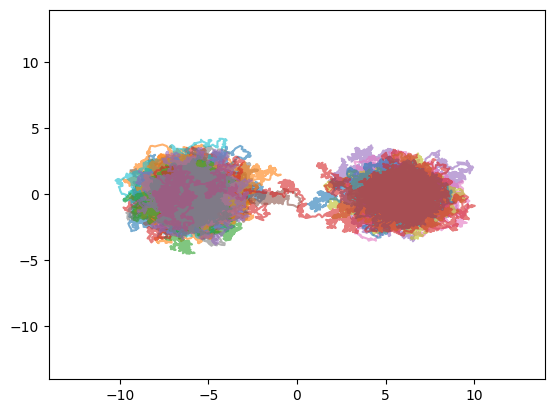}
    \caption{$T = 120000$ }
    \end{subfigure}
    \caption{2D projected trajectories of Langevin dynamics up to $T$ iterations with step size $0.001$ in a 32-dimensional mixture of Gaussians $\frac{2}{3} N(-6e_1, 1.5 I) + \frac{1}{3} N(6 e_1,1.5 I)$. The projection is the first two coordinates and the direction of separation of the components is the first axis direction. 
    Langevin is initialized from the empirical distribution (15 iid samples) and run with an approximate score function learned from samples using a one hidden-layer neural network.}
    \label{fig:32dimexample}
\end{figure}

\bibliographystyle{iclr2024_conference}
\bibliography{refs}

\begin{thebibliography}{45}
\providecommand{\natexlab}[1]{#1}
\providecommand{\url}[1]{\texttt{#1}}
\expandafter\ifx\csname urlstyle\endcsname\relax
  \providecommand{\doi}[1]{doi: #1}\else
  \providecommand{\doi}{doi: \begingroup \urlstyle{rm}\Url}\fi

\bibitem[Albergo et~al.(2021)Albergo, Boyda, Hackett, Kanwar, Cranmer,
  Racaniere, Rezende, and Shanahan]{albergo2021introduction}
Michael~S Albergo, Denis Boyda, Daniel~C Hackett, Gurtej Kanwar, Kyle Cranmer,
  S{\'e}bastien Racaniere, Danilo~Jimenez Rezende, and Phiala~E Shanahan.
\newblock Introduction to normalizing flows for lattice field theory.
\newblock \emph{arXiv preprint arXiv:2101.08176}, 2021.

\bibitem[Bakry et~al.(2014)Bakry, Gentil, Ledoux, et~al.]{bakry2014analysis}
Dominique Bakry, Ivan Gentil, Michel Ledoux, et~al.
\newblock \emph{Analysis and geometry of Markov diffusion operators}, volume
  103.
\newblock Springer, 2014.

\bibitem[Barp et~al.(2019)Barp, Briol, Duncan, Girolami, and
  Mackey]{barp2019minimum}
Alessandro Barp, Francois-Xavier Briol, Andrew Duncan, Mark Girolami, and
  Lester Mackey.
\newblock Minimum stein discrepancy estimators.
\newblock \emph{Advances in Neural Information Processing Systems}, 32, 2019.

\bibitem[Besag(1975)]{besag1975statistical}
Julian Besag.
\newblock Statistical analysis of non-lattice data.
\newblock \emph{Journal of the Royal Statistical Society: Series D (The
  Statistician)}, 24\penalty0 (3):\penalty0 179--195, 1975.

\bibitem[Block et~al.(2020)Block, Mroueh, and Rakhlin]{block2020generative}
Adam Block, Youssef Mroueh, and Alexander Rakhlin.
\newblock Generative modeling with denoising auto-encoders and langevin
  sampling.
\newblock \emph{arXiv preprint arXiv:2002.00107}, 2020.

\bibitem[Brennan et~al.(2018)Brennan, Bresler, and
  Huleihel]{brennan2018reducibility}
Matthew Brennan, Guy Bresler, and Wasim Huleihel.
\newblock Reducibility and computational lower bounds for problems with planted
  sparse structure.
\newblock In \emph{Conference On Learning Theory}, pp.\  48--166. PMLR, 2018.

\bibitem[Chen et~al.(2021)Chen, Chewi, and Niles-Weed]{chen2021dimension}
Hong-Bin Chen, Sinho Chewi, and Jonathan Niles-Weed.
\newblock Dimension-free log-sobolev inequalities for mixture distributions.
\newblock \emph{Journal of Functional Analysis}, 281\penalty0 (11):\penalty0
  109236, 2021.

\bibitem[Chen et~al.(2023)Chen, Chewi, Li, Li, Salim, and
  Zhang]{chen2023sampling}
Sitan Chen, Sinho Chewi, Jerry Li, Yuanzhi Li, Adil Salim, and Anru~R. Zhang.
\newblock Sampling is as easy as learning the score: theory for diffusion
  models with minimal data assumptions, 2023.

\bibitem[Chewi et~al.(2021)Chewi, Erdogdu, Li, Shen, and
  Zhang]{chewi2021analysis}
Sinho Chewi, Murat~A. Erdogdu, Mufan~Bill Li, Ruoqi Shen, and Matthew Zhang.
\newblock Analysis of langevin monte carlo from poincar\'e to log-sobolev,
  2021.

\bibitem[Dalalyan(2017)]{dalalyan2017theoretical}
Arnak~S Dalalyan.
\newblock Theoretical guarantees for approximate sampling from smooth and
  log-concave densities.
\newblock \emph{Journal of the Royal Statistical Society. Series B (Statistical
  Methodology)}, pp.\  651--676, 2017.

\bibitem[Diaconis \& Saloff-Coste(1996)Diaconis and
  Saloff-Coste]{diaconis1996logarithmic}
Persi Diaconis and Laurent Saloff-Coste.
\newblock Logarithmic sobolev inequalities for finite markov chains.
\newblock \emph{The Annals of Applied Probability}, 6\penalty0 (3):\penalty0
  695--750, 1996.

\bibitem[Forbes \& Lauritzen(2015)Forbes and Lauritzen]{forbes2015linear}
Peter~GM Forbes and Steffen Lauritzen.
\newblock Linear estimating equations for exponential families with application
  to gaussian linear concentration models.
\newblock \emph{Linear Algebra and its Applications}, 473:\penalty0 261--283,
  2015.

\bibitem[Gao et~al.(2018)Gao, Lu, Zhou, Zhu, and Wu]{gao2018learning}
Ruiqi Gao, Yang Lu, Junpei Zhou, Song-Chun Zhu, and Ying~Nian Wu.
\newblock Learning generative convnets via multi-grid modeling and sampling.
\newblock In \emph{Proceedings of the IEEE Conference on Computer Vision and
  Pattern Recognition}, pp.\  9155--9164, 2018.

\bibitem[Ge et~al.(2018)Ge, Lee, and Risteski]{ge2018simulated}
Rong Ge, Holden Lee, and Andrej Risteski.
\newblock Simulated tempering langevin monte carlo ii: An improved proof using
  soft markov chain decomposition.
\newblock \emph{arXiv preprint arXiv:1812.00793}, 2018.

\bibitem[Hinton(2002)]{hinton2002training}
Geoffrey~E Hinton.
\newblock Training products of experts by minimizing contrastive divergence.
\newblock \emph{Neural computation}, 14\penalty0 (8):\penalty0 1771--1800,
  2002.

\bibitem[Hinton(2012)]{hinton2012practical}
Geoffrey~E Hinton.
\newblock A practical guide to training restricted boltzmann machines.
\newblock \emph{Neural Networks: Tricks of the Trade: Second Edition}, pp.\
  599--619, 2012.

\bibitem[Hyv{\"a}rinen(2005)]{hyvarinen2005estimation}
Aapo Hyv{\"a}rinen.
\newblock Estimation of non-normalized statistical models by score matching.
\newblock \emph{Journal of Machine Learning Research}, 6\penalty0 (4), 2005.

\bibitem[Hyv{\"a}rinen(2007{\natexlab{a}})]{hyvarinen2007connections}
Aapo Hyv{\"a}rinen.
\newblock Connections between score matching, contrastive divergence, and
  pseudolikelihood for continuous-valued variables.
\newblock \emph{IEEE Transactions on neural networks}, 18\penalty0
  (5):\penalty0 1529--1531, 2007{\natexlab{a}}.

\bibitem[Hyv{\"a}rinen(2007{\natexlab{b}})]{hyvarinen2007some}
Aapo Hyv{\"a}rinen.
\newblock Some extensions of score matching.
\newblock \emph{Computational statistics \& data analysis}, 51\penalty0
  (5):\penalty0 2499--2512, 2007{\natexlab{b}}.

\bibitem[Karatzas \& Shreve(1991)Karatzas and Shreve]{karatzas1991brownian}
Ioannis Karatzas and Steven~E Shreve.
\newblock \emph{Brownian motion and stochastic calculus}, volume 113.
\newblock Springer Science \& Business Media, 1991.

\bibitem[Koehler et~al.(2022)Koehler, Heckett, and
  Risteski]{koehler2022statistical}
Frederic Koehler, Alexander Heckett, and Andrej Risteski.
\newblock Statistical efficiency of score matching: The view from isoperimetry.
\newblock \emph{arXiv preprint arXiv:2210.00726}, 2022.

\bibitem[Lawrence \& Yamauchi(2021)Lawrence and
  Yamauchi]{lawrence2021normalizing}
Scott Lawrence and Yukari Yamauchi.
\newblock Normalizing flows and the real-time sign problem.
\newblock \emph{Physical Review D}, 103\penalty0 (11):\penalty0 114509, 2021.

\bibitem[Lee et~al.(2018)Lee, Risteski, and Ge]{lee2018}
Holden Lee, Andrej Risteski, and Rong Ge.
\newblock Beyond log-concavity: Provable guarantees for sampling multi-modal
  distributions using simulated tempering langevin monte carlo.
\newblock In S.~Bengio, H.~Wallach, H.~Larochelle, K.~Grauman, N.~Cesa-Bianchi,
  and R.~Garnett (eds.), \emph{Advances in Neural Information Processing
  Systems}, volume~31. Curran Associates, Inc., 2018.
\newblock URL
  \url{https://proceedings.neurips.cc/paper/2018/file/c6ede20e6f597abf4b3f6bb30cee16c7-Paper.pdf}.

\bibitem[Lee et~al.(2022{\natexlab{a}})Lee, Lu, and Tan]{lee2022convergence}
Holden Lee, Jianfeng Lu, and Yixin Tan.
\newblock Convergence for score-based generative modeling with polynomial
  complexity.
\newblock \emph{arXiv preprint arXiv:2206.06227}, 2022{\natexlab{a}}.

\bibitem[Lee et~al.(2022{\natexlab{b}})Lee, Lu, and Tan]{lee2022convergence2}
Holden Lee, Jianfeng Lu, and Yixin Tan.
\newblock Convergence of score-based generative modeling for general data
  distributions.
\newblock \emph{arXiv preprint arXiv:2209.12381}, 2022{\natexlab{b}}.

\bibitem[Levin \& Peres(2017)Levin and Peres]{levin2017markov}
David~A Levin and Yuval Peres.
\newblock \emph{Markov chains and mixing times}, volume 107.
\newblock American Mathematical Soc., 2017.

\bibitem[Madras \& Randall(2002)Madras and Randall]{Madras2002MarkovCD}
Neal Madras and Dana Randall.
\newblock Markov chain decomposition for convergence rate analysis.
\newblock \emph{Annals of Applied Probability}, 12:\penalty0 581--606, 2002.

\bibitem[Mironov(2017)]{mironov2017renyi}
Ilya Mironov.
\newblock R{\'e}nyi differential privacy.
\newblock In \emph{2017 IEEE 30th computer security foundations symposium
  (CSF)}, pp.\  263--275. IEEE, 2017.

\bibitem[Nijkamp et~al.(2019)Nijkamp, Hill, Zhu, and Wu]{nijkamp2019learning}
Erik Nijkamp, Mitch Hill, Song-Chun Zhu, and Ying~Nian Wu.
\newblock Learning non-convergent non-persistent short-run mcmc toward
  energy-based model.
\newblock \emph{Advances in Neural Information Processing Systems}, 32, 2019.

\bibitem[Nijkamp et~al.(2020)Nijkamp, Hill, Han, Zhu, and
  Wu]{nijkamp2020anatomy}
Erik Nijkamp, Mitch Hill, Tian Han, Song-Chun Zhu, and Ying~Nian Wu.
\newblock On the anatomy of mcmc-based maximum likelihood learning of
  energy-based models.
\newblock In \emph{Proceedings of the AAAI Conference on Artificial
  Intelligence}, volume~34, pp.\  5272--5280, 2020.

\bibitem[Rigollet \& H{\"u}tter(2017)Rigollet and H{\"u}tter]{rigollet2015high}
Phillippe Rigollet and Jan-Christian H{\"u}tter.
\newblock High dimensional statistics.
\newblock \emph{Lecture notes for course 18S997}, 2017.

\bibitem[Rombach et~al.(2022)Rombach, Blattmann, Lorenz, Esser, and
  Ommer]{rombach2022high}
Robin Rombach, Andreas Blattmann, Dominik Lorenz, Patrick Esser, and Bj{\"o}rn
  Ommer.
\newblock High-resolution image synthesis with latent diffusion models.
\newblock In \emph{Proceedings of the IEEE/CVF Conference on Computer Vision
  and Pattern Recognition}, pp.\  10684--10695, 2022.

\bibitem[Schlichting(2019)]{Schlichting_2019}
Andr{\'{e}} Schlichting.
\newblock Poincar{\'{e}} and log{\textendash}sobolev inequalities for mixtures.
\newblock \emph{Entropy}, 21\penalty0 (1):\penalty0 89, 2019.
\newblock \doi{10.3390/e21010089}.
\newblock URL \url{https://doi.org/10.3390%2Fe21010089}.

\bibitem[Song \& Ermon(2019)Song and Ermon]{song2019generative}
Yang Song and Stefano Ermon.
\newblock Generative modeling by estimating gradients of the data distribution.
\newblock \emph{Advances in Neural Information Processing Systems}, 32, 2019.

\bibitem[Song et~al.(2020{\natexlab{a}})Song, Garg, Shi, and
  Ermon]{song2020sliced}
Yang Song, Sahaj Garg, Jiaxin Shi, and Stefano Ermon.
\newblock Sliced score matching: A scalable approach to density and score
  estimation.
\newblock In \emph{Uncertainty in Artificial Intelligence}, pp.\  574--584.
  PMLR, 2020{\natexlab{a}}.

\bibitem[Song et~al.(2020{\natexlab{b}})Song, Sohl-Dickstein, Kingma, Kumar,
  Ermon, and Poole]{song2020score}
Yang Song, Jascha Sohl-Dickstein, Diederik~P Kingma, Abhishek Kumar, Stefano
  Ermon, and Ben Poole.
\newblock Score-based generative modeling through stochastic differential
  equations.
\newblock \emph{arXiv preprint arXiv:2011.13456}, 2020{\natexlab{b}}.

\bibitem[Sriperumbudur et~al.(2017)Sriperumbudur, Fukumizu, Gretton,
  Hyv{\"a}rinen, and Kumar]{sriperumbudur2017density}
Bharath Sriperumbudur, Kenji Fukumizu, Arthur Gretton, Aapo Hyv{\"a}rinen, and
  Revant Kumar.
\newblock Density estimation in infinite dimensional exponential families.
\newblock \emph{Journal of Machine Learning Research}, 18, 2017.

\bibitem[Tzen et~al.(2018)Tzen, Liang, and Raginsky]{tzen2018local}
Belinda Tzen, Tengyuan Liang, and Maxim Raginsky.
\newblock Local optimality and generalization guarantees for the langevin
  algorithm via empirical metastability.
\newblock In \emph{Conference On Learning Theory}, pp.\  857--875. PMLR, 2018.

\bibitem[Van~Handel(2014)]{van2014probability}
Ramon Van~Handel.
\newblock Probability in high dimension.
\newblock Technical report, PRINCETON UNIV NJ, 2014.

\bibitem[Vempala \& Wibisono(2019)Vempala and Wibisono]{vempala2019rapid}
Santosh Vempala and Andre Wibisono.
\newblock Rapid convergence of the unadjusted langevin algorithm: Isoperimetry
  suffices.
\newblock \emph{Advances in neural information processing systems}, 32, 2019.

\bibitem[Vershynin(2018)]{vershynin2018high}
Roman Vershynin.
\newblock \emph{High-dimensional probability: An introduction with applications
  in data science}, volume~47.
\newblock Cambridge university press, 2018.

\bibitem[Vincent(2011)]{vincent2011connection}
Pascal Vincent.
\newblock A connection between score matching and denoising autoencoders.
\newblock \emph{Neural computation}, 23\penalty0 (7):\penalty0 1661--1674,
  2011.

\bibitem[Wasserman(2006)]{wasserman2006all}
Larry Wasserman.
\newblock \emph{All of nonparametric statistics}.
\newblock Springer Science \& Business Media, 2006.

\bibitem[Wenliang et~al.(2019)Wenliang, Sutherland, Strathmann, and
  Gretton]{wenliang2019learning}
Li~Wenliang, Danica~J Sutherland, Heiko Strathmann, and Arthur Gretton.
\newblock Learning deep kernels for exponential family densities.
\newblock In \emph{International Conference on Machine Learning}, pp.\
  6737--6746. PMLR, 2019.

\bibitem[Xie et~al.(2016)Xie, Lu, Zhu, and Wu]{xie2016theory}
Jianwen Xie, Yang Lu, Song-Chun Zhu, and Yingnian Wu.
\newblock A theory of generative convnet.
\newblock In \emph{International Conference on Machine Learning}, pp.\
  2635--2644. PMLR, 2016.

\end{thebibliography}

\newpage
\appendix
\section{Organization of Appendix}
In Appendix~\ref{sec:preliminaries}, we review some basic mathematical preliminaries and notation, such as the definition of log-Sobolev and Poincar\'e inequalities. In Appendix~\ref{sec:mixture-sobolev} we prove Theorem~\ref{thm:log sobolev and poincare for mixture formal} 
(\cref{thm:log sobolev and poincare for mixture} of the main text), which shows that when clusters have significant overlap that the Langevin dynamics for the mixture distribution will successfully mix. Appendix~\ref{sec:initialization} and Appendix~\ref{sec:perturbation} contain intermediate results which are used in the following sections: Appendix~\ref{sec:initialization} shows how to analyze the Langevin diffusion starting from a point, and Appendix~\ref{sec:perturbation} shows how to bound the drift of the continuous Langevin diffusion over a short period of time. 
In Appendix~\ref{sec:continuous} we prove \cref{thm:continuous mixing}, which shows that the continuous Langevin diffusion with score function $\nabla V $ converges to $\mu$ after a suitable time $T.$  In Appendix~\ref{sec:discrete}, we prove our main results \cref{thm:discrete mixing for cluster of distribution with close centers} and \cref{cor:mixing of discrete chain with score error}, which show that the discrete LMC with score function $s$ with appropriately chosen step size is close to $\mu$ in total variation distance at a suitable time. Corollary~\ref{cor:mixing of discrete chain with score error} corresponds to Theorem 1 of the main text. In \cref{sec:remove minimum weight assumption}, we remove the dependency of the runtime and number of samples on the minimum weight of the components i.e. $p_* =\min_{i\in I} p_i$ (see \cref{thm:discrete mixing for cluster of distribution with close centers modified,cor:mixing of discrete chain with score error modified} for the analogy of \cref{thm:discrete mixing for cluster of distribution with close centers,cor:mixing of discrete chain with score error} respectively that has no dependency on $p_*$). Appendix~\ref{sec:additional-details} contains some additional simulations.

\section{Preliminaries}\label{sec:preliminaries}
In the preliminaries, we review in more detail the needed background on divergences between probability measures, functional inequalities, log-concave distributions, etc. in order to prove our main results. 

\paragraph{Notation.} We use standard big-Oh notation and use tildes, e.g. $\tilde{O}(\cdot)$, to denote inequality up to log factors. We similarly use the notation $\lesssim$ to denote inequality up to a universal constant. 
We let $d_{TV}(\mu,\nu) = \sup_{A} |\mu(A) - \nu(A)|$ be the usual total variation distance between probability measures $\mu$ and $\nu$ defined on the same space, where the supremum ranges over measurable sets. Given a random variable $X$, we write $\mathcal L(X)$ to denote its law. In general, we use the same notation for a measure and its probability density function as long as there is no ambiguity. For random variables $X, Z$, we will write $d_{TV}(X,Z)$ to denote the total variation distance between their laws $\mathcal{L}(X)$ and $\mathcal{L}(Z).$

\subsection{Renyi divergence}
The Renyi divergence, which generalizes the more well-known KL divergence, is a useful technical tool in the analysis of the Langevin diffusion --- see e.g. \cite{vempala2019rapid}.
The Renyi divergence of order $q \in (1, \infty)$ of $\mu$ from $\pi$ is defined to be
\begin{align*}
  \Renyi_q(\mu || \pi) &= \frac{1}{q-1} \ln \E_{\pi} \left[\left( \frac{d\mu(x)}{d\pi(x)}\right)^q\right] =\frac{1}{q-1}\ln  \int \left( \frac{d\mu(x)}{d\pi(x)}\right)^q d\pi(x)\\
  &  = \frac{1}{q-1}\ln  \int \left( \frac{d\mu(x)}{d\pi(x)}\right)^{q-1} d\mu(x) =   \frac{1}{q-1} \ln \E_{\mu} \left[\left( \frac{d\mu(x)}{d\pi(x)}\right)^{q-1}\right]   
\end{align*}
The limit $\Renyi_q$ as $q \to 1$ is the Kullback-Leibler divergence $\D_{\KL}(\mu || \pi)  = \int \mu(x) \log \frac{\mu(x)}{\pi(x)} dx,$ thus we write $ \Renyi_1 (\cdot ) = \D_{\KL}(\cdot).$ Renyi divergence increases as $q$ increases i.e. $\Renyi_q \leq \Renyi_{q'} $ for $1\leq q \leq q'.$
\begin{lemma}[Weak triangle inequality, {\cite[Lemma 7]{vempala2019rapid}}, \cite{mironov2017renyi}] \label{lem:weak triangle inequality}
For $q > 1$ and any measure $\nu$ absolutely continuous with respect to measure $\mu$,
\[\Renyi_q(\nu || \mu) \leq \frac{q - 1/2}{q - 1} \Renyi_{2q}(\nu || \nu') + \Renyi_{2q-1} (\nu' || \mu)\]
\end{lemma}
\begin{lemma}[Weak convexity of Renyi entropy] \label{lem:weak convexity}
 For $q> 1$, if $\mu$ is a convex combination of $\mu_i$ i.e. $\mu(x) = \sum p_i \mu_i(x)$  then
 \[ \E_{\nu}\left[\left(\frac{d\nu (x)}{d \mu(x)}\right)^{q-1}\right] \leq \sum_i  p_i \E_{\nu} \left[\left(\frac{d\nu (x)}{d \mu_i(x)}\right)^{q-1}\right]. \]
 Consequently,
 $\Renyi_q (\nu || \mu) \leq \max_i\Renyi_q (\nu || \mu_i)  $ and $ \Renyi_q(\mu || \nu) \leq \max_i\Renyi_q (\mu_i || \nu) $
 \end{lemma}
 \begin{proof}
 By Holder's inequality
 \[(\sum_i p_i \mu_i(x) )^{q-1} \left( \sum_{i=1}^d  \frac{p_i}{\mu_i(x)^{q-1} } \right) \geq (\sum_i p_i)^q = 1\]
 thus
 \[ \left(\frac{\nu(x)}{\mu(x)}\right)^{q-1}\leq \sum_i p_i \left(\frac{\nu(x)}{\mu_i(x)}\right)^{q-1} \]
 Taking expectation in $\nu$ gives the first statement. Similarly, since $q > 1 > 0,$
 \[\E_{\nu} \left[\left(\frac{\nu(x)}{\mu(x)}\right)^{q}\right]\leq \sum_i p_i \E_{\nu}\left[\left(\frac{\nu(x)}{\mu_i(x)}\right)^{q}\right] \]
 For the second statement
 \[ \Renyi_q(\nu || \mu) = \frac{\ln\E_{\nu}[(\frac{d\nu (x)}{d \mu(x)})^{q-1}] }{q-1} \leq \frac{\ln (\max_{i} \E_{\nu}[(\frac{d\nu (x)}{d \mu_i(x)})^{q-1}] ) }{q-1} = \max_i \Renyi_q (\nu ||\mu_i) \]
and
\[ \Renyi_q(\mu || \nu) = \frac{\ln\E_{\nu}[(\frac{d\nu (x)}{d \mu(x)})^{q}] }{q-1} \leq \frac{\ln (\max_{i} \E_{\nu}[(\frac{d\nu (x)}{d \mu_i(x)})^{q}] ) }{q-1} = \max_i \Renyi_q ( \mu_i||\nu).  \]

 \end{proof}
 
 \subsection{Log-concave distributions}
Consider a density function $\pi: \R^d \to \R_{\geq 0}$ where $\pi(x)=\exp(-V(x)) .$ 
Throughout the paper, we will assume $V$ is a twice continuously differentiable function. We say $\pi$ is $\beta$-smooth if $V$ has bounded Hessian for all $x \in \mathbb{R}^d$:
\[-\beta I\preceq \nabla^2 V(x) \preceq \beta I. \]
We say $\pi$ is $\alpha$-strongly log-concave if 
\[  0\prec \alpha I \preceq \nabla^2 V(x)\]
for all $x \in \mathbb{R}^d$. 

  \subsection{Functional  inequalities}
 For nonnegative smooth $ f: \R^d \to \R_{\geq 0},$ let the entropy of $f$ with respect to probability distribution $\pi$ be \[ \Ent_{\pi} [f]  = \E_{\pi} [f\ln (f/ \E_{\pi}[f]) ]. \]
We say $\pi$ satisfies a log-Sobolev inequality (LSI) with constant $C_{LS}$ if  for all smooth functions $f$,
\[\Ent_{\pi} [f^2] \leq 2 C_{LS} \E_{\pi}[ \norm{\nabla f}^2 ] \]
and $\pi$ satisfies a Poincare inequality (PI) with constant $C_{PI}$ if
$\Var_{\pi}[f] \leq 2 C_{PI} \E_{\pi}[ \norm{\nabla f}^2 ]$.
The log-Sobolev inequality implies Poincare inequality: $C_{PI} \leq C_{LS}.$
Due to the Bakry-Emery criterion, if $\pi$ is $\alpha$-strongly log-concave then $ \pi $ satisfies LSI with constant $C_{LS}=1/\alpha.$ 

LSI and PI are equivalent to statements about exponential ergodicity of the continuous-time Langevin diffusion, which is defined by the Stochastic Differential Equation
\[ d\bar X_t^{\pi} = \nabla \log \pi(\bar X_t^{\mu})\, dt + \sqrt{2}\, dB_t. \]   
Specifically, let $\pi_t$ denote the law of the diffusion at time $t$ initialized from $\pi_0$ then a LSI is equivalent to the inequality
\[ \D_{\KL} (\pi_t || \pi) \leq \exp(-2t/C_{LS}) \D_{\KL}(\pi_0 || \pi) \]
holding for an arbitrary initial distribution $\pi_0$.
Similarly, a PI is equivalent to
$\chi^2 (\pi_t || \pi) \leq \exp(-2t/C_{PI}) \chi^2(\pi_0 || \pi)$. Here $\D_{KL}(P,Q) = \E_P[\log \frac{dP}{dQ}]$ is the Kullback-Liebler divergence and $\chi^2(P,Q) = \E_Q[(dP/dQ - 1)^2]$ is the $\chi^2$-divergence. See \cite{bakry2014analysis,van2014probability} for more background.

\subsection{Concentration}
\begin{proposition}[Concentration of Brownian motion, {\cite[Lemma 32]{chewi2021analysis}}] \label{prop:brownian}
Let $(B_t)_{t\geq 0}$ be a standard Brownian motion in $\R^d$. Then, if $\lambda  \geq 0$ and $h \leq 1/(4\lambda),$
\[\E\left[\exp \left(\lambda \sup_{t\in [0,h]} \norm{B_t}^2\right) \right] \leq \exp(6dh \lambda) \]
In particular, for all $\eta \geq 0$
\[\P \left[\sup_{t\in [0,h]} \norm{B_t}^2 \geq \eta\right]\leq \exp\left(-\frac{\eta^2}{6dh}\right) \]
\end{proposition}

\begin{proposition} \label{prop:moment bound for subgaussian concentration}
Suppose a random non-negative real variable $Z$ satisfies
\[\forall t: \P[Z \geq D + t]\leq 2 \exp(-\gamma t^2)\]
for some $D\geq 0,\gamma > 0. $ Then there exists numerical constant $C$ s.t.
\[\E[Z^p] \leq C p^{p/2} (D+\gamma^{-1/2})^p \]
\end{proposition}
\begin{proof}
For some $R\geq D$ to be chosen later
\begin{align*}
     \E[ Z^p] 
     &=\int_0^{\infty} \P[Z^p \geq x] dx\\
     &= \int_0^{R^p}  \P[Z^p \geq x] dx +  \int_{R^p}^{\infty}  \P[Z^p \geq x] dx\\
     &\leq  \int_0^{R^p}  1 dx + \int_{R}^{\infty} \P[Z\geq y ] d (y^p)\\
    &\leq R^p + 2p \int_{R}^{\infty} y^{p-1} \exp(-\gamma (y-D)^2) dy \\
    &\leq  R^p + p 2^{p} (\int_{R}^{\infty} z^{p-1} \exp(-\gamma z^2) dz + D^{p-1} \int_{R}^{\infty}  \exp(-\gamma z^2) dz)\\
    &\leq R^p +  2^{p-1} (\gamma^{-p/2} (p/2)^{p/2} + p D^{p-1} \gamma^{-1/2} \sqrt{\pi}  ) 
 \end{align*}
 where in the last inequality, we make a change of variable $ u = \gamma z^2$ and note that $2p \int z^{p-1} \exp(-\gamma z^2) dz = \gamma^{-p} p \int u^{p/2-1} \exp(-u) du = \Gamma(p/2)\leq (p/2)^{p/2}$ and $ \int_0^{\infty} \exp(-\gamma z^2) dz = ( 2\gamma )^{-1/2} \sqrt{2\pi}/2.$ Take $R = D$ gives the desired result.
 \end{proof}
 
\begin{proposition}[{\cite[5.4.2]{bakry2014analysis}}, restated in {\cite[Lemma E.2]{lee2022convergence}} ] \label{prop:concentration lsi}
Suppose $ \pi: \R^d \to \R_{\geq 0}$ satisfies LSI with constant $1/\alpha.$
Let $f:\R^d \to \R$ be a $L$-Lipschitz function then
\[\P_{x\sim \pi} [\abs{f(x)-\E_{\pi} [f(x)]}\geq  t ] \leq \exp\left(-\frac{\alpha t^2}{2L^2}\right)\]

\end{proposition}

\begin{proposition}[Sub-Gaussian concentration of norm for strongly log concave measures] \label{prop:concentration of norm for log concave measure}
Let $V:\R^d \to \R$ be a $\alpha$-strongly convex and $\beta$-smooth function. Let $\kappa = \beta/\alpha .$ Let $\pi$ be the probability measure with $\pi(x) \propto \exp(-V(x)).$ Let $x_*=\arg\min_x V(x)$ then for $D = 5 \sqrt{\frac{d}{\alpha}}  \ln (10\kappa) $ we have
\[\P_{x\sim \pi}[\norm{x-x_*} \geq  D+ t] \leq \exp(-\alpha t^2/4) \]
thus by \cref{prop:moment bound for subgaussian concentration}, for $p\geq 1.$
 \[\E_{\pi}[\norm{x-x^*}^p]^{1/p} \leq O(1) \sqrt{p} \sqrt{\frac{d}{\alpha}}  \ln (10\kappa)^p \]
\end{proposition}
\begin{proof}
By \cite[Lemma E.3]{lee2022convergence}, let $\bar{x} = \E_{\pi} [x]$ then $ \norm{\bar{x}-x^*}\leq \frac{1}{2} \sqrt{\frac{d}{\alpha}} \ln (10 \kappa).$ By \cref{prop:concentration lsi}, for any unit vector $v\in \R^d,$ the function $\langle v, x-\bar{x} \rangle$ is 1-Lipschitz, since $\abs{\langle v, x\rangle - \langle v, y\rangle} \leq \sqrt{\norm{v}_2}\norm{x-y}_2 = \norm{x-y}_2. $ Thus, by \cref{prop:concentration lsi}, $\langle v, x-\bar{x}\rangle$ has mean $0$ and sub-Gaussian concentration for all unit vector $v,$ thus $x-\bar{x}\rangle$ is a sub-Gaussian random vector. From sub-Gaussianity, a standard argument (see e.g. Theorem 1.19 of \cite{rigollet2015high}) shows that
\[\P_{\pi}\left[\norm{x-\bar{x}}\geq 4\sqrt{\frac{d}{\alpha}} + t \right]\leq \exp(-\alpha t^2/4) \]
thus by triangle inequality, using that $ \norm{\bar{x}-x^*}\leq \sqrt{\frac{d}{\alpha}} \frac{1}{2}\ln (10 \kappa)$, we have
\[\P_{\pi}\left[\norm{x-x^*}\geq (4 + 1/2\ln (10\kappa)) \sqrt{\frac{d}{\alpha}}   + t \right] \leq \P_{\pi}\left[\norm{x-\bar{x}}\geq 4\sqrt{\frac{d}{\alpha}} + t \right] \leq \exp(-\alpha t^2/4)\]
\end{proof}
\begin{proposition}[Normalization factor bound] \label{prop:normalization factor bound}
Let $V:\R^d \to \R$ be a $\alpha$-strongly convex and $\beta$-smooth function. Let $\pi $ be the probability measure defined by $\pi(x) \propto \exp(-V(x))$ and $Z:=Z_{\pi} = \int \exp(-V(x)) dx$ be its normalization factor. 
For any $ y\in \R^d$
\[ \exp\left(-V(y) + \frac{\norm{\nabla V(y)}^2}{2\beta} \right) (2\pi \beta^{-1})^{d/2} \leq Z \leq \exp
\left(-V(y) + \frac{\norm{\nabla V(y)}^2}{2\alpha} \right) (2\pi \alpha^{-1})^{d/2}\]
Let $ y= x^*= \arg\min V(x)$ and assume w.l.o.g. $V(y) = 0$ gives
\[\frac{d}{2} \ln \frac{1}{\beta} \leq \ln Z_{\pi} - \frac{d}{2}  \ln (2\pi)  \leq \frac{d}{2} \ln\frac{1}{\alpha} \]
\end{proposition}
\begin{proof}
Since $ \alpha I\preceq \nabla^2 V(x) \preceq \beta I,$
 \[\langle \nabla V(y), x-y\rangle + \alpha\norm{x-y}^2/2 \leq V(x) - V(y) \leq \langle \nabla V(y), x-y\rangle + \beta\norm{x-y}^2/2\]
\begin{align*}
 Z&\leq \int \exp(-V(y) - \langle \nabla V(y), x-y\rangle -\alpha\norm{x-y}^2/2) dx \\
&=\exp\left(-V(y) + \frac{\norm{\nabla V(y)}^2}{2\alpha} \right) \int \exp \left(-\frac{ \alpha \norm{ (x-y) +\alpha^{-1} \nabla V(y)}^2 }{2}\right) dx \\
&= \exp\left(-V(y) + \frac{\norm{\nabla V(y)}^2}{2\alpha} \right) (2\pi \alpha^{-1})^{d/2}
 \end{align*}
The lower bound follows similarly.
The second statement follows from the first since $\nabla V(x^*) = 0.$

\end{proof}

\subsection{Girsanov's theorem}
\begin{theorem}[Girsanov's Theorem {\cite[Chapter 3.5]{karatzas1991brownian}}] \label{thm:girsanov}
    Let $(X_t)_{t \geq 0}$ be stochastic processes adapted to the same filtration. Let $ P_T$ and $Q_T$ be probability measure on the path space $C([0,T]; \R^d)$ s.t. $X_t$ evolved according to
\begin{align*}
    d X_t &= b^P_t dt + \sqrt{2} d B_t^P \text{ under } P_T\\
    d X_t &= b^Q_t dt + \sqrt{2} d B_t^Q \text{ under } Q_T
\end{align*}
Assume that Novikov's condition 
\begin{equation} \label{ineq:novikov}
    \E_{Q_T} \left[\exp \left(\frac{1}{4} \int_0^T \norm{ b^P_t - b^Q_t }^2 dt \right) \right] < \infty
\end{equation}
holds. Then
\begin{equation}
    \frac{d P_T}{dQ_T} = \exp \left(
    \int_0^T \frac{1}{\sqrt{2}} \langle b^P_t - b^Q_t, d B_t^Q \rangle  - \frac{1}{4}\int_0^T \norm{ b^P_t - b^Q_t }^2 dt\right)
\end{equation}
\end{theorem}
\begin{lemma}[Application of Girsanov with approximation argument {\cite[Equation 5.5, Proof of Theorem 9]{chen2023sampling}}] \label{lem:approximation argument}
    Let $(X_t)_{t \geq 0}$ be stochastic processes adapted to the same filtration. Let $ P_T$ and $Q_T$ be probability measure on the path space $C([0,T]; \R^d)$ s.t. $X_t$ evolved according to
\begin{align*}
    d X_t &= b^P_t dt + \sqrt{2} d B_t^P \text{ under } P_T\\
    d X_t &= b^Q_t dt + \sqrt{2} d B_t^Q \text{ under } Q_T
\end{align*}
Suppose  $\E_{Q_T} [ \int_0^T \norm{ b^P_t - b^Q_t }^2 dt  ]< \infty$ then
\[2d_{TV}(Q_T || P_T)^2 \leq  \D_{\KL}(Q_T || P_T) \leq \E_{Q_T} \left[\int_0^T \norm{ b^P_t - b^Q_t }^2 dt  \right] \]
\end{lemma}
\begin{lemma}[Corollary of \cref{thm:girsanov}, {\cite[Corollary 20]{chewi2021analysis}}] 
\label{lem:bound Ht}
With the setup and preconditions in \cref{thm:girsanov},
    For any event $\mathcal{E},$
    \[ \E_{Q_T} \left[\left(\frac{d P_T}{d Q_T}\right)^q \textbf{1}_{\mathcal{E}} \right] \leq \sqrt{\E_{Q_T} \left[\exp\left( q^2 \int_0^T\norm{ b^P_t - b^Q_t }^2 dt \right)  \textbf{1}_{\mathcal{E}}  \right] }  \]
\end{lemma}
\subsection{Mixture potential}
\paragraph{Notation for indexing components.} Let $I = [K]$ be the set of indices $i$ for the components $\mu_i$ of the mixture distribution $\mu$.
We will need to work with subsets $S$ of $I$ and the mixture distribution forms by components $\mu_i$ for $i\in S.$
\begin{definition} \label{def:subset mixture}
    For $S\subseteq I,$ let $p_S = \sum_{i\in S} p_i,$  and $ \mu_S = p_S^{-1} \sum_{i\in S}p_i \mu_i. $ Let $V_S = -\log \mu_S.$

    If $S = I$ we omit the subscript $S.$
\end{definition}

\paragraph{Derivative computations.}
For future use, we compute the derivatives of $V.$
\begin{proposition}[Gradient of $V$] \label{prop:gradient}
\begin{equation}
\nabla V(x) = \frac{\sum p_i \mu_i(x) \nabla V_i(x) }{\mu(x)}
\end{equation}
Consequently, $ \norm{\nabla V(x)} \leq \max \norm{\nabla V_i(x)}.$
\end{proposition}
\begin{proof}
The statement follows from
\begin{align*}
    \nabla V(x) = \nabla \log \mu(x) = \frac{\nabla \mu(x)}{\mu(x)}
\end{align*}
and
\begin{align*}
    \nabla \mu(x)= \nabla (\sum p_i Z_i^{-1} \exp(-V_i(x)) = -\sum p_i\mu_i(x) \nabla V_i(x).
\end{align*}
\end{proof}
\begin{proposition}[Hessian of $V$] \label{prop:hessian}
\begin{equation}
\nabla^2 V(x) = \frac{\sum_i p_i \mu_i(x) \nabla^2 V_i(x) }{\mu(x)} -\sum_{i,j} \frac{p_i p_j \mu_i(x) \mu_j(x) (\nabla V_i(x) - \nabla V_j(x))(\nabla V_i(x) - \nabla V_j(x))^\top  }{4 \mu^2(x)}
\end{equation}
hence if $ \nabla^2 V_i\preceq\beta I$ for all $i\in I$ then $\nabla^2 V(x) \preceq \beta I.$
\end{proposition}
\begin{proof}
Let $Z_i = \int \exp(-V_i(x)) dx $ be the normalization factor of $\mu_i.$ Note that
\begin{align*}
   &\nabla (\mu_i(x) \nabla V_i(x)) \\
   &= \nabla (Z_i^{-1}\exp(-V_i(X)) \nabla V_i(x)) = Z_i^{-1} \exp(-V_i(x)) (-\nabla V_i(x) \nabla V_i(x)^\top + \nabla^2 V_i(x))\\
   &= \mu_i(x) (\nabla^2 V_i(x) -\nabla V_i(x) \nabla V_i(x)^\top ) 
\end{align*}
and $\nabla \mu(x) =-\sum p_i\mu_i(x) \nabla V_i(x),$ thus
\begin{align*}
&\nabla^2 V(x)\\
&= \frac{\nabla (\sum_i p_i \mu_i(x) \nabla V_i(x)) }{\mu(x)} - \frac{(\sum p_i \mu_i(x) \nabla V_i(x) ) \nabla\mu(x) }{\mu^2(x)}\\
&= \frac{\sum p_i \mu_i(x) (\nabla^2 V_i(x) -\nabla V_i(x) \nabla V_i(x)^\top )  }{\mu(x)} +\frac{(\sum p_i \mu_i(x) \nabla V_i(x) ) (\sum p_i \mu_i(x) \nabla V_i(x) )^\top }{\mu^2 (x)}
\end{align*}
Next,
\begin{align*}
    \MoveEqLeft (\sum p_i \mu_i(x) \nabla V_i(x) ) (\sum p_i \mu_i(x) \nabla V_i(x) )^\top - (\sum p_i \mu_i \nabla V_i(x) \nabla V_i(x)^\top) (\sum p_i \mu_i)\\
    &=\sum_{i, j} p_i p_j \mu_i(x)\mu_j(x) \nabla V_i(x) \nabla V_j^\top - \sum_{i,j}  p_i p_j \mu_i(x)\mu_j(x) \nabla V_i(x) \nabla V_i(x)^\top\\
    &=\frac{1}{2}\sum_{i\neq j} p_i p_j \mu_i(x)\mu_j(x) (\nabla V_i(x) \nabla V_j^\top  + \nabla V_j(x) \nabla V_i^\top - \nabla V_i(x) \nabla V_i^\top-\nabla V_j(x) \nabla V_j^\top )\\
    &= -\frac{1}{2}\sum_{i\neq j} p_i p_j \mu_i(x)\mu_j(x) (\nabla V_i(x) -\nabla V_j(x))(\nabla V_i(X) -\nabla V_j(x))^\top  
\end{align*}
thus the first statement follows. The second statement follows from noticing that $ (\nabla V_i(x) -\nabla V_j(x))(\nabla V_i(X) -\nabla V_j(x))^\top \succeq 0.$
\end{proof}

\subsection{Properties of smooth and strongly log-concave distribution}
We record the consequences of $\alpha$-strongly log-concave and $\beta$-smooth that we will use.
\begin{lemma}\label{lem:single distribution}
    Suppose $\mu_i$ is $\alpha$-strongly log-concave and $\beta$-smooth then for $\kappa = \beta/\alpha,$ $u_i = \arg\min V_i(x),$ $D = 5\sqrt{\frac{d}{\alpha}} \ln (10 \kappa), $ and $c_z = \frac{d}{2}\ln\kappa ,$ we have
\begin{enumerate}
    \item \label{item:smooth}For all $x:$  $ \norm{\nabla^2 V_i(x)}_{OP} \leq \beta$ and $\norm{ \nabla V_i(x)} \leq \beta  \norm{x-u_i} $
    \item \label{item:log concave} $ \alpha \norm{x-u_i}^2 \leq V_i(x) \leq  \beta  \norm{x-u_i}^2.$ 
    
    Consequently, for $Z_i = \int \mu_i(x) dx,$ there exists $z_+\leq z_-$ with $ z_+ = z_- - c_z$ s.t.  \[\exp(-\beta \norm{x- u_i}^2 - z_- ) \leq \mu_i(x) = Z_i^{-1} \exp(-V_i(x)) \leq \exp(-\alpha \norm{x- u_i}^2 - z_+)\]
\item \label{item:concentration} Sub-gaussian concentration:   \[\P[\norm{x-u_i} \geq D+ t] \leq \exp(-\alpha t ^2/4 ) \]
By \cref{prop:moment bound for subgaussian concentration}, this implies that for all $p$
\[\E_{\mu_i}[\norm{x-u_i}^p ] \lesssim_p D^p.\] 
\item \label{item:log sobolev} $\mu_i$ satisfies a LSI with constant $C_{LS} = \frac{1}{\alpha}.$
\end{enumerate}     
\end{lemma}
\begin{proof}
   This is due to \cref{prop:normalization factor bound} and \cref{prop:moment bound for subgaussian concentration}, and the fact that $ \nabla V_i(u_i)=0$ for $u_i = \arg\min V_i(x).$
\end{proof}
\subsection{Basic mathematical facts}
\begin{proposition} \label{prop:composite of exponential and polynomial terms}
For any constant $a > 0, b, p\in \N_{\geq 0}$
$f(x) = \exp (-a x - b)  x^p$ is decreasing on $[p/a, +\infty)$ 
\end{proposition}
\begin{proof}
Let $g(x) = \log f(x) = -ax- b + p \log x$ and observe that
\[g'(x) = -a + p/x  \leq 0 \] when $ x \geq p/a$, so the claim follows by integrating. 
\end{proof}

\begin{proposition}\label{prop:tv distance mixture bound}
Let $P_1, \dots, P_k, Q_1, \dots, Q_k$ be distributions s.t. $d_{TV} (P_i, Q_i) \leq \epsilon_i.$
Let $\alpha_1, \cdots, \alpha_k, \beta_1, \cdots, \beta_k$ be s.t. $\alpha_i, \beta_i \geq 0\forall i$ and $\sum_i \alpha_i = \sum_i \beta_i = 1.$  Then
\[d_{TV}(\sum_i \alpha_i P_i, \sum_i \alpha_i Q_i) \leq \sum_i \alpha_i \epsilon_i \]
and
\[d_{TV}(\sum_i \alpha_i Q_i, \sum_i \beta_i Q_i)\leq \frac{1}{2} \sum_{i} \abs{\alpha_i -\beta_i} \]
\end{proposition}
\begin{proof}
By triangle inequality
\begin{align*}
  2 d_{TV} (\sum_i \alpha_i P_i, \sum_i \alpha_i Q_i) &=
\int_{x\in \Omega} \abs{\sum_{i} \alpha_i P_i(x) - \sum_i \alpha_i Q_i(x)} dx \\
&\leq \int_{x\in \Omega}\sum_{i} \alpha_i \abs{P_i(x) -Q_i(x)} dx = 2\sum_{i} \alpha_i  d_{TV}(P_i, Q_i)
\end{align*}
Similarly,
\begin{align*}
  2 d_{TV} (\sum_i \alpha_i Q_i, \sum_i \beta_i Q_i) &=
\int_{x\in \Omega} \abs{\sum_{i} \alpha_i Q_i(x) - \sum_i \beta_i Q_i(x)} dx \\
&\leq \int_{x\in \Omega}\sum_{i}\abs{ \alpha_i-\beta_i}  Q_i(x) dx =\sum_{i} \abs{\alpha_i -\beta_i}
\end{align*}
\end{proof}

\section{Log-Sobolev Inequality for well-connected mixtures}\label{sec:mixture-sobolev}
In this section, we show that the mixture $\sum p_i \mu_i$ has a good log-Sobolev constant if its component distributions $\mu_i$ have high overlap. The below \cref{thm:log sobolev and poincare for mixture formal} corresponds to \cref{thm:log sobolev and poincare for mixture} of the main text. 
\begin{definition}
For distributions $\nu, \pi,$ let $\delta(\nu, \pi) =  \int \min \set{\nu(x), \pi (x)} dx$ be the overlap of $\nu$ and $\pi.$ Let $\delta_{ij} $ denote $\delta(\mu_i, \mu_j).$ 
Note that \[1- \delta(\nu, \pi) =  \int (\nu(x) -\min \set{\nu(x), \pi (x)} )dx = \int_{x: \nu(x) \geq \pi(x)} (\nu(x) - \pi(x)) dx = d_{TV}(\nu, \pi).  \]
\end{definition}
\begin{theorem} \label{thm:log sobolev and poincare for mixture formal}
Let $G$ be the graph on $I $ where $\set{i,j} \in E(G)$ iff $ \mu_i, \mu_j$ have high overlap i.e.
\[ \delta_{ij} :=\int \min \set{\mu_i(x), \mu_j(x)} dx\geq \delta.\]
Suppose $G$ is connected. Let $M\leq \abs{I}$ be the diameter of $G.$ The mixture distribution $\mu = \sum_{i\in I} p_i \mu_i$ has 
\begin{enumerate}
    \item Poincare constant \cite[Theorem 1.2]{Madras2002MarkovCD}
\[C_{PI} (\mu) \leq \frac{4M}{\delta}\max_{i\in I} \frac{C_{PI}(\mu_i)}{p_i} \]
\item Log Sobolev constant
\[C_{LS} (\mu) \leq \frac{4M C_{LS}(p)}{\delta} \max_{i} \frac{C_{LS} (\mu_i) }{p_i} \]
where for $p_*=\min_i p_i,$  $C_{LS}(p) = 1+ \log(p_*^{-1})$ is the log Sobolev constant of the instant mixing chain for $p.$ Hence
\[C_{LS} (\mu) \leq  C_{\abs{I},p_*} \delta^{-1} \max_i C_{LS}(\mu_i) \]
where $C_{\abs{I},p_*} = 4\abs{I}^2 (1+ \log(p_*^{-1}))  p_*^{-1}$ only depends on $ \abs{I}$ and $p_*$
\end{enumerate}
\end{theorem}

Below we fix a test function $f$ s.t. $\E_{\mu}[f^2] \leq \infty.$ 
Let
\begin{equation}\label{eq:covariance}
  C_{i,j} =\int \int (f(x)-f(y))^2 \mu_i(x) \mu_j(x) dx dy .  
\end{equation}

\begin{lemma}[Triangle inequality]
    \[C_{i_0,i_{\ell}}\leq \ell \sum_{j = 0}^{\ell-1} C_{i_j,i_{j+1}} \]
\end{lemma}
\begin{proof}
Without loss of generality, assume $i_j = j$ for all $j.$ Then
    \begin{align*}
        C_{i_0,i_{\ell}} &= \int \int (f(x_0)-f(x_{\ell}))^2 \mu_{0}(x_{0}) \mu_{\ell}(x_{\ell}) dx_0 dx_{\ell}\\
        &= \int_{x_0}  \dots \int_{x_{\ell}} (f(x_0)- f(x_1) + \dots + f(x_{\ell-1})-f(x_{\ell}))^2 \prod_{j=0}^{\ell} \mu_j(x_j) dx_0 d_{x_1} \dots dx_{\ell}\\
        &\leq  \int_{x_0}  \dots \int_{x_{\ell}} \ell \left(\sum_{j=0}^{\ell-1} (f(x_j)-f(x_{j+1}))^2 \right) \prod_{j=0}^{\ell} \mu_j(x_j) dx_0 d_{x_1} \dots dx_{\ell}\\
        &= \ell \sum_{j=0}^{\ell-1} \int_{x_j} \int_{x_{j+1}}  (f(x_j)-f(x_{j+1}))^2 \mu_j(x_j) \mu_{j+1} x_{j+1} d_{x_j} d_{x_{j+1}}
        = \ell \sum_{j=0}^{\ell-1} C_{j,j+1}
    \end{align*}
    where the inequality is Holder's inequality.
\end{proof}
The following comes from \cite[Proof of Theorem 1.2]{Madras2002MarkovCD}
\begin{lemma}\label{lem:compare pair}
If $ \int \min \set{\mu_i(x), \mu_j(x)} dx \geq \delta$ then
\[ C_{i,j} \leq \frac{2(2-\delta)}{\delta} (\Var_{\mu_i} (f) + \Var_{\mu_j} (f)).\]
\end{lemma}
\begin{proposition}[Variance decomposition]\label{prop:variance decomposition}
\begin{align*}
  2 \Var_{\mu}(f) &= \int_x \int_y (f(x) -f(y))^2 \mu(x) \mu(y) dx dy\\
  &= \sum_{i, j} p_i p_j C_{ij}\\
  &= 2 \sum_i p_i^2\Var_{\mu_i}(f) +2 \sum_{i< j} p_i p_j C_{ij} 
\end{align*}
\end{proposition}
\begin{proof}
    \begin{align*}
       \Var_{\mu}(f) &= \int_x \mu(x) f^2(x) dx - \left(\int_x \mu(x) f(x) dx\right)^2 \\
       &= \int_x\int_y  f^2(x) \mu(x) \mu(y) dx dy - \int_x \int_y  \mu(x) f(x) \mu(y) f(y) dx dy\\
       &=\frac{1}{2} \int_x \int_y \mu(x) \mu(y) (f^2(x) + f^2(y)-2 f(x) f(y)) dx dy\\
       &= \frac{1}{2}\int_x \int_y \mu(x) \mu(y) (f(x)-f(y))^2 dx dy
\end{align*}

Since $\mu(x) = \sum_i p_i \mu_i(x),$ we can further rewrite
\begin{align*}
    2\Var_{\mu}(f) &=\int_x \int_y  (f(x)-f(y))^2 \left(\sum_{i} p_i \mu_i(x)\right) \left(\sum_i p_i \mu_i(y) \right) dx dy\\
    &= \int_x \int_y  (f(x)-f(y))^2 \left(\sum_{i, j} p_i p_j \mu_i(x)\mu_j(y) \right) dx dy\\
    &= \sum_{i,j} p_i p_j \int_x \int_y (f(x)-f(y))^2 \mu_i(x) \mu_j(y) dx dy\\
    &= \sum_{i,j} p_i p_j C_{ij}\\
    &= \sum_i p_i^2 C_{ii} +\sum_{i < j} (C_{ij } + C_{ji})\\
    &= 2\sum_i p_i^2 \Var_{\mu_i}[f] + 2\sum_{ij} C_{ij}
\end{align*}
where the last equality is because $C_{ij} = C_{ji}.$
\end{proof}
\begin{lemma}\label{lem:bound sum cij}
For $i, j$ let $\gamma_{ij}$ be the shortest path in $G$ from $i$ to $j$ and let $\abs{\gamma_{ij}}$ be its length i.e. the number of edges in that path. For $u,v$, let $uv$ denote the edge $\set{u,v}$ of $G$ if it is in $E(G).$ Let $D=\max_{ij} \abs{\gamma_{ij}}$ be the diameter of $G.$ Then
    \begin{align*}
        \sum_{i< j} p_i p_j C_{ij}&\leq \sum_{uv \in E(G)} \left(C_{uv}
        \sum_{i<j: uv\in\gamma_{ij} } p_i p_j \abs{\gamma_{ij}}\right)\\
        &\leq \frac{M(2-\delta)}{\delta} \sum_{u}\Var_{\mu_u} (f)\\
        &\leq \frac{M(2-\delta)}{\delta} \sum_{u} C_{PI}(\mu_u)\E_{\mu_u}[\norm{\nabla f}^2]
    \end{align*}
\end{lemma}
\begin{proof}
    \begin{align*}
        \sum_{i< j} p_i p_j C_{ij}&\leq \sum_{uv \in E(G)} (C_{uv}
        \sum_{i<j: uv\in\gamma_{ij} } p_i p_j \abs{\gamma_{ij}})\leq  M\sum_{uv \in E(G)}(C_{uv} \sum_{i<j: uv\in\gamma_{ij}}p_i p_j)\\
     \end{align*}
    By \cref{lem:compare pair} and the definition of $G$, $C_{uv} \leq \frac{2(2-\delta)}{\delta}(\Var_{\mu_u} (f) + \Var_{\mu_v} (f) ),$ thus
    \begin{align*}
         \sum_{i< j} p_i p_j C_{ij} &\leq  \frac{2M(2-\delta)}{\delta}\sum_{uv \in E(G)} \left[ (\Var_{\mu_u} (f) + \Var_{\mu_v} (f) )  \sum_{i< j: uv\in\gamma_{ij}}p_i p_j\right]\\
         &\leq \frac{2M(2-\delta)}{\delta}\sum_{u} \left[ \Var_{\mu_u} (f)\sum_{v, uv\in E(G), i< j: uv\in\gamma_{ij}}p_i p_j\right]\\
         &= \frac{2M(2-\delta)}{\delta}\sum_{u} \left[ \Var_{\mu_u} (f) \sum_{i<j: u\in \gamma_{ij}}p_i p_j\right]\\
         &\leq \frac{M(2-\delta)}{\delta} \sum_{u}\Var_{\mu_u} (f)\\
         &\leq \frac{M(2-\delta)}{\delta} \sum_{u} C_{PI}(\mu_u)\E_{\mu_u}[\norm{\nabla f}^2]
    \end{align*}
\end{proof}

\begin{proposition} 
For $C_{i,j}$ be as in \cref{eq:covariance}
\label{prop:rewrite cij}
    \[C_{i,j}= \frac{1}{2} (\Var_{\mu_i}(f) + \Var_{\mu_j}(f) +( \E_{\mu_i} [f] -\E_{\mu_j} [f] )^2) \]
\end{proposition}
\begin{proof}
    Let $\nu = \frac{1}{2}\mu_i + \frac{1}{2} \mu_j.$ We write $\Var(\nu)$ in two ways.
First, $\E_{\nu}[f] = \frac{1}{2} (E_{\mu_i}[f] +E_{\mu_j}[f]) $ thus
\begin{align*}
  Var_{\nu}(f) &= \E_{\nu}[f^2] - (\E_{\nu}[f])^2 = \frac{1}{2}(\E_{\mu_i}[f^2] +\E_{\mu_j}[f^2]) - \frac{1}{4} (E_{\mu_i}[f] +E_{\mu_j}[f])^2\\
  &= \frac{1}{2}\sum_{k\in \set{i,j}} (\E{\mu_k}[f^2] - (\E_{\mu_k}[f])^2) + \frac{1}{4} (E_{\mu_i}^2[f] + E_{\mu_j}^2{\mu_j}[f]  -2 E_{\mu_i}[f] E_{\mu_j}[f])\\
  &= \frac{1}{2}(\Var_{\mu_i}[f]+\Var_{\mu_j} [f]) + \frac{1}{4} (\E_{\mu_i}[f]-\E_{\mu_j}[f])^2
\end{align*}
On the other hand, by \cref{prop:variance decomposition}, 
\[ \Var_{\nu}(f) = \frac{1}{4}(\Var_{\mu_i}(f) + \Var_{\mu_j} (f)) + \frac{1}{2} C_{ij}  \]
    Rearranging terms gives the desired equation.
\end{proof}
\begin{proposition} \label{prop:entropy decomposition}
Let $ g\equiv f^2.$ Let the projection of $g$ on $I$ be defined by $\bar{g}(i) = \E_{\mu_i}[g]. $ Then 
\[ \Ent[f^2] =  \sum_{i\in I} p_i \Ent_{\mu_i}[f^2]  + \Ent_{p} [\bar{g}] \]
\end{proposition}
\begin{proof}
\begin{align*}
    \Ent[f^2] &= \int \mu(x) g(x) \log g(x)dx - \E_{\mu}[g(X)] \log(\E_{\mu}[g(x)]) \\ 
    &= \int \left(\sum_i p_i \mu_i(x)) g(x) \log g(x) dx -  \E_{\mu}[g(x)] \log(\E_{\mu}[g(x)]\right) \\
    &= \sum_i p_i \left(\int\mu_i (x) g(x)\log g(x) dx - \E_{\mu_i}[g(x)] \log (\E_{\mu_i}[g(x)])\right)   \\
    &+ \sum_i p_i \bar{g}(i) \log \bar{g}(i) - \E_{\mu}[g(x)] \log(\E_{\mu}[g(x)])
\end{align*}
   where in the last equality, we use the definition of $\bar{g}(i).$ Note that \[\E_{i\sim p} [\bar{g}(i)]=\sum_i p_i \bar{g}(i) = \sum_i\left( p_i\int\mu_i(x) g(x) dx \right)= \int \left(\sum_{i} p_i \mu_i\right) g(x) = \E_{\mu} [g(x)]\]
   thus
   \[\Ent[f^2] = \sum_i p_i \Ent_{\mu_i}[f^2] + \Ent_{i \sim p} [\bar{g}(i)]\]
\end{proof}
\begin{proposition} Let $\bar{g}$ be defined as in \cref{prop:entropy decomposition}, then
    \begin{align*}
        (\sqrt{\bar{g}}(i) - \sqrt{\bar{g}}(j))^2 \leq \Var_{\mu_i}[f^2] + \Var_{\mu_j}[f^2] + ( \E_{\mu_i} [f] -\E_{\mu_j} [f] )^2 = 2 C_{ij}
    \end{align*}
\end{proposition}
\begin{proof}
    The first inequality comes from \cite[Proof of Lemma 3]{Schlichting_2019} and the second part from \cref{prop:rewrite cij}.
\end{proof}
\begin{proposition}[Log Sobolev inequality for the instant mixing chain, {\cite[Theorem A.1]{diaconis1996logarithmic}}]  \label{prop:compare ent and var}
Let $p$ be the distribution over $I$ where the probability of sampling $i\in I$ is $p_i.$
For a function $ h : I \to \R_{\geq 0} $
    \[\Ent_p[h] \leq C_p \Var_p [\sqrt{h}]\]
    with $C_p = \ln (4 p_*^{-1})$ with $ p_*=\min_{i}p_i.$ 
\end{proposition}
\begin{lemma} \label{lem:bounding variance sqrt g}
With $\bar{g}$ defined as in \cref{prop:entropy decomposition},
   \begin{align*}
       \Var_{p} [\sqrt{\bar{g}}] = \sum_{i<j} p_i p_j  (\sqrt{\bar{g}}(i) - \sqrt{\bar{g}}(j))^2 \leq 2\sum_{i<j} p_i p_j C_{ij} 
   \end{align*} 
\end{lemma}
\begin{proof}[Proof of \cref{thm:log sobolev and poincare for mixture formal} part 2]
We can rewrite
   \begin{align*}
       \Ent_{\mu}[f^2] &= \sum_{i\in I} p_i \Ent_{\mu_i}[f^2]  + \Ent_{p} [\bar{g}] \\
       &\leq_{(1)} \sum_{i} p_i C_{LS}(\mu_i) \E_{\mu_i}[\norm{\nabla f}^2] + C_{LS}(p) \Var_p(\sqrt{\bar{g}})\\
       &\leq_{(2)}\sum_{i} p_i C_{LS}(\mu_i) \E_{\mu_i}[\norm{\nabla f}^2] + 2C_{LS}(p)\sum_{i<j} p_i p_j C_{ij}\\
       &\leq_{(3)} \sum_{i} p_i C_{LS}(\mu_i) \E_{\mu_i}[\norm{\nabla f}^2] +\frac{2M(2-\delta)C_{LS}(p) }{\delta} \sum_{u} C_{PI}(\mu_u)\E_{\mu_u}[\norm{\nabla f}^2]\\
       &\leq_{(4)} \frac{4M C_{LS}(p)}{\delta} \max_{i} \set{\frac{C_{LS} (\mu_i) }{p_i}} \sum_{i} p_i \E_{\mu_i}[\norm{\nabla f}^2] \\
       &=  \frac{4M C_{LS}(p)}{\delta} \max_{i} \set{\frac{C_{LS} (\mu_i) }{p_i}} \E_{\mu}[\norm{\nabla f}^2]  
   \end{align*}
   where (1) is due to definition of $C_{LS}(\mu_i)$ and \cref{prop:compare ent and var}, (2) is due to \cref{lem:bounding variance sqrt g}, (3) is due to \cref{lem:bound sum cij}, and (4) is due to $ C_{PI}(\mu_i) \leq C_{LS}(\mu_i)$ and $C_{LS}(p) , M \geq 1.$
\end{proof}
\section{Initialization Analysis} \label{sec:initialization}
For the continuous Langevin diffusion $(\bar{X}_t)_{t\geq 0}$ initialized at a bounded support distribution $\nu_0,$ we bound $\Renyi_q(\mathcal{L}(\bar{X}_h)|| \mu)$  for some small $h.$ Consequently, for $\mu$ being the stationary distribution of the Langevin diffusion and satisfying a LSI with constant $C_{LS},$ we can use the fact that $\D_{\KL}(\mathcal{L}(\bar{X}_t ) || \mu) \leq \exp(-\frac{t-h}{C_{LS}}) \D_{\KL}(\mathcal{L}(\bar{X}_h) ||\mu)$  to show that $\bar{X}_t$ converges to $\mu.$ 
\begin{lemma}[Initialization bound]\label{lem:continuous initialization}
Let $\mu = \sum_{i\in I}p_i \mu_i$ be a mixture of distributions $\mu_i\propto\exp(-V_i(x))$ which are $\alpha$-strongly log concave and $\beta$-smooth. Let $V(x) = -\ln \mu(x).$
Let $ (\bar{\nu}_t)_{t\in [0,h]}, (\nu_t)_{t\in [0,h]}$ be respectively the distribution of the continuous Langevin diffusion and the LMC with step size $h$ and score function $\nabla V$ initialized at $\delta_x$. Let $G(x):=\max_i \norm{\nabla V_i(x)}.$
Suppose $h \leq 1/(30\beta) $ then for $q \in (2, \frac{1}{10\beta h }),$
\[ \Renyi_q  (\bar{\nu}_h || \nu_h)
\leq O(q^2 h (G^2(x) + \beta^2 dh)) , \]
\[\Renyi_{q-1} (\nu_h || \mu) \leq \frac{d}{2} \ln((2\alpha h)^{-1})  + \alpha^{-1} G(x) \]
and
\[R_{q/2}(\bar{\nu}_h ||  \mu) \leq O(q^2 h (G^2(x) + \beta^2 dh)) + \frac{d}{2} \ln((2\alpha h)^{-1})  + \alpha^{-1} G^2(x) \]
If we replace $\delta_x$ with any $\nu_0$ then by weak convexity of Renyi divergence (\cref{lem:weak convexity}), the claim holds when we replace $G(x)$ with $G_{\nu} = \sup_{x\in \text{supp}(\nu_0)} G(x).$
\end{lemma}

\begin{proposition} \label{prop:discrete initilization bound}
    Let $\nu = \mathcal{N}(y, \sigma^2 I).$ If $\pi(x)\propto \exp(-W(x))$ is $\alpha$-strongly log concave and $\beta$-Lipschitz and $\sigma^2 \beta \leq 1/2$ then
    \[ \Renyi_{\infty} (\nu || \pi) \leq -\frac{d}{2} \ln (\alpha \sigma^2) + \norm{\nabla W(y)}^2/\alpha   \]
\end{proposition}
\begin{proof}
Since $ \alpha I\preceq \nabla^2 W(x) \preceq \beta I,$
 \[\langle \nabla W(y), x-y\rangle + \alpha\norm{x-y}^2/2 \leq W(x) - W(y) \leq \langle \nabla W(y), x-y\rangle + \beta\norm{x-y}^2/2\]
 By \cref{prop:normalization factor bound}, we can upper bound the normalization factor $ Z = \int \exp(-W(x)) dx$ by $ \exp\left(-W(y) + \frac{\norm{\nabla W(y)}^2}{2\alpha} \right) (2\pi \alpha^{-1})^{d/2}.$
 
 For $x\in \R^d,$ using the upper bound on $Z$
\begin{align*}
    \nu(x)/\pi(x) &= (2\pi \sigma^2)^{-d/2}  Z \exp\left(-\frac{\norm{x - y}^2}{2\sigma^2} + W(x)\right)   \\
    &\leq (\alpha \sigma^2)^{-d/2} \exp \left(W(x) - W(y) +  \frac{\norm{\nabla W(y)}^2}{2\alpha}  -\frac{\norm{x - y}^2}{2\sigma^2}\right)\\
    &= (\alpha \sigma^2)^{-d/2} \exp ^{\norm{\nabla W(y)}^2(\frac{1}{2\alpha} +\frac{\sigma^2}{2(1-\beta\sigma^2) }) } \exp ^{ -(\sqrt{\frac{(1-\beta\sigma^2) \norm{x-y}^2 }{2 \sigma^2}} - \sqrt{\frac{\sigma^2\norm{\nabla W(y)}^2}{2(1-\beta\sigma^2) }})^2 )}\\
    &\leq (\alpha \sigma^2)^{-d/2} \exp \left(\norm{\nabla W(y)}^2\frac{1- (\beta-\alpha) \sigma^2}{2\alpha (1-\beta\sigma^2) }\right) \\
    &\leq (\alpha \sigma^2)^{-d/2} \exp (\norm{\nabla W(y)}^2/\alpha)
\end{align*}
where the last inequality follows from $ 1/2\leq 1-\beta \sigma^2 \leq 1-(\beta-\alpha) \sigma^2 \leq 1.$   
\end{proof}

\begin{proof}[Proof of \cref{lem:continuous initialization}]
    We apply \cref{thm:girsanov} with $ T= h, P_T = (\bar{\nu}_t)_{t\in [0,h]}$ and $Q_T = (\nu_t)_{t\in [0,h]}.$ Note that, $b^P_t = -\nabla V(X_t) $ and $b^Q_t = -\nabla V(x).$  
    We first check that Novikov's condition \cref{ineq:novikov} holds.
    \begin{align*}
        \E_{Q_T} \left[\exp \left(\frac{1}{4} \int_0^T \norm{ b^P_t - b^Q_t }^2 dt \right) \right]
    =\E \left[\exp \left(\frac{1}{4} \int_0^h \norm{ \nabla V(X_t) - \nabla V(x) }^2 dt \right) \right]
    \end{align*}
    with $(X_t)_{t\in [0,h]}$ be the solution of the interpolated Langevin process i.e.
    \[X_t -x = -t \nabla V(x) + \sqrt{2} B_t\]
    By $\beta$-Lipschitzness of $ \nabla V_j$
    \[\norm{\nabla V_j(X_t)} - \norm{\nabla V_j(x)}  \leq  \beta_j \norm{X_t -  x} \leq \beta t\norm{\nabla V(x)} + \beta \sqrt{2} \norm{B_t} \]
    thus
    \begin{align*}
      \norm{\nabla V(X_t)}\leq G(X_t) = \max_{j\in I} \norm{\nabla V_j (X_t)} &\leq G(x) + \beta t G(x) +  \beta \sqrt{2} \sup_{t\in [0,h]} \norm{B_t} \\
      &\leq 1.1 G(x) +  \beta \sqrt{2} \sup_{t\in [0,h]} \norm{B_t}
    \end{align*}
    and
    \begin{equation}
    \begin{split}
        \int_0^h \norm{ \nabla V(X_t) - \nabla V(x) }^2 dt   
         &\leq 2 \int_0^h (\norm{ \nabla V(X_t)} ^2 +\norm{ \nabla V(x) }^2) dt \\
        &\leq   h [2 (1.1 G(x))^2 +   4 \beta^2  \sup_{t\in [0,h]} \norm{B_t}  + G(x)^2 ]\\
        &\leq 4 h G^2(x) + 4\beta^2 h  \sup_{t\in [0,h]} \norm{B_t}^2
    \end{split}
    \end{equation}
We first prove the following.
    \begin{proposition} \label{prop:bounding exponential moment}
        For any $\lambda < \frac{1}{8 \beta^2 h^2},$
        \[E_{Q_T}\left[\exp \left(\lambda \int_0^T \norm{ b^P_t - b^Q_t }^2 dt  \right) \right] \leq  \exp (4\lambda h G^2 (x) ) \left(\frac{1 + 8 \lambda\beta^2 h^2 }{ 1- 8 \lambda\beta^2 h^2 }\right)^d  .\]
    \end{proposition}
    \begin{proof}
    By \cref{prop:brownian}, for $\lambda \leq \frac{1}{16\beta^2 h^2}$
    \begin{align*}
    \E \left[\exp \left(\lambda \int_0^h \norm{ \nabla V(X_t) - \nabla V(x) }^2 dt \right) \right]
    &\leq \E \left[\exp \left( 4h \lambda G^2 (x) +4 \lambda \beta^2 h  \sup_{t\in [0,h]} \norm{B_t}^2 \right)\right] \\
    &\leq \exp (4 \lambda h G^2 (x)) \exp(6\beta^2 h^2 d \lambda) 
\end{align*}
    \end{proof}
Apply \cref{prop:bounding exponential moment} with $ \lambda = 1/4$ gives
\begin{align*}
     \E_{Q_T} \left[\exp \left(\frac{1}{4} \int_0^T \norm{ b^P_t - b^Q_t }^2 dt \right) \right]
    &=\E \left[\exp \left(\frac{1}{4} \int_0^h \norm{ \nabla V(X_t) - \nabla V(x) }^2 dt \right) \right]\\
    &\leq \exp \left( h G^2 (x)) \exp(1.5\beta^2 h^2 d \lambda\right) < \infty
\end{align*}
Next, let
\[H_t =\int_{0}^t  \frac{1}{\sqrt{2}} \langle b^P_s - b^Q_s, d B_s^Q \rangle  - \frac{1}{4}\int_0^t \norm{ b^P_s - b^Q_s }^2 ds \]
then $\frac{d P_t}{d Q_t} = \exp(H_t)$ and
\[d H_t = - \frac{1}{4} \norm{\nabla V(X_t) - \nabla V(x)}^2 dt + \frac{1}{\sqrt{2}} \langle -\nabla V(X_t) + \nabla V(x), d B_t^Q \rangle   \]
    By Ito's formula,
    \begin{align*}
      &d \exp (q H_t) \\
    = &\frac{q^2-q}{4} \exp(q H_t)  \norm{\nabla V(X_t) - \nabla V(x)}^2 +  q \exp(q H_t) \frac{1}{\sqrt{2}} \langle \nabla V(x)-\nabla V(X_t) , dB_t^Q\rangle
    \end{align*}
    Thus
\begin{align*}
    \E_{Q_T} [ \exp (q H_T) ] -1 &= \frac{q^2-q}{4} \E\left[\int_{0}^h \exp(qH_t) \norm{\nabla V(X_t) - \nabla V(x)}^2 dt  \right] \\
    &\leq \frac{q^2}{4} \int_{0}^h  \sqrt{\E[\exp (2q H_t) ] } \cdot \sqrt{\E[\norm{\nabla V(X_t) - \nabla V(x)}^4] } dt
\end{align*} 
We bound each term under the square root.
\begin{align*}
    \E[\norm{\nabla V(X_t) - \nabla V(x)}^4] 
    &\leq \E[(1.1 G(x) + \beta \sqrt{2} \sup_{t\in [0,h]} \norm{B_t}  + G(x))^4] \\
    &\leq 40 G^4 (x) + 32 \beta^4 \E[ \sup_{t\in [0,h]} \norm{B_t}^4]\\
    &\leq 40 G^4 (x) + O(\beta^4 d^2 h^2)
\end{align*}
By \cref{lem:bound Ht} and \cref{prop:bounding exponential moment}, if $q^2 <\frac{1}{100 \beta^2 h^2 }$  then
\begin{align*}
    (\E[\exp (2q H_t) ])^2 &\leq \E\left[\exp \left(4q^2 \int_0^h \norm{ \nabla V(X_t) - \nabla V(x) }^2 dt  \right) \right] \\
    &\leq  \exp(16 q^2 h  G^2(x)) \exp(24 q^2 \beta^2 h^2) \\
    &\leq \exp (16 q^2 h  G^2(x) + 72  q^2 \beta^2 h^2 d)
\end{align*}
Substitute back in gives
    
\[ \E_{Q_T} [ \exp (q H_T) ] -1\leq \frac{q^2 h}{4} ( 7 G^2(x) + O( \beta^2 d h) ) \exp (4 q^2 h  G^2(x) + 18  q^2 \beta^2 h^2 d) \]
By the data processing inequality
\begin{align*}
    \Renyi_q (\bar{\nu}_h ||\nu_h) &\leq \Renyi_q (P_T || Q_T) =\frac{\ln \E_{Q_T} [ \exp (q H_T) ]}{q-1} \\
    &\leq \ln \left( 1 +  \frac{q^2 h}{4} ( 7 G^2(x) + 6 C \beta^2 d h) \exp (4 q^2 h  G^2(x) + 18  q^2 \beta^2 h^2 d)  \right)\\
    &\leq \ln \left[\left(1+ \frac{q^2 h}{4} ( 7 G^2(x) + 6 C \beta^2 d h\right) \exp (4 q^2 h  G^2(x) + 18  q^2 \beta^2 h^2 d)  \right]\\
    &\leq \ln \left(1+ \frac{q^2 h}{4} ( 7 G^2(x) + 6 C \beta^2 d h) ) + (4 q^2 h  G^2(x) + 18  q^2 \beta^2 h^2 d\right)\\
    &\leq 6q^2h ( G^2(x) + (3+ C/2) \beta^2 d h ) 
\end{align*}
Now, note that $\nu_h = \mathcal{N}(y  ,  \sigma^2 I) $ with $y= x-h \nabla V(x)$ and $\sigma^2 = 2h.$ Note that $\norm{\nabla V_i(y)} \leq \norm{\nabla V_i(x)} + \beta\norm{y-x} \leq \norm{\nabla V_i(x)} + \beta h \norm{\nabla V(x)} \leq 1.1 G(x).$ By \cref{lem:weak convexity} and \cref{prop:discrete initilization bound}
\begin{align*}
    \Renyi_{2q-1} (\nu_h || \mu) \leq \max_{i} \Renyi_{2q-1} (\nu_h || \mu_i) &\leq \frac{d}{2} \ln((2\alpha h)^{-1})  + \alpha^{-1}\max_i \norm{\nabla V_i(y)}^2  \\
    &\leq \frac{d}{2} \ln((2\alpha h)^{-1})  + 2\alpha^{-1} G^2(x)
\end{align*}
The final statement follows from the weak triangle inequality (\cref{lem:weak triangle inequality}).
\end{proof}

\section{Perturbation Analysis}\label{sec:perturbation}
In this section, we bound the drift $\norm{\bar{X}_{t}-\bar{X}_{kh}}$ for $t\in [kh, (k+1) h]$ of the continuous Langevin diffusion $\bar{X}_t.$ These bounds will be used to bound the mixing time of the continuous Langevin diffusion and to compare the discrete LMC with the continuous process via Girsanov's theorem.

We will consider subset $S$ of $I$ such that the components $\mu_i$ for $i\in S$ have modes that are close together. We record the properties of the mixture distribution $\mu_S$ (see \cref{def:subset mixture} for definition) and and its log density function $V_S =-\log\mu_S$ in Assumption~\ref{assumption:cluster}. To be clear, we are defining this assumption as it is shared between multiple lemmas (and will be satisfied when we apply the lemmas), it is not a new assumption for the final result.  
\begin{assumption}[Cluster assumption] \label{assumption:cluster} We say a subset $S$ of $I$ satisfies the cluster assumption if 
there exists $u_S \in \R^d $, $A_{\text{Hess}, 1} , A_{\text{Hess}, 0},  A_{\text{grad}, 1} ,  A_{\text{grad}, 0} $ s.t.
\begin{enumerate}
    \item  $\norm{\nabla^2 V_S(x)}_{OP} \leq\min_{i\in S} A_{\text{Hess}, 1} \norm{x-u_i}^2 + A_{\text{Hess}, 0}$
    \item \label{item:cluster gradient bound} 
    $\norm{\nabla V_S(x)} \leq  A_{\text{grad}, 1} \norm{x-u_S} + A_{\text{grad}, 0}. $
\end{enumerate}

\end{assumption}
\begin{proposition} \label{prop:cluster bound with small distance betwen centers}
Suppose for all $i\in S$, $\mu_i$ satisfies item \ref{item:smooth} of \cref{lem:single distribution}.  Let $u_i$ and $D$ be as in \cref{lem:single distribution} and suppose $\norm{u_i-u_j}\leq L$ for $i, j\in S$ with $L\geq 10 D.$
Then $\mu_S$ satisfies  Assumption~\ref{assumption:cluster} with $ u_S= p_S^{-1}\sum_{i\in S} p_i u_i,$ $ A_{\text{grad}, 1} = \beta $, $A_{\text{grad}, 0} =  \beta L,$ $A_{\text{Hess}, 1} = 2\beta^2 $, $A_{\text{Hess}, 0} =  2 \beta^2 L^2.$

In addition, if $\mu_i$ satisfies item \ref{item:concentration} of \cref{lem:single distribution} then
\[\P_{\mu_S}[\norm{x-u_S}\geq 1.1 L + t] \leq \exp(-\alpha t^2/4).\]

\end{proposition}
\begin{proof}
First, $\forall i\in S: \norm{u_i - u_S }=p_S^{-1} \sum_{j\in S} p_j \norm{u_i - u_j}\leq L. $ By \cref{prop:gradient}
\begin{align*}
p_S \nabla V_S (x) = \sum_{i\in S} p_i \nabla V_i(x) 
&\leq \sum_{i\in S} p_i \beta \norm{x-u_i}  \\ 
&\leq \sum_{i\in S} p_i \beta (\norm{x-u_S} + \norm{u_i-u_S}) \leq p_S ( \beta  \norm{x-u_S} + L)
\end{align*}
We replace $I$ with $S$ and use the formula from \cref{prop:hessian}. By Holder's inequality \[\norm{\nabla V_i(x) - \nabla V_j(x)}^2 \leq 4 \max_{k\in S} \norm{\nabla V_k(x)}^2 \leq 4\beta^2 \max_{k\in S} \norm{x-u_k}^2 \leq 8\beta^2 \min_{k\in S} (\norm{x-u_k}^2 +L^2 )\]
Next, for $\tilde{p}_i=p_i/p_S,$ we have 
\[\sum_{i,j\in S} \tilde{p}_i \tilde{p}_j \mu_i(x) \mu_j(x) = \left(\sum_{i\in S} \tilde{p}_i\mu_i (x)\right)^2 = \mu_C^2(x)\]
thus
\[\beta I \succeq \nabla^2 V_C(x) \succeq 0 -I \max_{i,j\in S}  \norm{\nabla V_i(x) - \nabla V_j(x)}^2/4 \succeq- 2 I \beta^2 \min_{k\in S} (\norm{x-u_k}^2 +L^2 ). \]
For $\tilde{D} = D+L\leq 1.1 L$ and $\gamma = \frac{2}{\alpha}.$  
\begin{align*}
 \P_{\mu_S}[\norm{\bar{Z}-u_S} \geq\tilde{D} + \sqrt{\gamma \ln(1/\eta)} ) 
  &= p_S^{-1}\sum_{i\in S } p_i \mu_i(\bar{Z}: \norm{\bar{Z}-u_S} \geq\tilde{D} + \sqrt{\gamma \ln(1/\eta)})   \\
  &\leq  p_S^{-1}\sum_{i\in S } p_i \mu_i(\bar{Z}: \norm{\bar{Z}-u_i} \geq  D+ \sqrt{\gamma \ln(1/\eta)}) \\
  &\leq  p_S^{-1}\sum_{i\in S} p_i \eta = \eta
\end{align*}
where  first inequality is due to $\norm{u_i-u_S}\leq L$ for all $i\in S.$

\end{proof}

\begin{proposition} \label{prop:pertubation bound helper}
Suppose $S\subseteq I$ satisfies item 1 and item 2 of Assumption~\ref{assumption:cluster}.
Let $(\bar{Z}_t)_{t\geq 0}$ be the continuous Langevin diffusion with score $\nabla V_S$ initialized at $\bar{Z}_0\sim \nu_0$ then for $t \in [kh, (k+1) h)$
\begin{align*}
  &\E [\norm{\nabla V(\bar{Z}_{k h}) - \nabla V(\bar{Z}_{t}) }^2 ]  \\
  &\lesssim  \sqrt{\E[A_{\text{Hess}, 1}^4 (\norm{\bar{Z}_{kh} - u_S}^8 + \norm{\bar{Z}_{t} - u_S}^8 ) + A_{\text{Hess}, 0}^4  ]} \\
    &\quad \times  \sqrt{ (t-kh)^3\int_{kh}^t ( A_{\text{grad}, 1}^4  \E[\norm{\bar{Z}_s-u_S}^4] + A_{\text{grad}, 0}^4) ds + d^2 (t-kh)^2   }
\end{align*}
\end{proposition}
\begin{proof}
By the mean value inequality
\[||\nabla V_S(\bar{Z}_{k h}) - \nabla V_S(\bar{Z}_{t})||^2 \leq \norm{\bar{Z}_{k h}  - \bar{Z}_{t}} \max_{y= \eta\bar{Z}_{kh} + (1-\eta) \bar{Z}_t, \eta\in [0,1] } \norm{\nabla^2 V_S (y) }  \]
By item 1 of Assumption~\ref{assumption:cluster}, the fact that $y  = \eta \bar{Z}_{kh} + (1-\eta) \bar{Z}_t$ and Holder's inequality
\[\norm{\nabla^2 V_S (y)}_{OP} \leq A_{\text{Hess}, 1} \norm{y-u_S}^2 + A_{\text{Hess}, 0} \leq A_{\text{Hess}, 1} (\norm{\bar{Z}_{kh}-u_S}^2 + \norm{\bar{Z}_{t}-u_S}^2) + A_{\text{Hess}, 0}   \]
and so
\begin{align*}
  &\E[\norm{\nabla V_S(\bar{Z}_{k h}) - \nabla V_S(\bar{Z}_{t})}^2]\\
  \leq   & \E\left[\left (A_{\text{Hess}, 1} (\norm{\bar{Z}_{kh}-u_S}^2 + \norm{\bar{Z}_{t}-u_S}^2) + A_{\text{Hess}, 0}  \right)^2 \cdot  \norm{-\int_{kh}^t \nabla V_S(\bar{Z}_s)ds + 
  \sqrt{2} B_{t-kh}}^2\right] \\
  \leq &\sqrt{\E \left ( A_{\text{Hess}, 1} (\norm{\bar{X}_{kh}-u_S}^2 + \norm{\bar{X}_{t}-u_S}^2) + A_{\text{Hess}, 0}  \right)^4}\\
  & \cdot \sqrt{\E \norm{-\int_{kh}^t \nabla V_S(\bar{Z}_s)ds + \sqrt{2} B_{t-kh}}^4}.
\end{align*}
By item 2 of Assumption~\ref{assumption:cluster} and Holder's inequality, for $p=O(1)$
\begin{align*}
  &\E[\norm{-\int_{kh}^t \nabla V_S(\bar{Z}_s)ds + \sqrt{2} B_{t-kh}}^{2p}]\\
  &\lesssim \E[(t-kh)^{2p-1} \int_{kh}^t \norm{\nabla V_S(\bar{Z}_s)}^{2p} ds ]  + \E[\norm{B_{t-kh}}^{2p}]\\
  &\lesssim (t-kh)^{2p-1} \int_{kh}^t (A_{\text{grad},1}^{2p}  \norm{\bar{Z}_s - u_S}^{2p} + A_{\text{grad},0}^{2p})ds +  (d(t-kh))^p 
\end{align*}
The desired result follows from $p=4.$
\end{proof}
\begin{proposition} \label{prop:drift bound}
Suppose $S\subseteq I$ satisfies item 2 of Assumption~\ref{assumption:cluster}.
Let $(\bar{Z}_t)_{t\geq 0}$ be the continuous Langevin diffusion wrt $\mu_S$ initialized at $\nu_0.$ 
Suppose $h \leq \frac{1}{2 A_{\text{grad},1}}$ and $\sup_{k\in [0,N-1]\cap \N}\norm{\bar{Z}_{kh} -u_S}\leq D$ 
then 
\[\sup_{k\in [0,N-1]\cap \N, t\in [0,h]} \norm{\bar{Z}_{kh+t} - \bar{Z}_{kh}} \leq  2 h( A_{\text{grad}, 0}   + A_{\text{grad}, 1}  \norm{\bar{Z}_{kh} - u_S}) + \sqrt{48 dh \ln \frac{6N}{\eta}}   \]
thus with probability $\geq 1-\eta$
\[\sup_{k\in [0,N-1]\cap \N, t\in [0,h]} \norm{\bar{Z}_{kh+t} - u_S} \leq  2 hA_{\text{grad}, 0}   +  2D  +  \sqrt{48 dh \ln \frac{6 N}{\eta}} \]
\end{proposition}
\begin{proof}
The proof is identical to \cite[Lemma 24]{chewi2021analysis}. By triangle inequality,
\begin{align*}
   &\norm{ \bar{Z}_{kh+t}-\bar{Z}_{kh}} \\
   &\leq\int_0^t \norm{\nabla V_S (\bar{Z}_{kh+r})} dr + \sqrt{2} \norm{B_{kh+t}-B_{kh}}\\
   &\leq h A_{\text{grad}, 0}  + A_{\text{grad}, 1}\int_0^t  \norm{\bar{Z}_{kh+r} - u_S} dr +  \sqrt{2} \norm{B_{kh+t}-B_{kh}}\\
   &\leq h A_{\text{grad}, 0}   + A_{\text{grad}, 1}\left( h \norm{\bar{Z}_{kh}-u_S}+\int_0^t  \norm{\bar{Z}_{kh+r} - \bar{Z}_{kh}} dr\right) +  \sqrt{2} \norm{B_{kh+t}-B_{kh}}
\end{align*}
where we use item 2 of Assumption~\ref{assumption:cluster} in the second inequality. Gronwall's inequality then implies
\begin{align*}
  &\norm{ \bar{Z}_{kh+t}-\bar{Z}_{kh}} \\
  &\leq\left( h( A_{\text{grad}, 0}   + A_{\text{grad}, 1} \norm{ \bar{Z}_{kh}-u_S}) +  \sqrt{2} \sup_{t\in [0,h]}\norm{B_{kh+t}-B_{kh}} \right) \exp (h A_{\text{grad}, 1})  \\
  &\leq 2 h( A_{\text{grad}, 0}   + A_{\text{grad}, 1}  \norm{\bar{Z}_{kh} - u_S}) + \sqrt{8}\sup_{t\in [0,h]}\norm{B_{kh+t}-B_{kh}}  
\end{align*}
as long as $h \leq \frac{1}{2  A_{\text{grad}, 1}  }.$

Thus by triangle inequality,
\begin{align*}
  \norm{\bar{Z}_{kh+t}-u_S} &\leq \norm{\bar{Z}_{kh}- u_S} + \norm{ \bar{Z}_{kh+t}-\bar{Z}_{kh}} \\
  &\leq 2 h A_{\text{grad}, 0}   + \norm{\bar{Z}_{kh}- u_S} (2h A_{\text{grad}, 1}  + 1) +  \sqrt{8}\sup_{t\in [0,h]}\norm{B_{kh+t}-B_{kh}}    
\end{align*}
By union bounds and concentration for Brownian motion (see \cite[Lemma 32]{chewi2021analysis}),
with probability $ 1-\eta,$
\[\sup_{k\in [0,N-1]\cap \N, t\in [0,h]}\norm{B_{kh+t}-B_{kh}} \leq \sqrt{6 dh \ln \frac{6N}{\eta}}\] thus
\[\sup_{k\in [0,N-1]\cap \N, t\in [0,h]}\norm{\bar{Z}_{kh+t}-u_S} \leq 2 h A_{\text{grad}, 0}   + 2D +   \sqrt{48 dh \ln \frac{6N}{\eta}}  \]
\end{proof}
 
\section{Analysis of Continuous-time Diffusion}\label{sec:continuous}
In this section, we analyze an idealized version of the final LMC chain: we assume knowledge of the exact score function and run the continuous time Langevin diffusion. First in Lemma~\ref{lem:well separated cluster} below, we prove that when the diffusion is initialized from a point, it converges in a certain amount of time to a sample from a mixture distribution corresponding to the clusters near the initialization. Then in Theorem~\ref{thm:continuous mixing} we deduce the analogue of our main result for the idealized process: the diffusion started from samples converges to the true distribution. 

\begin{definition} \label{def:max index}
    For $S\subseteq I$ and $ x\in \R^d,$ let $i_{\max, S}(x) = \arg\max_{i\in S} \mu_i(x).$ We break ties in lexicographic order of $i $ i.e. we let $i_{\max, S}(x)$ be the maximum index among all indices $i$ s.t. $ \mu_i(x) = \max_{j\in S} \mu_j(x).$ 
\end{definition}
\begin{lemma}\label{lem:well separated cluster}
Fix $ \epsilon_{TV} , \tau \in(0,1/2), \delta \in (0,1].$ Fix $S\subseteq I.$ 
Let $\bar{p}_i=p_i p_S^{-1}$ and recall that $\mu_S = \sum_{i\in S} \bar{p}_i\mu_i.$ Let $p_*=\min_{i\in S} \bar{p}_i.$ Note that $ p_*\geq \min_{i\in I} p_i.$ Recall that $\abs{I} = K.$

Suppose for $i\in S,$ $\mu_i $ are $\alpha$-strongly log-concave and $ \beta$-smooth with $\beta \geq 1.$ 
Let $u_i =\arg\min_x V_i(x)$ and $D\geq 5 \sqrt{\frac{d}{\alpha}}$ be as defined in \cref{lem:single distribution}. Suppose there exists $L\geq 10 D$ such that for any $i,j\in S,$
$ \norm{u_i-u_j}\leq L.$  

Let $\mathbb{G}^{\delta} :=\mathbb{G}^{\delta}(S, E )$ be the graph on $S$ with an edge between $i,j$ iff $ \delta_{ij}\leq \delta.$ 
Let  
\[ T= \frac{2 C_{p_*,K} }{\delta \alpha }\left( \ln \left(\frac{\beta^2 L}{\alpha}\right) + \ln \ln \tau^{-1}   + 2\ln \tilde{\epsilon}_{TV}^{-1}\right). \]
Suppose for all $i, j\in S$ which are not in the same connected component of $\mathbb{G}^{\delta}$, $\delta_{ij} \leq\delta'$ with
\[\delta' =   \frac{\delta^{3/2} \alpha ^{3/2}  p_*^{5/2} \epsilon_{TV}^2 \tau  }{10^5  K^5 d (\beta L)^3 \ln^{3/2} (p_*^{-1})\ln^{3/2} \frac{\beta^2 L \epsilon_{TV}^{-1} \ln \tau^{-1}}{\alpha}  \ln^{2.51} \frac{16 d (\beta L)^2 }{\epsilon_{TV}\tau  \delta \alpha }  }      \]
For $x\in\R^d$, let $ (\bar{X}_{t}^{\delta_x})_{t\geq 0}$ denote the continuous Langevin diffusion with score $\nabla V_S$ initialized at $\delta_x,$ and let $C_{\max} (x)$ be the unique connected component of $\mathbb{G}^{\delta}$ containing $i_{\max, S}(x) =\arg\max_{i\in S} \mu_i(x)$ as defined in \cref{def:max index}. Then
\[\mathbb{P}_{x\sim\mu_S}[d_{TV} (\mathcal{L}(\bar{X}_{t}^{\delta_x} | x) , \mu_{C_{\max}(x)} ) \leq \epsilon_{TV}] \geq 1-\tau\]

\end{lemma}

From the above lemma, we can deduce the following theorem. (The proof of the lemma is deferred until after the proof of the theorem.)
In this result, the reader can consider simply the case $S = I$; the flexibility to pick a subset of indices is allowed for convenience later. 
\begin{theorem}\label{thm:continuous mixing}
Fix $ \epsilon_{TV} , \tau \in(0,1/2).$ Fix $ S \subseteq I.$ Let $\bar{p}_i=p_i p_S^{-1}$ and recall that $\mu_S = \sum_{i\in S} \bar{p}_i\mu_i.$ Let $p_*=\min_{i\in S} \bar{p}_i.$ Note that $ p_*\geq \min_{i\in I} p_i.$ Recall that $\abs{I} = K.$

Suppose for $i\in S,$ $\mu_i$ are $\alpha$-strongly log-concave and $ \beta$-smooth with $\beta \geq 1.$ Let $u_i =\arg\min_x V_i(x)$ and $D\geq 5 \sqrt{\frac{d}{\alpha}}$ be as defined in \cref{lem:single distribution}. Suppose there exists $L\geq 10 D$ such that for any $i,j\in S,$
$ \norm{u_i-u_j}\leq L.$   
 Let $U_{\text{sample}}$ be a set of $M$ i.i.d. samples from $\mu_S$ and $ \nu_{\text{sample}}$ be the uniform distribution over  $U_{\text{sample}}.$
Let $(\bar{X}_t^{\nu_{\text{sample}}})_{t\geq 0 }$ be the continuous Langevin diffusion with score $\nabla V_S$ initialized at $\nu_{\text{sample}}.$ 
 Let \[\tilde{\Gamma} = \frac{p_*^{7/2} \epsilon_{TV}^3  \alpha^{3/2} }{10^8 d (\beta L)^3 {\exp(K) \ln^{3/2} (p_*^{-1}) \ln^{5}  \frac{16 d (\beta L)^2 }{\epsilon_{TV}\tau  \alpha } } },\]
 If $M \geq 600(\epsilon_{TV}^2 p_*)^{-1} K^2 \log (K\tau^{-1}) $ and
 
 \[ T\geq  \Theta \left(\alpha^{-1} K^2 p_*^{-1}\ln(10 p_*^{-1})\tilde{\Gamma}^{-2((3/2)^{K-1} -1)} \right ) \]
 
 then 
 \[\mathbb{P}_{U_{\text{sample}}} [d_{TV} (\mathcal{L}(\bar{X}_{T}^{\nu_{\text{sample}}}| U_{\text{sample}}),\mu_S) \leq  \epsilon_{TV} ] \geq 1-\tau \]
\end{theorem}
\begin{remark}
    Note that after fixing $U_{\text{sample}}$, $\hat{\mu}_S^{U_{\text{sample}}}:=\mathcal{L}(\bar{X}_{T}^{\nu_{\text{sample}}} | U_{\text{sample}})$ is a function of  
$U_{\text{sample}}$ and Brownian motions $(B_t)_{t\in [0,T]}.$ Each run of the Langevin diffusion produces a sample from $\hat{\mu}_S^{U_{\text{sample}}}$ by choosing/sampling a value for the Brownian motions, thus we can produce as many samples as desired from $\hat{\mu}_S^{U_{\text{sample}}},$ while  \cref{thm:continuous mixing} guarantees that $ \hat{\mu}_S^{U_{\text{sample}}}$ is approximately close to $\mu_S$ in total variation distance for a typical set of samples $U_{\text{sample}}.$
\end{remark}
\begin{proof}[Proof of \cref{thm:continuous mixing}]
Let $\bar{p}_i=p_i p_S^{-1}$ then $\mu_S = \sum_{i\in S}\bar{p}_i \mu_i.$ For $C\subseteq S,$ let $\bar{p}_C =\sum_{i\in C} \bar{p}_i.$

Let $\tilde{\epsilon}_{TV} = \frac{\epsilon_{TV}}{9K}$ and $ \tilde{\tau} = \frac{p_*\epsilon_{TV}}{9 K }\leq \min \set{\frac{\epsilon_{TV}}{9 K^2}, p_*/3}.$ Define the sequence $ 1= \delta_0 > \delta_1 >\cdots > \delta_{K} $ inductively as follow: 
\begin{align*}
   \delta_{s+1} &=  \frac{\delta_s^{3/2} \alpha ^{3/2}  p_*^{5/2} \tilde{\epsilon}_{TV}^2 \tilde{\tau}  }{10^5  K^5 d (\beta L)^3 \ln^{3/2} (p_*^{-1})\ln^{3/2} \frac{\beta^2 L \tilde{\epsilon}_{TV}^{-1} \ln \tilde{\tau}^{-1}}{\alpha}  \ln^{2.51} \frac{16 d (\beta L)^2 }{\tilde{\epsilon}_{TV}\tilde{\tau}  \delta_s \alpha }  }   \\
   &\geq \frac{\delta_s^{3/2} \alpha ^{3/2}  p_*^{7/2} \epsilon_{TV}^3   }{ 10^8  K^8 d (\beta L)^3 \ln^{3/2} (p_*^{-1})\ln^{3/2} \frac{\beta^2 L \epsilon_{TV}^{-1} K }{\alpha}  \ln^{2.51} \frac{16 d (\beta L)^2K  }{\epsilon_{TV}  \delta_s \alpha }  } 
\end{align*}
 Let $\mathbb{G}^{s} :=\mathbb{G}^{\delta}(S, E )$ be the graph on $S$ with an edge between $i,j$ iff $ \delta_{ij}\leq \delta_s.$ Fix one such $s$ s.t. $s\leq K-2.$ Suppose $\delta_{ij} \leq\delta_{s+1}$ for all $i, j$ not in the same connected component of $\mathbb{G}^{s}$,  y  then \cref{lem:well separated cluster} applies.
 Let the connected components of $\mathbb{G}^{s}$ be $ C^s_1, \dots, C^s_m.$
 For $x\in\R^d$, let ${C^s_{\max}(x)}$ be the unique connected component of $\mathbb{G}^{s}$ containing $i_{\max, S}(x)$ and let $ (\bar{X}_{t}^{\delta_x})_{t\geq 0}$ denote the continuous Langevin diffusion with score $\nabla V$ initialized at $\delta_x,$ then for $ T_s =\frac{2 C_{p_*,K} }{\delta_s \alpha }(\ln \frac{\beta^2 L}{\alpha }  + \ln \ln \tilde{\tau}^{-1}  + 2\ln \tilde{\epsilon}_{TV}^{-1})$,
\[\mathbb{P}_{x\sim\mu_S}[d_{TV} (\bar{X}_{T_s}^{\delta_x} , \mu_{C^s_{\max}(x)} ) \leq \tilde{\epsilon}_{TV}] \geq 1-\tilde{\tau}\]
and by \cref{prop:init partition bound},
\[\mathbb{P}_{x\sim\mu_S}[i_{\max,S}(x) \in C^s_r ] \geq (1-\delta_{s+1}) \bar{p}_{C^s_r}.\]
It is easy to see that $ \delta_{s+1}\leq \tilde{\epsilon}_{TV}/3.$ By \cref{prop:process init from sample disjoint case}, as long as $M \geq 600 (\epsilon_{TV}^2 p_*)^{-1}K^2 (\log K + \log \tau^{-1})$
\[\mathbb{P}_{U_{\text{sample}}} [d_{TV} (\mathcal{L}(\bar{X}_{T_s}^{\nu_{\text{sample}}}| U_{\text{sample}}), \mu_S) \leq \epsilon_{TV} ] \geq 1- \tau \]
Since $\mu_S$ is the stationary distribution of the continuous Langevin with score function $\nabla V_S,$ for any $ T \geq T_{K-1} \geq T_s,$ $d_{TV} (\mathcal{L}(\bar{X}_{T}^{\nu_{\text{sample}}}| U_{\text{sample}}), \mu_S) \leq d_{TV} (\mathcal{L}(\bar{X}_{T_s}^{\nu_{\text{sample}}}| U_{\text{sample}}), \mu_S) $ thus 
\[ \mathbb{P}_{U_{\text{sample}}} [d_{TV} (\mathcal{L}(\bar{X}_{T_s}^{\nu_{\text{sample}}}| U_{\text{sample}}), \mu_S) \leq \epsilon_{TV} ] \geq 1- \tau . \]
On the other hand, suppose for all $s \in [0, K-2]\cap \N ,$ there exists $i, j$ not in the same connected component of $\mathbb{G}^{s}$ s.t. $\delta_{ij} > \delta_{s+1},$ then $\mathbb{G}^{s+1}$ has one fewer connected components than $ \mathbb{G}^s.$ Thus $\mathbb{G}^{K-1} $ is connected then  $\mu$ has LSI constant $\propto \delta_{K-1}^{-1},$ 
thus \cref{lem:well separated cluster} apply with $\delta = \delta_{K-1}$ and \cref{prop:init partition bound} apply with $ \delta'=0.$ 
For $T  \geq T_{K-1} $,
\[ \mathbb{P}_{U_{\text{sample}}} [d_{TV} (\mathcal{L}(\bar{X}_{T}^{\nu_{\text{sample}}}| U_{\text{sample}}), \mu_S) \leq \epsilon_{TV} ] \geq 1- \tau . \]
Let $\Gamma = \frac{p_*^{7/2} \epsilon_{TV}^3  \alpha^{3/2} }{10^8  K^8 d (\beta L)^3  }.$ If we ignore log terms, then $\delta_{s+1} = \delta_{s}^{3/2} \Gamma $  thus $\delta_s \approx \Gamma^{1+ 3/2 + \cdots + (3/2)^{s-1}} = \Gamma ^{2((3/2)^s -1)}.$
To get the correct bound for $\delta_s$ and $ T_s,$ we can let 
\[ \Gamma_1 =  \frac{p_*^{7/2} \epsilon_{TV}^3  \alpha^{3/2} }{8000 d (\beta L)^3 {\exp(K) \ln^{3/2} (p_*^{-1}) \ln^{4.5}  \frac{16 d (\beta L)^2 }{\epsilon_{TV}\tau  \alpha } } }\leq \Gamma \]
then we can inductively prove $\delta_s\geq \Gamma_1^{2((3/2)^s -1)} $ and thus get the bound on $T_s$ i.e.
\begin{align*}
  T_s &\leq \Theta (\alpha^{-1} K^2 p_*^{-1}  \ln (10 p_*^{-1}) \ln(\frac{\beta^2 L\epsilon_{TV}  }{\alpha p_* K } ) \Gamma_1^{-2((3/2)^s -1)}) \\
  &=\Theta(\alpha^{-1} K^2 p_*^{-1}\ln(10 p_*^{-1})\tilde{\Gamma}^{-2((3/2)^s -1)})  
\end{align*}
with $\tilde{\Gamma}= \frac{p_*^{7/2} \epsilon_{TV}^3  \alpha^{3/2} }{10^8 d (\beta L)^3 {\exp(K) \ln^{3/2} (p_*^{-1}) \ln^{5}  \frac{16 d (\beta L)^2 }{\epsilon_{TV}\tau  \alpha } } }$.
\end{proof}
\begin{proof}[Proof of \cref{lem:well separated cluster}]
Let $(\bar{X}_t)$ denote the continuous Langevin with score $\nabla V_S$ initialized at $\mu_S.$ Since $\mu_S$ is the stationary distribution of continuous the Langevin with score $\nabla V_S,$ the law $\mathcal{L}(\bar{X}_t)$ of $\bar{X}_t$ is $\mu_S$ at all time $t.$
 Let $\eta = \tau \epsilon_{TV}/2.$
 Let $h > 0, \gamma \in (0,1)$ to be chosen later. Let $\mathcal{C}$ be the partition of $S$ consisting of connected components of the graph $\mathbb{G}^{\delta},$ and $B_{S, \mathcal{C} , \gamma}$ be defined as in \cref{def:bad set for partition}.
 Suppose $\delta',\gamma$ satisfies $ K^2 \gamma^{-1} \delta'\times T/h \leq\eta/2, $ then by \cref{lem:bad set bound}, $\mu_S(B_{S, \mathcal{C} , \gamma}) N \leq K^2 \gamma^{-1} \delta' \times T/h \leq \eta/2$

 Since the law of $\bar{X}_{kh}$ is $\mu_S,$ we can bound $\norm{\bar{X}_{kh}-u_S}$ using sub-Gaussian concentration of $\mu_S$ (due to \cref{prop:cluster bound with small distance betwen centers}).
 By the union bound, with probability $ 1- 
 \eta,$ the event $\mathcal{E}_{\text{discrete}}$  happens where $\mathcal{E}_{\text{discrete} }$ is defined by: $\forall k\in [0,N-1]\cap \N:\norm{\bar{X}_{kh} -u_S}\leq 2L + \sqrt{\frac{64}{\alpha} \ln \frac{16N}{\eta} } $ and $ \bar{X}_{kh} \not \in B_{S,\mathcal{C},\gamma}.$
Since $\mu_S = \E_{x\sim \mu_S}[\delta_x]$ and $\mu_S$ and $\delta_x$ are the initial distribution of $\bar{X}_t$ and $\bar{X}_{t}^{\delta_x}$ respectively, so 
$\mathcal{L}(\bar{X}_{kh}) = \E_{x\sim \mu_S} [\mathcal{L} (\bar{X}_{kh}^{\delta_x} | x) ]$ where $\mathcal{L}(X)$ denote the law of the random variable $X.$ 
 Thus, let $\tilde{L}:=2L + \sqrt{\frac{64}{\alpha} \ln \frac{16N}{\eta} }$ and
  $\mathcal{G}_x$ is the event \[\mathbb{P}_{\mathcal{F}_t}[\forall k \in [0,N-1]\cap \N : \norm{\bar{X}_{kh}^{\delta_x} -u_S}\leq \tilde{L} \land  \bar{X}_{kh}^{\delta_x} \not \in B_{S, \mathcal{C}, \gamma} ] \geq 1-\epsilon_{\text{TV}}/10 \]
  where the probability is taken over the randomness of the Brownian motions, then  $\mathbb{P}_{x\sim \mu} [\mathcal{G}_x]\geq 1-\tau/2 $
  
  Fix $x$, let $C=C_{\max}(x)$ and suppose $\mathcal{G}_x$ holds.
Suppose $h$ satisfies the precondition of \cref{prop:error bound for continuous process},  then with probability $\geq 1-\epsilon_{\text{TV}}/5,$ 
\[ \sup_{t\in [0,T]} \norm{\nabla V_S(\bar{X}_{t}^{\delta_x}) -\nabla V_{C_{\max}(x)}(\bar{X}_{t}^{\delta_x}) } \leq \epsilon_{\text{score},1}:= 36 p_*^{-1} \gamma\beta \tilde{L} \] thus $ \bar{X}_{t}^{\delta_x} \not \in B $ for all $t\in [0,T],$ where $B$ is the "bad" set defined by $B= \set{z\in \R^d: \norm{\nabla V_S(z) -\nabla V_C(z) }>\epsilon_{\text{score},1} }.$  
Let $\nu_0$ be the distribution of $ \bar{X}_{h'}^{\delta_x}$ for some $h'\leq 1/(2\beta).$ Let $ \mathcal{G}_{\text{init}, x}$ be the event that $ \norm{x-u_S} \leq L_1:= 2L + \log (10/\tau)$ then $\mathbb{P}_{x\sim \mu}[ \mathcal{G}_{\text{init}, x}]\geq 1-\tau/10.  $ Suppose  $ \mathcal{G}_{\text{init},x}$ happens. Then $ G_S(x)= \max_{i\in S} \norm{\nabla V_i(x)} \lesssim \beta L_1 .$
Set $ h' =\min \set{ \frac{1}{\beta \tilde{L}},\frac{1}{\beta d} }$ then by \cref{lem:continuous initialization},
\[\D_{\KL} (\nu_0 ||\mu_C) \lesssim  d \ln L_1  +\alpha^{-1}\beta^2 L_1^2  \]
Pick $T=  \frac{2 C_{p_*,K} }{\delta \alpha }( \ln \frac{\beta^2 L}{\alpha}  + \ln \ln \tau^{-1})   $  then  $T-h'\geq T_{\text{process}} := \frac{C_{p_*,K} }{\delta \alpha } (\ln \D_{\KL} (\nu_0 ||\mu_S) + 2 \ln \epsilon_{TV}^{-1}) $ 

Let $(\bar{Z}_t^{\nu_0})_{t\geq 0}$ be the continuous Langevin initialized at $ \nu_0$ with score $s_{\infty}$ defined by
\[ s_{\infty} (z) = \begin{cases} \nabla V_S(z) \text{ if } x\not \in B   \\ \nabla V_C (z)  \text{ if } x \in B  \end{cases} \]
then $ \sup_{z\in \R^d}\norm{s_{\infty}(z) -\nabla V_C(z) }^2\leq \epsilon_{\text{score},1}^2. $ 
Note that if $\mathcal{G}_x$ holds then $\bar{X}_{t+h'}^{\delta_x}\not\in B\forall t\in [0, T-h']  $ and $ \bar{Z}_t^{\nu_0} =\bar{X}_{t+h'}^{\delta_x}\forall t\in [0, T-h'] $ thus 
\[d_{TV}(\bar{X}_{t}^{\delta_x}, \bar{Z}_{T-h'}^{\nu_0} )\leq \epsilon_{TV}/5\]
\cref{prop:continuous chain with linfty score} gives 
\[d_{TV} ( \mathcal{L}(\bar{Z}_{T-h'}^{\nu_0} | x) , \mu_C) \leq \epsilon_{\text{score},1} \sqrt{T/2} +\epsilon_{TV}/5 \]
Set $\gamma = \frac{ p_* \epsilon_{TV}}{18 \beta \tilde{L}\sqrt{T}}$ then $\epsilon_{\text{score},1} =  18 p_*^{-1} \gamma \beta \tilde{L} \leq \frac{\epsilon_{TV}}{  \sqrt{T}} $  then by triangle inequality
\[d_{TV}(\mathcal{L}(\bar{X}_{t}^{\delta_x}|x), \mu_C)\leq \epsilon_{TV} .\] 
This holds conditioned on $ \mathcal{G}_x$ and $\mathcal{G}_{\text{init}, x}$ both happen, thus by union bound
\[\mathbb{P}_{x\sim \mu} [d_{TV}(\mathcal{L}(\bar{X}_{t}^{\delta_x}|x), \mu_{C_{\max} (x) })\leq \epsilon_{TV} ] \geq 1 - \tau\]
Plug in $T ,\gamma$
and set \[h = \frac{1}{2000 d (\beta L)^2 \ln^2 \frac{16 d (\beta L)^2 T}{\epsilon_{TV}\tau}  }\]
then $h \ln (1/h) \leq  \frac{1}{2000 d (\beta L)^2}$ and $h\ln^2(1/h) = \frac{1}{1000(\beta^2/\alpha) }$ and $ h\leq \frac{1}{100(\beta^2/\alpha)\ln^2(16 T/\eta) }.$ Hence $h$ satisfies the precondition of \cref{prop:error bound for continuous process}.

Finally, since $\tilde{L} \leq  L \sqrt{\ln \frac{16T}{h\eta} }\leq 2 L \sqrt{\ln (\beta L \epsilon_{TV}^{-1} \tau^{-1} T)},$ thus with \[\delta' \leq \frac{\delta^{3/2} \alpha ^{3/2}  p_*^{5/2} \epsilon_{TV}^2 \tau \ln^{3/2} \frac{\beta^2 L \epsilon_{TV}^{-1} \ln \tau^{-1}}{\alpha}  }{10^5  K^5 d (\beta L)^3 \ln (p_*^{-1}) \ln^{2.51} \frac{16 d (\beta L)^2 }{\epsilon_{TV}\tau  \delta \alpha }  }   \leq  \frac{ p_* \epsilon_{TV}^2 \tau}{10^5  K^2  T^{3/2}  d (\beta L)^3 \ln^{2.51} \frac{16 d (\beta L)^2 T}{\epsilon_{TV}\tau}  }       \]
the precondition
\begin{align*}
   K^2 \delta' \gamma^{-1} \times T/h &= K^2 \delta'\times \frac{18 \beta \tilde{L} \sqrt{T}}{p_* \epsilon_{TV} } \times T/h \\
   &\leq \delta'\times \frac{ 36 K^2\beta L  T^{3/2}  \sqrt{\ln (\beta L \epsilon_{TV}^{-1} \tau^{-1} T)}  }{p_* \epsilon_{TV} h }\\
   &\leq \eta/2 \\
   &= \epsilon_{TV} \tau/4 
\end{align*}
holds, so we are done.

\end{proof}

\begin{proposition}[Continuous chain with score estimation with $L_{\infty}$ error bound] \label{prop:continuous chain with linfty score}
Fix $C\subseteq I.$
Let $(\bar{Z}_t)_{t\geq 0}$ and $\bar{X}_t$ be the continuous Langevin diffusion with score functions $\nabla V_C$ and $ s$ respectively and both $(\bar{Z}_t)$ and $(\bar{X}_t)$ are initialized at $\nu_0.$ Suppose $ \sup_{x\in \R^d} \norm{s(x) -\nabla V_C(x)}^2\leq \epsilon_{\text{score}, 1}^2$ then
\[2 d_{TV} (\bar{X}_T ,\bar{Z}_T )^2\leq \D_{\KL} (\bar{X}_T|| \bar{Z}_T) \leq \E \left[\int_{0}^T \norm{s(\bar{Z}_t) -\nabla V_C(\bar{Z}_t) }^2  dt \right]  \leq \epsilon_{\text{score}, 1}^2 T\]
 Suppose $\mu_S$ has log Sobolev constant $C_{LS}$ and $T\geq C_{LS} (\log (2\D_{KL} (\nu_0||\mu_S)) + 2 \log \epsilon_{TV}^{-1}) $
\[d_{TV} (\mathcal{L}(\bar{X}_T), \mu_C) \leq d_{TV} (\bar{X}_T ,\bar{Z}_T ) + d_{TV} (\mathcal{L}(\bar{Z}_T),\mu_C )\leq \epsilon_{\text{score}, 1} \sqrt{ T/2} + \epsilon_{TV}/2 \]
\end{proposition}
\begin{proof}
Clearly, by the assumption on $s,$
$\E [\int_{0}^T \norm{s(\bar{Z}_t) -\nabla V_C(\bar{Z}_t) }^2  dt ]\leq \int_{0}^T \epsilon_{\text{score}, 1}^2 dt =  \epsilon_{\text{score}, 1}^2 T.$ The first statement thus follows from Girsanov and the approximation argument in \cite[Lemma 9]{chen2023sampling} and Pinsker's inequality. Next, since $\mu_C$ has LSI constant $C_{LS},$ with this choice of $T,$
\[\D_{KL} (\mathcal{L}(\bar{Z}_T)|| \mu_C) \leq \D_{KL}(\nu_0 ||\mu_S) \exp(-\frac{T}{C_{LS}}) \leq \epsilon_{TV}^2/2\]
and the second statement follows from Pinsker's inequality and triangle inequality for TV distance.
\end{proof}
We need these propositions to go from \cref{lem:well separated cluster} to \cref{thm:continuous mixing}

\begin{proposition} \label{prop:init partition bound} Suppose $\mu = \sum_{i\in I} p_i \mu_i.$
Fix a set $C\subseteq I.$
If the overlap between $\mu_i,\mu_j$ for $i\in C$ and $j\not \in C$ is $\leq \delta'$ for all such $i,j$ then
\[\mu(\set{x: i_{\max}(x) \in C})\geq p_C (1-\delta'\abs{I})  
\] 
\end{proposition}
To remove dependency on $p_*,$ we will use the following modified version of \cref{prop:init partition bound}
\begin{proposition}\label{prop:init partition bound modified}
   Fix $C,C_*\subseteq I$ s.t. $ C\cap C_*=\emptyset.$ Let $I' = I\setminus C_*.$ If for $i\in C, j\in I'\setminus C,$ the overlap between $\mu_i $ and $\mu_j$ is $\leq \delta'$ then for $i_{\max, I'} (x) =\arg\max_{i\in I'} \mu_i(x)$
   \[\mu_I (\set{x: i_{\max, I'} (x) \in C })\geq p_C (1-\delta' \abs{I}) \] 
\end{proposition}
\begin{proof}[Proof of \cref{prop:init partition bound,prop:init partition bound modified}]
We first prove \cref{prop:init partition bound}.
For $i\in C,j\not \in C$
\begin{align*}
  \mu_i(\set{x: \mu_i(x) \leq \mu_j(x) }) &= \int_{x: \mu_i(x) \leq \mu_j(x) }\mu_i(x) dx \\
  &= \int_{x: \mu_i(x) \leq \mu_j(x) } \min \set{\mu_i(x), \mu_j(x)} dx  \\
  &\leq \int \min \set{\mu_i(x), \mu_j(x)}  dx  \leq \delta'   
\end{align*}
By union bound, for $i\in C$
\[\mu_i( \set{x \mid \exists j\not \in C: \mu_i(x) \leq \mu_j(x) }) \leq \delta' \abs{I}\]
Let $\Lambda =\set{x: i_{\max}(x) \in C}  . $
If $ \forall j \not\in C: \mu_i(x) > \mu_j(x)$ then $ i_{\max}(x) \in C.$ Thus $\Lambda_i :=\set{x: \mu_i(x) > \mu_j(x) \forall j\not\in C }\subseteq\Lambda $
and $\mu_i(\Lambda_i) = 1- \mu_i(\set{x \lvert \exists j\not\in C : \mu_i(x) \leq \mu_j(x)} \geq 1-\delta'\abs{I}.$ Since $\mu(x) \geq \sum_{i\in C} p_i \mu_i(x)$
\begin{align*}
    \mu(\set{x:i_{\max}(x)\in C}) &= \int_{x\in \Lambda} \mu(x) dx \geq \int_{x\in \Lambda} \sum_{i\in C} p_i\mu_i(x) dx = \sum_{i\in C}p_i \mu_i(\Lambda) \\
    &\geq \sum_{i\in C}p_i\mu_i(\Lambda_i) \geq\sum_{i\in C} p_i (1-\delta'\abs{I})= p_C (1-\delta'\abs{I}) 
\end{align*}

The proof of \cref{prop:init partition bound modified} is identical, except we will consider $i\in C,j\in I'\setminus C$ and argue that $ \mu_i({x: \mu_i(x) \leq\mu_j(x)}) \leq \delta'.$ Then $ \mu_i({x\lvert \exists j\in I'\setminus C: \mu_i(x)\leq \mu_j(x) })\leq \delta'\abs{I}.$ For $i\in C,$ $\Lambda_i = \set{x\lvert \mu_i(x) >\mu_j(x) \forall j\in I'\setminus C}$ then $ \mu_i(\Lambda_i) \geq 1 - \delta'\abs{I}$ and $\Lambda_i \subseteq \set{x:i_{\max, I'}(x)\in C}.$ Finally,
\[\mu (\set{x: i_{\max, I'}(x) \in  C})\geq \sum_{i\in C} p_i \mu_i(\Lambda_i) \geq p_C (1 - \delta'\abs{I}). \]

\end{proof}

\begin{proposition}\label{prop:process init from sample disjoint case}
Consider distributions $\mu_i$ for $i\in I.$ Suppose $\mu = \sum_{i\in I} p_i \mu_i$ for $p_i> 0$ and $\sum_{i\in I }p_i =1.$ Suppose we have a partition $\mathcal{C}$ of $I$ into $C_1, \dots, C_m$. For $x\in \R^d,$ let $C= C_{\max}(x)$ be the unique part of the partition $\mathcal{C}$ containing $i_{\max}(x)=\arg\max_{i\in I}\mu_i(x). $
Let $p_* =\min_{i\in I} p_i.$
For $x\in \R^d,$ let $ (X_t^{\delta_x})_t$ be a process initialized at $\delta_x.$ 
Suppose for any $\tilde{\epsilon}_{TV}\in (0,1/10), \tilde{\tau}\in (0,p_*/3) ,$ there exists $T_{\tilde{\epsilon}_{TV}, \tilde{\tau}}$ such that the following holds: 
\[\P_{x\sim \mu} [d_{TV}(\mathcal{L}(X_{T_{\tilde{\epsilon}_{TV}, \tilde{\tau}}}^x |x)  ,\mu_{C_{\max}(x)}) \leq \tilde{\epsilon}_{TV}] \geq 1-\tilde{\tau} .\]

In addition, there exists $\delta'\in (0, \tilde{\epsilon}_{TV})$ s.t. for $C \in \set{C_1, \dots, C_m}$
\[\P_{x\sim\mu}[ C_{\max}(x) = C]\geq p_C (1-\delta').\]
Let $U_{\text{sample}}$ be a set of $M$ i.i.d. samples from $\mu$ and $ \nu_{\text{sample}}$ be the uniform distribution over  $U_{\text{sample}}.$
Let $(X_{t}^{\nu_{\text{sample}}})_{t\geq 0 }$ be the process with score estimate $s$ initialized at $\nu_{\text{sample}}.$
If $M \geq  6\times 10^2 \abs{I}^2 \epsilon_{TV}^{-2} p_*^{-1} \log (K \tau^{-1}),$
with probability $ \geq  1- \tau$ over $U_{\text{sample}},$ let $T = T_{\frac{\epsilon_{TV}}{9\abs{I}},  \min \set{\frac{\epsilon_{TV}}{9\abs{I}^2 } , p_*/3} } $ and $\hat{\mu} = \mathcal{L}(X_{T}^{\nu_{\text{sample}}} | U_{\text{sample}})$, then
\[\P_{U_{\text{sample}} } [d_{TV}( \mathcal{L}(X_{T}^{\nu_{\text{sample}}} | U_{\text{sample}}), \mu) \leq \epsilon_{TV} ] \geq 1- \tau \]
\end{proposition}
To remove the dependency on $p_*=\min_{i\in I} p_i$, we will use this modified version of \cref{prop:process init from sample disjoint case}.
\begin{proposition}\label{prop:process init from sample disjoint case modified}
Consider distributions $\mu_i$ for $i\in I.$ Suppose $\mu = \sum_{i\in I} p_i \mu_i$ for $p_i> 0$ and $\sum_{i\in I }p_i =1.$ 
For $x\in \R^d,$ let $ (X_t^{\delta_x})_t$ be a process initialized at $\delta_x.$ 
Suppose for any $\tilde{\epsilon}_{TV}\in (0,1/10), \tilde{\tau}\in (0,1) $, there exists $T_{\tilde{\epsilon}_{TV}, \tilde{\tau}}$ such that the following holds. Let $I'=\set{i\in I: p_i \geq  \frac{\tilde{\epsilon}_{TV}}{\abs{I}}}$ and $C_* =C\setminus I'.$ Suppose we have a partition $\mathcal{C}$ of $I'$ into $C_1, \cdots, C_r .$ For $x\in \R^d,$ let $C_{\max}(x)$ be the unique part of the partition $\mathcal{C}$ containing $i_{\max,I'}(x)=\arg\max_{i\in I' }\mu_i(x). $

\[\P_{x\sim \mu} [d_{TV}(\mathcal{L}(X_{T_{\tilde{\epsilon}_{TV}, \tilde{\tau}}}^{\delta_x} |x)  ,\mu_{C_{\max}(x) }) \leq \tilde{\epsilon}_{TV}] \geq 1-\tilde{\tau} .\]

In addition, there exists $\delta'\in (0, \tilde{\epsilon}_{TV})$ s.t. for $C \in \set{C_1, \dots, C_m}$
\[\P_{x\sim\mu}[ C_{\max, I'}(x) = C]\geq p_C (1-\delta').\]
Let $U_{\text{sample}}$ be a set of $M$ i.i.d. samples from $\mu$ and $ \nu_{\text{sample}}$ be the uniform distribution over  $U_{\text{sample}}.$
Let $(X_{t}^{\nu_{\text{sample}}})_{t\geq 0 }$ be the process with score estimate $s$ initialized at $\nu_{\text{sample}}.$
If $M \geq 2\times 10^4 \abs{I}^3 \epsilon_{TV}^{-3} \log (\abs{I} \tau^{-1}),$ then
\[\P_{U_{\text{sample}} } [d_{TV}( \mathcal{L}(X_{T}^{\nu_{\text{sample}}} | U_{\text{sample}}), \mu) \leq \epsilon_{TV} ] \geq 1- \tau \]
\end{proposition}
\begin{proof}[Proof of \cref{prop:process init from sample disjoint case} and \cref{prop:process init from sample disjoint case modified}]

We will prove \cref{prop:process init from sample disjoint case modified}. The proof of \cref{prop:process init from sample disjoint case} is similar. Set
 $\tilde{\tau} = \frac{\tilde{\epsilon}_{TV}}{\abs{I}}. $
Let $ \Omega_r=\set{x: C_{\max}(x)=C_r \land  d_{TV} (X_{T}^{\delta_x} , \mu_{C_{\max}(x)} ) \leq \epsilon_{TV}}.$ Clearly, $\Omega_r$ are disjoint, and by union bound $\mu(\Omega_r) \geq \tilde{p}_{C_r}:= (1-\delta') p_{C_r} - \tilde{\tau}\geq \frac{\tilde{\epsilon}_{TV} }{10 \abs{I}}.$

Let $ U_r = \Omega_r \cap U_{\text{sample}} $ then Chernoff bound gives
\[\P[\abs{U_r} \geq M \tilde{p}_{C_r} (1-\tilde{\epsilon}_{TV})   ]\geq 1- \exp(-\tilde{\epsilon}_{TV}^2 \tilde{p}_{C_r} M /2) \geq 1- \exp(-\frac{\tilde{\epsilon}_{TV}^3 M }{20 \abs{I} })  \]

Let $\mathcal{E}$ be the event $\forall r: \abs{U_r} \geq M \tilde{p}_{C_r} (1-\tilde{\epsilon}_{TV}).$ By union bound, $\P[\mathcal{E}] \geq 1- \abs{I} \exp(-\frac{\tilde{\epsilon}_{TV}^3 M }{20 \abs{I} }). $ 

Suppose $\mathcal{E}$ happens. 
Let $U_{\emptyset} =U_{\text{sample}} \setminus \bigcup_{r\in J}  U_{r} $ then 
\begin{align*}
 \abs{U_{\emptyset}} &\leq M - M (1-\tilde{\epsilon}_{TV}) \sum_{r\in J }( p_{C_r} (1- \delta') -  \frac{\tilde{\epsilon}_{TV}}{\abs{I}}) \\
 &\leq M\left[1- (1- \tilde{\epsilon}_{TV}) ( (1-\delta') (1-\tilde{\epsilon}_{TV}) -   \tilde{\epsilon}_{TV} ) \right]  \\
 &\leq M  ( 3 \tilde{\epsilon}_{TV} + \delta')    \leq 4 M \tilde{\epsilon}_{TV}
\end{align*}
where the second inequality is due to $ \sum_{r\in J} p_{C_r} \geq 1 - \sum_{r\not\in J} p_{C_r} \geq 1 - \abs{I}\times \tilde{\epsilon}_{TV}/\abs{I}.$

Note that $\mathcal{L}(X_T^{\nu_{\text{sample}}} | U_{\text{sample}}) = \frac{1}{M}  \sum_{x\in U_{\text{sample}}} \mathcal{L}(X_t^{\delta_x} |x ).$ 
Thus, let $\hat{\mu} = \sum_{r} \frac{\abs{U_r}}{\abs{U_{\text{sample}}\setminus U_\emptyset }} \mu_{C_r}$ and $ \tilde{\mu} := \sum_{r} \frac{\abs{U_{r} }}{M} \mu_{C_r} + \frac{\abs{U_{\emptyset} }}{M}\hat{\mu}$,  we can apply part 1 of \cref{prop:tv distance mixture bound}
\begin{align*}
    &d_{TV}(\mathcal{L}(X_T^{\nu_{\text{sample}}}| U_{\text{sample}}),  \tilde{\mu}) \\
    &\leq M^{-1} \left ( \sum_{r} \sum_{x\in U_{r}} d_{TV}(\mathcal{L}(X_T^{\delta_x}|x), \mu_{C_r}) + \sum_{x\in U_{\emptyset} } d_{TV} (\mathcal{L}(X_T^{\delta_x}|x), \hat{\mu})\right)\\
    &\leq M^{-1} (\tilde{\epsilon}_{TV} (M - \abs{U_{\emptyset}}) +   \abs{U_{\emptyset}}) \\
    &\leq \tilde{\epsilon}_{TV} + 4 \tilde{\epsilon}_{TV} \leq 5 \tilde{\epsilon}_{TV}
\end{align*}
Next, note that $\mu=\sum_{r} p_{C_r} \mu_{C_r} + p_{C_*}\mu_{C_*} $ and $\tilde{\mu} =\sum_{r}\bar{p}_{C_r} \mu_{C_r} $ with $\bar{p}_{C_r}  :=  \frac{\abs{U_r} }{M} (1+ \frac{\abs{U_{\emptyset} }}{ \abs{U_{\text{sample}}\setminus U_\emptyset }  } )  = \frac{\abs{U_r} }{M - \abs{U_{\emptyset}} }.$  We bound $\abs{\bar{p}_{C_r} - p_{C_r}}.$
\begin{align*}
    \frac{\abs{U_r} }{M - \abs{U_{\emptyset}} } 
    &\geq \frac{\abs{U_r}}{M}\geq \frac{M (1-\tilde{\epsilon}_{TV} )  ((1-\delta') p_{C_r} - \frac{\tilde{\epsilon}_{TV}}{\abs{I}} ) }{M  } \\
    &\geq p_{C_r} (1 - \tilde{\epsilon}_{TV} - \delta') -\frac{\tilde{\epsilon}_{TV}}{\abs{I}} \geq p_{C_r} - \tilde{\epsilon}_{TV} (2 + \frac{1}{\abs{I}})
\end{align*}
We upper bound $\abs{U_r}.$ Since $U_r $'s are disjoint,
\begin{align*}
    \abs{U_r} &\leq M - \sum_{s\in J, s\neq r}
    \abs{U_{s}} \leq M - M \sum_{s\in J, s\neq r} \left[ p_{C_s} (1 - \tilde{\epsilon}_{TV} - \delta') -\frac{\tilde{\epsilon}_{TV}}{\abs{I}} \right] 
    \\
    &\leq M (p_{C_r} + 3\tilde{\epsilon}_{TV}  +\delta' )\leq M (p_{C_r}+ 4 \tilde{\epsilon}_{TV} )
\end{align*}
where the first inequality is due to the lower bound of $\abs{U_s}$ above and the second inequality is due to $ 1-\sum_{s: s\neq r} p_{C_s} \leq p_{C_r} + \tilde{\epsilon}_{TV}$ and $ \sum_{s:s\neq r} p_{C_s}  (\tilde{\epsilon}_{TV} + \delta') \leq (\tilde{\epsilon}_{TV} + \delta') .$ Thus
\begin{align*}
    \frac{\abs{U_r} }{M - \abs{U_{\emptyset}} }  -p_{C_r}\leq \frac{M (p_{C_r} +4 \tilde{\epsilon}_{TV} ) }{M (1- 4 \tilde{\epsilon}_{TV} )} -p_{C_r} \leq \frac{4 \tilde{\epsilon}_{TV} (p_{C_r}+1) }{1-4 \tilde{\epsilon}_{TV} } \leq 16 \tilde{\epsilon}_{TV} 
\end{align*}
where in the last inequality, we use the bounds $ \tilde{\epsilon}_{TV}\leq 1/10$ and $p_{C_r}\leq 1.$
thus 
\[\abs{\bar{p}_{C_r} - p_{C_r} }\leq \max \set{  16 \tilde{\epsilon}_{TV} , \tilde{\epsilon}_{TV} (2 + \frac{1}{\abs{I}}) } \]
Part 2 of \cref{prop:tv distance mixture bound} gives
\begin{align*}
    2d_{TV}(\mu,\tilde{\mu}) \leq \sum_{r} \abs{\bar{p}_{C_r} - p_{C_r} } + p_{C_*} &\leq\abs{I}  \max \set{  16 \tilde{\epsilon}_{TV} , \tilde{\epsilon}_{TV} (2 + \frac{1}{\abs{I}}) }   + \abs{I}\times \tilde{\epsilon}_{TV}/\abs{I}\\
    &\leq (16\abs{I} +1)\tilde{\epsilon}_{TV}
\end{align*}
Thus by triangle inequality,
\[ d_{TV}(\mathcal{L}(X_T^{\nu_{\text{sample}}}| U_{\text{sample}}),  \mu) \leq  d_{TV}(\mathcal{L}(X_T^{\nu_{\text{sample}}}| U_{\text{sample}}),  \tilde{\mu}) + d_{TV}(\mu,\tilde{\mu})  \leq 9 \abs{I} \tilde{\epsilon}_{TV}\]
Letting $\tilde{\epsilon}_{TV}=\frac{\epsilon_{TV}}{9\abs{I}}$ and $M  \geq  2\times 10^4 \abs{I}^3 \epsilon_{TV}^{-3} \log (\abs{I} \tau^{-1}) \geq  20 \abs{I} \tilde{\epsilon}_{TV}^{-3} \log (\abs{I} \tau^{-1})   $ gives the desired result.
 
 In the proof of \cref{prop:process init from sample disjoint case} we will set $\tilde{\tau} = \min\set{\frac{\tilde{\epsilon}_{TV}}{\abs{I}}, p_*/3}$ which implies $\mu(\Omega_r) \geq p_*/3$ and the event $\mathcal{E}$ happens with probability $ 1-\abs{I} \exp(-\frac{p_* \tilde{\epsilon}_{TV}^2 M}{6 }).$ The rest of the argument follows through, and we need to set $M = 6\times 10^2 p_*^{-1} \abs{I}^2 \epsilon_{TV}^{-2} \log (\abs{I} \tau^{-1} )  $ to ensure $\mathcal{E}$ happens with probability $\geq 1-\tau. $
 
\end{proof}
\subsection{Gradient error bound for continuous process} \label{subsec:gradient error bound continuous}
\begin{definition}[Bad set for partition]\label{def:bad set for partition}
Let $ \mathcal{C} = \set{C_1, \dots, C_m}$ be a partition of $S$ i.e. $\bigcup C_r = S$ and $C_{r}\cap C_{r'} =\emptyset$ if $r
\neq r'.$  
For $x\in \R^d$, let $\mu_{\max,S }(x) = \max_{i\in S} \mu_i(x),$  $i_{\max,S }(x) = \arg\max_{i\in S} \mu_i(x)$\footnote{If there are ties, we break ties according to the lexicographic order of $I.$} and $C_{\max}(x) $ is the unique part of the partition containing $ i_{\max,S}(x).$ For $\gamma \in (0,1)$ let 
\[B_{S, \mathcal{C} , \gamma } = \set{x \mid \exists j \in S\setminus C_{\max}:  \mu_{\max, S } (x) \leq \gamma^{-1} \mu_j(x)} \]
If these are clear from context, we omit $S, \mathcal{C} $ in the subscript.
\end{definition}

\begin{lemma} \label{lem:bad set bound}
Fix $S\subseteq I,$ $\mathcal{C}$ is a partition of $S$, and define $B_{\gamma} = B_{S, \mathcal{C}, \gamma}$ as in \cref{def:bad set for partition}. If  $ \delta_{ij} \leq\delta$ for $i, j$ not being in the same part of the partition then
 $\mu(B_{\gamma}) \leq \gamma^{-1} \delta \abs{I}^2/2.$
\end{lemma}

\begin{proposition}[Absolute gradient difference bound]
\label{prop:gradient error absolute bound} Fix $S\subseteq I.$  For $i\in S,$ let $\bar{p}_i = p_i p_S^{-1}$ and recall that $\mu_S(x) =\sum_{i\in S} \bar{p}_i \mu_i(S).$
Let $i:=i_{\max, S}(x) =\arg\max_{i'\in S'}\mu_{i'}(x). $
Suppose $i
\in C \subseteq S$ and for all $j\in S\setminus C$, $\mu_i(x) \geq \gamma^{-1} \mu_j(x).$

 Let $G_S(x)=\max_{i\in S}\norm{ \nabla V_i(x)}$ then 
\[\norm{\nabla V_S(x) - \nabla V_{C}(x)}     \leq \frac{4 \gamma }{\bar{p}_i}G_S(x) \]
\end{proposition} 
In \cref{sec:remove minimum weight assumption}, we will state generalized versions of \cref{def:bad set for partition,lem:bad set bound,prop:gradient error absolute bound}. For proofs of \cref{lem:bad set bound} and \cref{prop:gradient error absolute bound}, refers to proof of \cref{lem:bad set bound without minimum weight,prop:gradient error absolute bound without minimum weight} respectively.

The following proposition shows that if the continuous Langevin process $(\bar{Z}_t^{\delta_x})$ initialized at $x$ doesn't hit the bad set $B_{S, \mathcal{C}, \gamma},$ then the gradient $\nabla V_S(\bar{Z}_t)$ will be close to the gradient $\nabla V_C(\bar{Z}_t)$ where $C$ is the unique part of the partition $\mathcal{C}$ containing $i_{\max,S}(x).$ 
\begin{proposition}\label{prop:error bound for continuous process}
 Fix a set $S.$ Suppose we have a partition $\mathcal{C}$ of $S$ as in \cref{def:bad set for partition}.
Suppose for $i\in S$, $\mu_i$ satisfies item \ref{item:smooth} 
of \cref{lem:single distribution} with $\beta \geq 1$ and
$ \norm{u_i-u_j}\leq L\forall i, j\in S.$ Let $\bar{p}_i = p_S^{-1} p_i,$ and recall that $\mu_S = \sum_{i\in S} \bar{p}_i \mu_i.$
Let $(\bar{Z}_t^{\delta_x})_{t\geq 0}$ be the continuous Langevin diffusion with score function $\nabla V_S$ initialized at $\delta_x.$  
 Fix $\gamma \in (0,1/2).$
Suppose for any $\eta\in (0,1),$ with probability $ 1- 
 \eta/2,$ the event $\mathcal{E}_{\text{discrete}, \eta}$  happens where $\mathcal{E}_{\text{discrete},\eta}$ is defined by: for all $k\in [0,N-1]\cap \N$,
 \[ \norm{\bar{Z}_{kh}^{\delta_x} -u_S}\leq \tilde{L}:= L + \sqrt{\frac{64}{\alpha} \ln \frac{16N}{\eta} } \] and 
 \[ \bar{Z}_{kh}^{\delta_x} \not \in B_{S,\mathcal{C},\gamma}. \]
Let $ T= Nh$ and $C = C_{\max} (x)\in \mathcal{C}$ be the unique part of the partition $\mathcal{C}$ containing $i_{\max,S}(x).$   

Fix $ \eta\in (0,1).$ 
Suppose $T\geq 1$, 
\[ h \leq  \min \set{\frac{1}{ (\beta^2/\alpha)\ln^2 (16T/\eta) }, \frac{1}{40(\beta L)^2}, \frac{1}{2000 d(\beta L)^2 \ln (16T/\eta) }}, \]
$ h \ln (1/h) \leq \frac{1}{2000 d (\beta L)^2}$ and $h\ln^2(1/h)\leq \frac{1}{1000(\beta^2/\alpha) }.$

Then with probability $ 1-\eta,$ 
\[\forall t\in [0,T]: \norm{\nabla V_S(\bar{Z}_t^{\delta_x})-\nabla V_{C} (\bar{Z}_t^{\delta_x}) } \leq \frac{18 \gamma \beta  \tilde{L} }{\min_{i\in C} \bar{p}_{i} } .\]
\end{proposition}
\begin{proof}
By \cref{prop:cluster bound with small distance betwen centers}, $S$ satisfies item 2 of Assumption~\ref{assumption:cluster} with $A_{\text{grad}, 0}  =\beta L$ and $ A_{\text{grad},1} =\beta.$

From \cref{prop:drift bound}, with probability $\geq 1-\eta/2,$ the following event $\mathcal{E}_{\text{drift}, \eta/2}$ happens
\[\sup_{k\in [0,N-1]\cap \N, t\in [0,h]} \norm{\bar{Z}_{kh+t}^{\delta_x} - \bar{Z}_{kh}^{\delta_x} }\leq 4 \beta h L + 2 \sqrt{\left(\frac{64 (\beta h)^2 }{\alpha} + 48 dh\right) 
      \ln \frac{16 N}{\eta}}  \leq 1/(20 \beta \tilde{L})\] 

    Here we use the fact that $\ln(16N/\eta) = \ln(16T/(\eta h) = \ln(16T/\eta)+\ln(1/h) $ thus
    \begin{align*}
       h (\beta L) (\beta \tilde{L}) &\leq h (\beta L)^2 + 2h (\beta L) \beta \sqrt{\frac{64}{\alpha}\ln \frac{16T}{\eta}} +  2h (\beta L) \beta \sqrt{\frac{64}{\alpha}\ln (1/h) } \\
       &\leq  h (\beta L)^2 + 16 \sqrt{h \frac{\beta^2}{\alpha}} \cdot \sqrt{ h (\beta L)^2 \ln \frac{16T}{\eta}} + 16 \sqrt{h \frac{\beta^2}{\alpha}} \cdot \sqrt{h (\beta L)^2 \ln(1/h) } \leq \frac{1}{160}
    \end{align*}
    and
    \begin{align*}
       &\sqrt{\left(\frac{64 (\beta h)^2 }{\alpha} + 48 dh\right) 
      \ln \frac{16 N}{\eta}} \times  (\beta \tilde{L})\\
      &\leq  10 \sqrt{dh} (\sqrt{  \ln \frac{16T}{\eta}} +
      \sqrt{\ln (1/h)} ) \left(\beta L + 2 \beta \sqrt{\frac{64}{\alpha}\ln \frac{16T}{\eta}}  + 2\beta \sqrt{\frac{64}{\alpha}\ln (1/h) }\right)\\
      &\leq  10  \left(\sqrt{h d(\beta L)^2\ln \frac{16T}{\eta} }  + \sqrt{h d(\beta L)^2 \ln(1/h)} + 48 \sqrt{\frac{h \beta^2}{\alpha} (\ln^2 \frac{16T}{\eta}   +\ln^2 (1/h))} \right) \leq \frac{1}{80}
    \end{align*}
Suppose both events $\mathcal{E}_{\text{drift}, \eta/2}$ and $ \mathcal{E}_{\text{drift}, \eta/2}$ happen. By union bound, this occurs with probability $ \geq 1-\eta.$ 
We have, by triangle inequality  
\[\sup_{k\in [0,N-1]\cap \N, t\in [0,h]} \norm{\bar{Z}_{kh+t}^{\delta_x} - u_S} \leq \tilde{L} +  1/(10 \beta \tilde{L})  \leq 1.1 \tilde{L} \]
and for $i\in S$, by item~\ref{item:smooth} of \cref{lem:single distribution} and $\norm{u_i-u_S}\leq L$
\begin{equation}\label{ineq:drift gradient bound in continuous process from samples}
  \norm{\nabla V_i(\bar{Z}_{kh+t}^{\delta_x})} \leq \beta (\norm{\bar{Z}_{kh+t}^{\delta_x}- u_S} +L) \leq 2.2 \beta \tilde{L}.   
\end{equation}

For any $i, j\in S$ and $ t\in [0,h]$
\begin{equation} \label{ineq:ratio drift}
    \begin{split}
     &\log \frac{\mu_j(\bar{Z}_{kh+t}^{\delta_x})}{\mu_i(\bar{Z}_{kh+t}^{\delta_x}) } -  \log \frac{\mu_j(\bar{Z}_{kh}^{\delta_x})}{\mu_i(\bar{Z}_{kh}^{\delta_x}) } \\
     &= V_j(\bar{Z}_{kh}^{\delta_x}) - V_j(\bar{Z}_{kh+t}^{\delta_x}) - (V_i(\bar{Z}_{kh}^{\delta_x}) - V_i(\bar{Z}_{kh+t}^{\delta_x}))  \\
     &\leq  (\norm{\nabla V_i(\bar{Z}_{kh}^{\delta_x})} + \norm{\nabla V_j(\bar{Z}_{kh}^{\delta_x})}) \norm{\bar{Z}_{kh+t}^{\delta_x}-\bar{Z}_{kh}^{\delta_x}}  + \beta \norm{\bar{Z}_{kh+t}^{\delta_x}-\bar{Z}_{k h}^{\delta_x}}^2\\
     &\leq  5 \beta\tilde{L} (20\beta \tilde{L})^{-1} + \beta (20\beta \tilde{L})^{-2} \leq 1/2   
    \end{split}
\end{equation}
where we use the assumption $ \beta \geq 1.$ 

Below we write $i_{\max}$ instead of $i_{\max,S}$ since $ S$ is clear from context.
We first argue by induction on $k$ that $ i_{\max}(\bar{Z}_{kh}^{\delta_x}) \in C.$ The base case $k=0$ holds trivially. Let $y$ be a realization of $\bar{Z}_{kh}^{\delta_x}.$ Condition on $\bar{Z}_{kh}^{\delta_x}=y,$ we argue that $ i_{\max} (\bar{Z}_{(k+1) h}^{\delta_x})\in C_{\max} (y) .$ Since $C_{\max} (y)=C$ by the inductive hypothesis for $k,$ the inductive hypothesis for $k+1$ follows.
Apply \cref{ineq:ratio drift} for $ t = h,$ $i:=i_{\max}(y)$ and $j\not\in C_{\max}(y)$ gives
\[\log \frac{\mu_j(\bar{Z}_{(k+1)h}^{\delta_x})}{\mu_i(\bar{Z}_{(k+1)h}^{\delta_x}) }\leq  \log \frac{\mu_j(\bar{Z}_{kh}^{\delta_x})}{\mu_i(\bar{Z}_{kh}^{\delta_x}) }  + 1/2  = \log \frac{\mu_j(y)}{\mu_{\max} (y)} +1/2 \leq \log \gamma + 1/2 < 0\]
where the penultimate inequality follows from $\bar{Z}_{k h} \not\in B_{\gamma} $ and $j\not \in C_{\max} (y),$ and the final inequality from $\gamma < 1/2.$ Thus, for all $j\not \in C_{\max}(y)$, $\mu_i(\bar{Z}_{(k+1) h}^{\delta_x}) > \mu_j(\bar{Z}_{(k+1) h}^{\delta_x})$ thus $ i_{\max} (\bar{Z}_{(k+1)h}^{\delta_x}) \in C_{\max}(y).$
Finally, we argue for $k \in [0,N-1]\cap\N$ and $t\in (0,h),$ $ i_{\max}(\bar{Z}_{kh+t}^{\delta_x}) \in C$ and $ \bar{Z}_{kh+t}^{\delta_x} \not\in B_{2\gamma}.$  Condition on $ \bar{Z}_{kh}^{\delta_x}=y$, apply \cref{ineq:ratio drift} for $ t = h,$ $i:=i_{\max}(y)$ and $j\not\in C_{\max}(y)=C$ gives
\[\log \frac{\mu_j(\bar{Z}_{kh+t}^{\delta_x})}{\mu_i(\bar{Z}_{kh +t}^x) }\leq  \log \frac{\mu_j(\bar{Z}_{kh}^{\delta_x})}{\mu_i(\bar{Z}_{kh}^{\delta_x}) }  + 1/2  = \log \frac{\mu_j(y)}{\mu_{\max} (y)} +1/2 \leq \log \gamma + 1/2 < \log (2\gamma) \]
thus $\forall j\not \in C: \mu_{\max} (\bar{Z}_{k h +t}^{\delta_x}) \geq \mu_i(\bar{Z}_{k h +t}^{\delta_x}) \geq (2\gamma)^{-1} \mu_j(\bar{Z}_{kh+t }^{\delta_x}).$ Combine this with the bound on $\nabla V_i(\bar{Z}_{kh+t}^{\delta_x})$ in \cref{ineq:drift gradient bound in continuous process from samples} and using \cref{prop:gradient error absolute bound} gives the desired result. Indeed,
\[\norm{\nabla V_S(\bar{Z}_t^{\delta_x}) - \nabla V_C(\bar{Z}_t^{\delta_x}) } \leq \frac{4 \times (2\gamma) G_S(\bar{Z}_t^{\delta_x}) }{\bar{p}_i } \leq \frac{18 \beta \tilde{L} }{\bar{p}_i}. \]

\end{proof}
\section{Analysis of LMC with Approximate Score}\label{sec:discrete}
In this section, we prove the main result (Corollary~\ref{cor:mixing of discrete chain with score error}).
\begin{definition}
Let $\mathbb{H}^L$ be the graph where there is an edge between $i,j$ iff $\norm{u_i - u_j} \leq L. $ 
\end{definition}
\begin{proposition}\label{prop:distance of center in a connected component}
Suppose $ C $ is a connected component of $\mathbb{H}^L$ then for any $i, j\in C,$ $\norm{u_i - u_j}\leq K L.$
\end{proposition}
\begin{proof}
For any $i, j\in C,$ there exists a path $i:=p_0, p_1, \cdots, p_m:=j$ s.t. $\norm{u_{p_s}- u_{p_{s+1}}}\leq L.$ The statement then follows from triangle inequality.
\end{proof}

\subsection{Expected score error bound}
\begin{lemma}\label{lem:expected score error}
Suppose $\mu_i$ satisfies the conditions stated in \cref{lem:single distribution}. Let $u_i$ be as defined in \cref{lem:single distribution}.
Fix $S, R \subseteq I,$ $ S\cap R =\emptyset.$  
Let $p_- = \max_{j\in R} p_j.$
Suppose for $j\in I \setminus (  S\cup R),$
$\norm{u_i -u_j} \geq L\forall i\in S,$ with $L \geq 30 \max\set{\sqrt{\frac{d}{\alpha}}, \kappa \sqrt{d} }\ln (10\kappa) .$  If score estimate $s$ satisfies $\E_{\mu}[\norm{s(x) - \nabla V(x)}^2]\leq \epsilon_{\text{score}}^2$ then
\[\E_{\mu_S}[ \norm{\nabla V(x) - \nabla V_S(x) }^2 ] \leq  3 p_S^{-1} (\epsilon_{\text{score}}^2 + 8 \beta^2 K \exp(-\frac{L^2}{80\kappa})+ 10 K^2  p_- \beta^2 L^2  )  \]

\end{lemma}
\begin{proof}
Since $\nabla V_S(x) = p_S^{-1} \sum_{i\in S} \nabla V_i(x),$ we can write
\begin{align*}
     \norm{\nabla V(x) - \nabla V_S(x)}
    &=(\mu(x) p_S \mu_S(x))^{-1}  \sum_{i\in S, j\not\in S} p_i p_j \mu_i(x) \mu_j(x) \norm{\nabla V_i(x) - \nabla V_j(x)}\\
    &\leq   \sum_{i\in S, j\not \in S: \norm{u_i-u_j}< L} \frac{p_i p_j \mu_i(x) \mu_j(x) \norm{\nabla V_i(x) - \nabla V_j(x)}}{\mu(x) p_S \mu_S(x)} \\
    &\qquad +  \sum_{i\in S, j\not \in S: \norm{u_i-u_j}\geq  L} \frac{p_i p_j \mu_i(x) \mu_j(x) \norm{\nabla V_i(x) - \nabla V_j(x)}}{\mu(x) p_S \mu_S(x)} 
\end{align*}
If $\norm{u_i-u_j}\leq L$ then $ \norm{\nabla V_i(x) - \nabla V_j(x)} \leq \beta (\norm{x-u_i} + \norm{x-u_j}) \leq \beta (2 \norm{x-u_j} + L) $ thus
the first term can be bounded by 
\begin{align*}
    &(p_S\mu_S(x))^{-1} \left( \sum_{i\in S} \sum_{j\not\in S: \norm{u_i-u_j}\leq L} \frac{p_i p_j\mu_i  \mu_j(x) }{\mu(x)} \beta( 2\norm{x-u_j} + L)\right) \\
    &\leq  \beta \sum_{j\not \in S: p_j \leq p_-} \frac{p_j \mu_j(x) 
    (2 \norm{x-u_j}+ L)}{\mu(x) }
\end{align*}
where in the last inequality we use the fact that if $\norm{u_i-u_j}\leq L$ for some $i\in S$ then $j\in R$ and $p_j\leq p_-.$
Hence, by Holder's inequality
\begin{equation}\label{eq:gradient bound 1}
    \E_{\mu_S}[\norm{\nabla V(x) - \nabla V_S(x)}^2] \leq 3 (\E_{\mu_S}[\norm{s(x)- \nabla V(x)}^2]+  A_1 + A_2)
\end{equation}
with 
$A_2 =\E_{\mu_S} [ (\sum_{i\in S, j\in T_2  } \frac{p_i p_j \mu_i(x) \mu_j(x) }{p_S \mu_S(x) \mu(x)})^2  ] $
and
\begin{align*}
  A_1 &=  \E_{\mu_S}\left[\beta^2\left ( \sum_{j\not \in S: p_j \leq p_-} \frac{p_j \mu_j(x) 
    (2\norm{x-u_j}+L) }{\mu(x) } \right)^2 \right ] \\
    &\leq 5 \beta^2 K \sum_{j\not \in S: p_j \leq p_-} \int\mu_S(x)  \left(\frac{p_j\mu_j(x)}{\mu(x)}\right)^2 (\norm{x-u_j}^2 +L^2) dx  \\
    &\leq 5 \beta^2 K   p_S^{-1} \sum_{j\not \in S: p_j \leq p_-} p_j\int  \mu_j(x) (\norm{x-u_j}^2 +L) dx \\
    &\leq 10 \beta^2 K^2 p_S^{-1} p_- L^2
\end{align*}
Now we bound the term $A_2$. Let $T_2 = \set{j:j\not\in S, p_j \geq p_-} .$ 
\begin{align*}
  &\E_{\mu_S} \left[ \left(\sum_{i\in S, j\in T_2  } \frac{p_i p_j \mu_i(x) \mu_j(x) }{p_S \mu_S(x) \mu(x)}\right)^2  \right]\\
  &\leq E_{\mu_S}\left[\frac{\left(\sum_{i\in S, j\in T_2 } p_ip_j \mu_i(x)\mu_j(x)\norm{\nabla V_i(x)-\nabla V_j(x)}^2 \right) \left(  \sum_{i\in S, j\in T_2 } p_ip_j \mu_i (x)\mu_j(x)  \right) }{(p_S\mu_S(x) \mu(x))^2 }  \right] \\
  &= p_S^{-1} \int \sum_{i\in S, j\in T_2 } \frac{p_ip_j \mu_i(x)\mu_j(x)\norm{\nabla V_i(x)-\nabla V_j(x)}^2}{\mu(x) } dx\\
  &=  p_S^{-1} \sum_{i\in S, j\in T_2} p_i \E_{\mu_i} \left[\frac{p_j \mu_j(x)}{\mu(x)} \norm{\nabla V_i(x)-\nabla V_j(x)}^2\right] \\
  &\leq  8 p_S^{-1} K \beta^2\exp(-\frac{L^2}{40\kappa}) 
\end{align*}
where in the last inequality we use \cref{prop:gradient diff for far apart modes}. Plug these inequalities back into \cref{eq:gradient bound 1}, and use \cref{prop:score error partial distribution} gives the desired results.
\end{proof}

 \begin{proposition} \label{prop:score error partial distribution}
 Suppose $s$ satisfies \cref{def:eps-score} then 
 \[\E_{\mu_S} [\norm{s(x) - \nabla V(x)}^2] \leq p_{S}^{-1} \epsilon_{\text{score}}^2 \] 
 \end{proposition}
 \begin{proof}
 \begin{align*}
     p_S \E_{\mu_S}[\norm{s(x) - \nabla V(x)}^2] &= p_S \int \mu_S (x) \norm{s(x) - \nabla V(x)}^2 dx   
     \\
     &\leq \int  (p_S \mu_S(x) + p_{S^c} \mu_{S^c}(x) )  \norm{s(x) - \nabla V(x)}^2 dx\\
     &=  \E_{\mu}[\norm{s(x) - \nabla V(x)}^2]  \leq \epsilon_{\text{score}}^2
 \end{align*}
 \end{proof}
 \begin{proposition}[Pairwise gradient difference for large $\norm{u_i-u_j}$] \label{prop:gradient diff for far apart modes}
Suppose $\mu_i, \mu_j $ satisfies items \ref{item:log concave} and \ref{item:concentration} in \cref{lem:single distribution}.
Let $u_i, u_j$ be as defined in \cref{lem:single distribution} and $ r:= \norm{u_i - u_j } .$  If $\frac{\alpha r ^2/2 + c_z}{17/16 \alpha + \beta } \geq 4 D^2 $
then 
\[ \E_{x\sim \mu_i}\left[\frac{p_j \mu_j(x)}{\mu(x)}\norm{\nabla V_i(x) - \nabla V_j(x)}^2 \right]  \leq 8\beta^2 p_i^{-1} r^2 \exp \left(-\frac{\alpha r ^2 + c_z}{17\alpha + 16 \beta } \right) \]
Consequently, suppose $\mu_i, \mu_j$ are $\alpha$-strongly log concave and $\beta $-smooth with $\beta \geq 1$ and $\kappa =\beta/\alpha,$ and $\norm{u_i-u_j}\geq L$ with $L \geq 30 \max\set{\sqrt{\frac{d}{\alpha}}, \kappa \sqrt{d} }\ln (10\kappa) $ 
\[p_i \E_{x\sim \mu_i} \left[\frac{p_j \mu_j(x)}{\mu(x)}\norm{\nabla V_i(x) - \nabla V_j(x)}^2 \right] \leq 8\beta^2 \exp \left(-\frac{L^2}{80 \kappa}\right) \]
\end{proposition}
\begin{proof}
By \cref{lem:single distribution}, item \ref{item:log concave}
\begin{align*} \frac{ \mu_i(x)}{\mu_j(x)} 
&\geq \exp\left(-\beta \norm{x-u_i}^2 -z_-  + \alpha \norm{x-u_j}^2 +z_+\right) \\
&\geq \exp \left(\frac{\alpha}{2} \norm{u_i-u_j}^2-(\alpha + \beta) \norm{x-u_i}^2  +c_z\right) 
\end{align*}
where the second inequality follows from $ \norm{u_i-u_j}^2/2\leq (\norm{x-u_i} + \norm{x-u_j})^2/2 \leq \norm{x-u_i}^2 + \norm{x-u_j}^2 $
thus
\begin{align*}
    &\frac{p_j \mu_j(x)}{\mu(x)}
    \leq  \frac{p_j \mu_j(x)}{p_j \mu_j(x) + p_i \mu_i(x) }  = \frac{1}{1 + \frac{p_i \mu_i(x)}{p_j \mu_j(x)}}  \leq H(\norm{x-u_i}^2)
\end{align*}
where 
\[H(y) =  \frac{1}{ 1+ p_i \exp (\frac{\alpha}{2} \norm{u_i-u_j}^2-(\alpha + \beta) y  +c_z)  }. \]
Let $A: =  \E_{x\sim \mu_i}[H(\norm{x-u_i}^2)]  $ and $B : = \E_{x\sim \mu_i}[\norm{x-u_i}^2 H(\norm{x-u_i}^2)] .$ Using the fact that
\[\norm{\nabla V_i(x) - \nabla V_j(x)}^2 \leq \beta^2 (\norm{x-u_i} + \norm{x-u_j})^2 \leq 2\beta^2 (4 \norm{x-u_i}^2 + \norm{u_i-u_j}^2) ,\] 
we can bound
\begin{align*}
    \E_{x\sim \mu_i} \left[\frac{p_j \mu_j(x)}{\mu(x)}\norm{\nabla V_i(x) - \nabla V_j(x)}^2 \right]   \leq 2\beta^2 (r^2 A +4 B) 
\end{align*}
First we bound $A.$ We have
\begin{align*}
\E_{\mu_i}[H(\norm{x-u_i}^2)] &= \int_{\norm{x-u_i} \geq R} H(\norm{x-u_i}^2) \mu_i(x) dx + \int_{\norm{x-u_i} < R} H(\norm{x-u_i}^2) \mu_i(x) dx\\
&\leq \mathbb{P}_{x\sim \mu_i}[\norm{x-u_i} \geq R] + H(R^2) \\
&\leq \exp(-\alpha (R-D)^2/4 )  + p_i^{-1}\exp(-\frac{\alpha}{2} r^2 +(\alpha + \beta) R^2 - c_z)  
\end{align*}
where the second inequality follows from $H$ being an increasing function bounded above by $1,$ 
and the third inequality follows from $H(y) \leq p_i ^{-1} \exp (-\alpha r^2 + (\alpha+\beta) y -c_z).$
Set $R^2 = \frac{\alpha r^2/2 + c_z}{\alpha + \beta + \alpha/16}$ then $R\geq 2D$ thus $\exp(-\alpha (R-D)^2/4) \leq \exp(-\alpha R^2/16) = \exp(-\alpha r^2/2 +(\alpha + \beta) R^2 - c_z).$ Hence, the rhs is bounded by 
 $2 p_i^{-1} \exp (-\frac{\alpha r^2/2 + c_z}{17 \alpha + 16\beta }).$
 
 Now we bound $B.$ By Holder's inequality
\begin{align*}
&\E_{\mu_i}[\norm{x-u_i}^2 H(\norm{x-u_i}^2)] \\
&\leq  \sqrt{\E_{\mu_i}[\norm{x-u_i}^4]}\cdot \sqrt{\E_{\mu_i} [ H^2(\norm{x-u_i}^2)] } \\
&\leq  D^2 \sqrt{ \mathbb{P}_{x\sim \mu_i}[\norm{x-u_i} \geq \tilde{R}] + H^2(\tilde{R}^2) } \\
&\leq D^2 \sqrt{\exp(-\alpha (\tilde{R}-D)^2/4 )  + p_i^{-2}\exp(-2\alpha r^2 +2(\alpha + \beta) \tilde{R}^2 - 2c_z)}  
\end{align*}
where we use the sub-Gaussian moment assumption to bound  $\E_{\mu_i}[\norm{x-u_i}^4] $ and the same argument as in the bound for $A$ to bound $\E_{\mu_i} [ H^2(\norm{x-u_i}^2)] ,$ noting that $ H^2(\cdot)$ is also an increasing function bounded above by $1.$ Set $\tilde{R}^2 = \frac{\alpha r^2/2 + c_z}{\alpha + \beta + \alpha/32}$ then $\tilde{R}\geq 2D$ thus $\exp(-\alpha (\tilde{R}-D)^2/4) \leq \exp(-\alpha \tilde{R}^2/16) = \exp(-2 (\alpha r^2 +(\alpha + \beta) \tilde{R}^2 - c_z) ).$ Hence, 
\[B  \leq 2 D^2 p_i^{-1} \exp(-\tilde{R}^2/32) = 2 D^2 p_i^{-1} \exp \left(-\frac{\alpha r^2/2 + c_z}{33 \alpha + 32\beta }\right) \] 

For the second statement, plug in $ D = 5 \sqrt{\frac{d}{\alpha}} \ln(10 \kappa)$ and $ c_z = -\frac{d}{2} \ln(\kappa),$ and use the fact that $\beta\geq 1,$ we have
\[\frac{\alpha r^2/2 + c_z}{17\alpha + 16 \beta } \geq \frac{0.45 \alpha r^2}{33 \beta } \geq 80\times \frac{\beta d}{\alpha}\ln^2 (10\kappa) = 4D^2 \]\
 Thus by \cref{prop:composite of exponential and polynomial terms} and the fact that $L^2 \geq 2\kappa,$
$r^2 \exp (-\frac{\alpha r^2/2 + c_z}{17\alpha + 16 \beta } ) \leq r^2 \exp(-\frac{r^2}{80 \kappa }) \leq L^2  \exp(-\frac{L^2}{80 \kappa })$
\end{proof}

\begin{theorem}\label{thm:discrete mixing for cluster of distribution with close centers}
Suppose each $\mu_i$ is $\alpha$ strongly-log-concave and $\beta$-smooth for all $i\in I$ with $\beta\geq 1.$ Recall that $\abs{I} = K.$
Let $u_i=\arg\min V_i(x)$, $p_* = \min_{i\in I} p_i$, $\kappa = \beta/\alpha.$ Set
\[L_0 =  \Theta\left( \kappa^2 K \sqrt{d} (\ln(10 \kappa) + \exp(K) \ln (d p_*^{-1}\epsilon_{TV}^{-1}) ) \right) =\tilde{\Theta}(\kappa^2 K \exp(K) \sqrt{d}).\]
Let $S$ be a connected component of $\mathbb{H}^L$, where there is an edge between $i,j$ 
if $\norm{u_i-u_j}\leq  L:=L_0/(\kappa K).$ 
Let $U_{\text{sample}}$ be a set of $M$ i.i.d. samples from $\mu_S$ and $ \nu_{\text{sample}}$ be the uniform distribution over $U_{\text{sample}}.$
Let $ (X_{nh}^{\nu_{\text{sample}}})_{n \in \N}$ be the LMC with score $s$ and step size $h$ initialized at $ \nu_{\text{sample}}.$
Set
\[T = \Theta \left(\alpha^{-1} K^2 p_*^{-1}\ln(10 p_*^{-1}) \left( \frac{10^8 d (\beta L_0)^3 {\exp(K) \ln^{3/2} (p_*^{-1}) \ln^{5}  \frac{16 d (\beta L_0)^2 }{\epsilon_{TV}\tau  \alpha } } }{p_*^{7/2} \epsilon_{TV}^3  \alpha^{3/2} } \right)^{2((3/2)^{K-1} -1)}  \right)\]
Let the step size $ h = \Theta \left(\frac{\epsilon_{TV}^4 }{(\beta L_0)^4 dT}\right) = \tilde{\Theta}\left(\frac{\epsilon_{TV}^4}{(\beta\kappa^2 K \exp(K) )^4  d^3 T  }\right) .$
Suppose $ s$ satisfies \cref{def:eps-score} with 
\begin{align*}
  &\epsilon_{\text{score}}\\
  \leq &\frac{p_*^{1/2} \epsilon_{TV}^2 \sqrt{h} }{7 T}  \\
  = &\Theta (\frac{p_*^{1/2} \epsilon_{TV}^4 }{(\beta L_0)^ 2 T^{3/2} }) \\
  = & \tilde{\Theta}\left(\frac{p_*^{1/2} \epsilon_{TV}^4 }{(\beta \kappa^2 K \exp(K) )^2 d^{3/2} T^{3/2}  }\right)\\
  = &\Theta \left(\frac{p_*^{2} \epsilon_{TV}^4 \alpha^{3/2}  }{K^3 \ln^{3/2} (10 p_*^{-1}) (\beta L_0)^2  }  \left( \frac{p_*^{7/2} \epsilon_{TV}^3  \alpha^{3/2} } {10^8 d (\beta L_0)^3 {\exp(K) \ln^{3/2} (p_*^{-1}) \ln^{5}  \frac{16 d (\beta L_0)^2 }{\epsilon_{TV}\tau  \alpha } } }\right)^{3((3/2)^{K-1} -1)} \right)   
\end{align*}

 Suppose the number of samples $M$
 satisfies $M \geq 4000 p_*^{-1} \epsilon_{TV}^{-4} K^2 \log (K \epsilon_{TV}^{-1})  \log(\tau^{-1}) ,$ then
\[\mathbb{P}_{U_{\text{sample}}} [d_{TV} (\mathcal{L}(X_{T}^{\nu_{\text{sample}}} \mid U_{sample}), \mu_S) \leq \epsilon_{TV} ] \geq 1-\tau\]

\end{theorem}
\begin{corollary} \label{cor:mixing of discrete chain with score error}
Suppose  $\mu_i$ is $\alpha$ strongly-log-concave and $\beta$-smooth for all $i$ with $\beta \geq 1.$
Let $p_* = \min_{i\in I} p_i.$
Suppose $ s$ satisfies \cref{def:eps-score}. Let $U_{\text{sample}}$ be a set of $M$ i.i.d. samples from $\mu$ and $ \nu_{\text{sample}}$ be the uniform distribution over $U_{\text{sample}}.$ With $ T, h, \epsilon_{\text{score}}^2$ as in \cref{thm:discrete mixing for cluster of distribution with close centers} and $M \geq 20000 p_*^{-2} \epsilon_{TV}^{-4} K^2 \log (K \epsilon_{TV}^{-1})  \log(K \tau^{-1})  $. Let $ (X_{nh}^{\nu_{\text{sample}}})_{n \in \N}$ be the LMC with score $s$ and step size $h$ initialized at $ \nu_{\text{sample}},$ then 
\[\mathbb{P}_{U_{\text{sample}}} [d_{TV} (\mathcal{L}(X_{T}^{\nu_{\text{sample}}} \mid U_{\text{sample}}), \mu) \leq \epsilon_{TV} ] \geq 1-\tau\]
\end{corollary}
\begin{proof}
    This is a consequence of \cref{thm:discrete mixing for cluster of distribution with close centers} and \cref{prop:from one distribution to mixture of distribution}. Here we apply \cref{prop:from one distribution to mixture of distribution}  with
    \[ M_0 = 4000 p_*^{-1} \epsilon_{TV}^{-4} K^2 \log (K \epsilon_{TV}^{-1})  \log(\tau^{-1}). \]
\end{proof}
\begin{proof}[Proof of \cref{thm:discrete mixing for cluster of distribution with close centers}]
Let $u_i = \arg\min_x V_i(x)$ then $\nabla V_i(u_i) = 0.$ W.l.o.g. we can assume $ V_i(u_i) = 0.$ 
By \cref{prop:distance of center in a connected component}, $\norm{u_i-u_j} \leq \hat{L}:=KL=L_0/\kappa$ for $i,j\in S.$ By \cref{prop:cluster bound with small distance betwen centers}, with $u_S = p_S^{-1} \sum_{i\in S} p_i u_i,$ $S$ satisfies Assumption~\ref{assumption:cluster} with $ A_{\text{grad}, 1} = \beta$, $ A_{\text{grad},0} = \beta \hat{L}$, $A_{\text{Hess}, 1} =2 \beta^2  $ and $A_{\text{Hess}, 0} = 2 \beta^2 \hat{L}^2. $ 

We first show the statement for $M = M_0 := 600 p_*^{-1} \epsilon_{TV}^{-2} K^2 (\log (K \tau^{-1}) ),$ where we set  $\tau =\epsilon_{TV},$  then use \cref{prop:boost small sample batch to big batch} to obtain the result for $M \geq  4000 p_*^{-1} \epsilon_{TV}^{-4} K^4 \log (K \epsilon_{TV}^{-1}) ) \log(\tau^{-1}) \geq 6 M_0 \epsilon_{TV}^{-2} \log(\epsilon_{\text{sample }}^{-1}) .$   

From this point onward set $\tau =\epsilon_{TV}$ and $M= M_0$ as defined above.
Let $(X_{nh}^{\mu_S})_{n \in \N}$ be the LMC with score estimate $s$ and step size $h$ initialized at $\mu_S$ and $(\bar{X}_t^{\mu_S})_{t\geq 0}$ be the continuous Langevin diffusion with score $\nabla V_S$ initialized at $\mu_S.$ Let $Q_T$ and $\bar{Q}_T$ denote the distribution of the paths  $(X_{nh}^{\mu_S})_{n\in [0,T/h]\cap\N}$ and  $(\bar{X}_t^{\mu_S})_{t\in [0,T]}.$  Note that $ L \geq 50\kappa \sqrt{d} \ln(10\kappa) \geq 10 D,$ so
 \cref{lem:bound TV between discretize and continuous starting from the muS} gives
 \begin{align*}
  2 d_{TV}(Q_T, \bar{Q}_T)^2 
  &\leq 2 h^2 T \beta^6 \hat{L}^6 +2 h T d \beta^4 \hat{L}^4 + T  \epsilon_{\text{score},0}^2  
 \end{align*}
 with $\epsilon_{\text{score},0}^2 :=3 p_S^{-1} (\epsilon_{\text{score}}^2 + 8 \beta^2  K \exp(-\frac{L^2}{80\kappa}) )  .$
 Let $\epsilon_{\text{score},1}^2=\frac{\epsilon_{TV}^2}{8T}$ and  $B = \set{z: \norm{s(z) - V_S(z)} > \epsilon_{\text{score},1 } }  $ then by Markov's inequality $\mu(B) \leq \frac{\epsilon_{\text{score},0}^2}{\epsilon_{\text{score},1}^2} = \frac{8\epsilon_{\text{score},0}^2 T}{\epsilon_{TV}^2}.$
 
 Let $\eta = \epsilon_{TV} \tau.$ Suppose 
 $T \epsilon_{\text{score},0}^2 \leq \eta^2/100$
and $ h \leq (100)^{-1}\min \set{\frac{\eta}{(\beta \hat{L})^3 \sqrt{T} }, \frac{\eta^2}{(\beta \hat{L})^4 d T }}$ then $d_{TV}(Q_T, \bar{Q}_T)\leq \eta/4,$ thus
 \begin{align*}
   \mathbb{P}[\exists n \in [0,N-1]\cap \N: X_{nh}^{\mu_S} \in B ] 
   &\leq  \mathbb{P}[\exists n \in [0,N-1]\cap \N: \bar{X}_{nh}^{\mu_S} \in B ] + d_{TV}(Q_T, \bar{Q}_T) \\
   &\leq  A_1:= \frac{T}{h} \times \frac{8 T \epsilon_{\text{score},0}^2}{\epsilon_{TV}^2} +\epsilon_{TV} \tau/4   
 \end{align*}
Since $ \E_{U_{\text{sample}}} [\nu_{\text{sample}}] = \mu_S,$
  $\mathcal{L}(X_{nh}^{\mu_S}) = \E_{U_{\text{sample}}} [\mathcal{L}(X_{nh}^{\nu_{\text{sample}}} | U_{\text{sample}} ) ] $ and
 \begin{align*}
   \E_{U_{\text{sample}}} [\mathbb{P}_{\mathcal{F}_n}[\exists n \in [0,N-1]\cap \N: X_{nh}^{\nu_{\text{sample}}} \in B ] ]
   = \mathbb{P}[\exists n \in [0,N-1]\cap \N: X_{nh}^{\mu_S} \in B ]\leq A_1
 \end{align*} 
 By Markov's inequality, let $ \mathcal{E}_0$ be the event $\mathbb{P}_{\mathcal{F}_n}[\exists n \in [0,N-1]\cap \N: X_{nh}^{\nu_{\text{sample}}} \in B ]\leq 2A_1/\tau$ 
 then
 \[\mathbb{P}_{U_{\text{sample}}}[ \mathcal{E}_0 \text{ occurs} ] \geq 1- \tau/2
 \]
 
 Suppose $  \mathcal{E}_0$ occurs.
  Let $ \nu:=\nu_{\text{sample}}.$  
Let $(Z_{nh}^{\nu_{\text{sample}}})_{n\in \N }$ be the LMC initialized at $\nu$  with score estimate $s_{\infty}$ defined by
\[s_{\infty} (z) = \begin{cases} s(z) \text{ if } z \not\in B\\ \nabla V_S(z)  \text{ if } x \in B  \end{cases} \]
then $ \sup_{z\in \R^d}\norm{s_{\infty}(z) -\nabla V_S(z) }^2\leq \epsilon_{\text{score},1}^2. $ 

Note that if $ X_{nh} \not\in B\forall n \in [0,N-1]\cap \N$ then $Z_{nh}^{\nu_{\text{sample}}}= X_{nh }^{\nu_{\text{sample}}}\forall n\in [0,N]\cap \N $ thus
conditioned on  $  \mathcal{E}_0$ occurs, $d_{TV} ( Z_{Nh}^{\nu_{\text{sample}}}, X_{N h}^{\nu_{\text{sample}}})\leq 2A_1/\tau.$ Let $(\bar{Z}_{t}^{\nu_{\text{sample}}})_{t}$ be the continuous Langevin with score $\nabla V_S$ initialized at $\nu.$ We want to bound $d_{TV} (Z_{Nh}^{\nu_{\text{sample}}}, \bar{Z}_T^{\nu_{\text{sample}}}).$
 By sub-Gaussian concentration of $\mu_i$ and union bound over $M$ samples, we have with probability $\geq 1-\tau/3,$ the following event $\mathcal{E}_1$ happens: 
\[\sup_{x \sim \nu} \max_{i\in S} \norm{x-u_i}  \leq \tilde{L}:= 2\hat{L} + \sqrt{\frac{4}{\alpha} \ln ( \frac{8M}{\tau }) } \leq 3 \hat{L}  \]
since for $\beta\geq 1,$ $\sqrt{\frac{4}{\alpha^3} \ln ( \frac{8M}{\tau}) } \leq 3 \kappa^{3/2} \sqrt{\ln (\epsilon_{TV}^{-1})}  \leq \hat{L}. $

Let $\mathcal{E}_2$ be the event
\[d_{TV}(\mathcal{L}(\bar{Z}_T^{\nu_{\text{sample}}} | U_{\text{sample}} ), \mu_S) \leq \epsilon_{TV}/4\]
By \cref{lem:well separated cluster}, if $M \geq 605 (p_*  \epsilon_{TV}^2)^{-1} K^2 \log (K\tau^{-1} ) $ then $\mathbb{P}_{U^{\text{sample}}} [\mathcal{E}_2] \geq 1 -\tau /6.$

Suppose $\mathcal{E}_0, \mathcal{E}_1, \mathcal{E}_2$  all hold; by union bound, this happens with probability $ \geq 1- \tau.$ By \cref{lem:lmc with linfity error}, for
\begin{align*}
L_0 &= \hat{L} + \kappa \tilde{L}+ \sqrt{\frac{d}{\alpha} \ln ((2\alpha h)^{-1} ) } + \sqrt{(16/\alpha + 200 dh) \ln (\frac{8T}{h})}
\end{align*}
we have
\begin{align*}
  d_{TV} (Z_{Nh}^{\nu_{\text{sample}}}, \bar{Z}_T^{\nu_{\text{sample}}})^2 &\lesssim h^2 T \beta^6 L_0^6  + h T d  \beta^4 L_0^4 + \epsilon_{\text{score}, 1}^2 T/2
  \leq \epsilon^2_{TV}/64
\end{align*}
if  $100 h \leq\frac{\epsilon_{TV}^2 }{(\beta L_0 )^4 d T } \leq \frac{\epsilon_{TV}}{(\beta L_0 )^3 \sqrt{T} }.$

By triangle inequality
\begin{align*}
   &d_{TV} (\mathcal{L}(X_{nh}^{\nu_{\text{sample}}} | U_{\text{sample}}), \mu_S) \\
   &\leq d_{TV} (X_{nh}^{\nu_{\text{sample}}}, Z_{Nh}^{\nu_{\text{sample}}}) + d_{TV} (Z_{Nh}^{\nu_{\text{sample}}}, \bar{Z}_T^{\nu_{\text{sample}}}) + d_{TV} (\mathcal{L}(\bar{Z}_T^{\nu_{\text{sample}}}| U_{\text{sample}}), \mu_S) \\
   &\leq \frac{16 T^{2} \epsilon_{\text{score}, 0}^2 }{h \epsilon_{TV}^2 \tau } + \epsilon_{TV}/2 +  \epsilon_{TV}/8 + \epsilon_{TV}/4 \leq \epsilon_{TV}
\end{align*}
if $ \frac{16 T^{2} \epsilon_{\text{score}, 0}^2 }{h \epsilon_{TV}^2 \tau } \leq \epsilon_{TV}/8.$

 
Our choice of parameters satisfies all the conditions mentioned above.
Since $h\leq 1/(100\beta d)$, $\beta\geq 1$ and $\hat{L}\geq \sqrt{\frac{d}{\alpha} }\ln(10 \kappa) \geq \sqrt{\frac{d}{\alpha} \ln (\alpha^{-1}) }.$ we can bound $L_0$ by
\begin{align*}
  L_0 &\leq 2 \hat{L} \kappa  + \sqrt{\frac{d}{\alpha} \ln ((2\alpha h )^{-1} ) }  + \sqrt{\frac{17}{\alpha} \ln (\frac{8T}{h})} \\
  &\leq 3 \hat{L} \kappa + \sqrt{\frac{17}{\alpha} \ln(8T)} + 2 \ln (1/h) \sqrt{\frac{d + 1}{\alpha}}\\
  &=  3 \hat{L} \kappa +  \exp(K) \sqrt{\kappa d} \ln(p_*^{-1} d \epsilon_{TV}^{-1} L_0 \kappa K )
\end{align*}
where we use the bound on $T$ and $h$ to bound $\ln(T)$ and $\ln(1/h).$

Set $L =50 \kappa \sqrt{d} \ln (10\kappa) + \sqrt{\kappa d} \exp(K) \ln (d \kappa p_*^{-1}  \epsilon_{TV}^{-1}  ) .$ Since $3 \kappa \hat{L} \leq L_0 \leq 5 \kappa \hat{L} $,
\[L_0 = \Theta (\kappa^2 \sqrt{d}  (  \ln (10\kappa) + \exp(K) \ln (d p_*^{-1}  \epsilon_{TV}^{-1}  ) ) )  .\]

We need to check that  $ h \lesssim \frac{\epsilon_{TV}^4 }{(\beta L_0)^4  dT } $ but this is true due to the choice of $h.$ Next, we need to check $\frac{16 T^{2} \epsilon_{\text{score}, 0}^2 }{h \epsilon_{TV}^2 \tau } \leq \epsilon_{TV}/8$ and $T^{1/2} \epsilon_{\text{score}, 0} \leq \eta/10=\epsilon_{TV}^2/10.$ We note that the former implies the latter, and the 
latter is true since
\[\epsilon_{\text{score}} \leq \frac{p_{S}^{1/2} \epsilon_{TV}^2 \sqrt{h}}{7 T } \]
and 
\[ p_S^{-1/2} \beta \sqrt{K} \exp(-\frac{L^2}{160\kappa}) T \leq \sqrt{h} \epsilon_{TV}^2/{20},\]
which in turn is implied by
\[ L/\sqrt{\kappa} \geq \exp(K) \ln (d\kappa p_*^{-1} \epsilon_{TV}^{-1} ) \geq 5 \ln ( T h^{-1}  \beta K p_*^{-1} \epsilon_{TV}^{-1} ) \]
which is true for our choice of $L$, $h$ and $T.$

\end{proof}
\begin{lemma}\label{lem:lmc with linfity error}
Fix $ S\subseteq I.$ For $u_i$ and $D$ as defined in \cref{lem:single distribution}, suppose $\norm{u_i-u_j} \leq L\forall i, j\in S$ with $L \geq 10D.$ 
Let $\nu_0$ be a distribution s.t. $ \sup_{x\sim \nu_0}  \max_{i\in S} \norm{x-u_i} \leq \tilde{L}.$ Let $(\bar{Z}_{t}^{\nu_0})_{t\geq 0}$  the continuous Langevin with score $\nabla V_S$ 
initialized at $\nu_0.$
Let  
$(Z_{nh}^{\nu_0} )_{n\in \N}$ be the LMC with step size $h$ and score $s_{\infty}$  s.t. $\sup_{x\in \R^d} \norm{s(x) - \nabla V_S(x)}\leq \epsilon_{\text{score}, 1}^2.$ Suppose $h\leq 1/(30\beta)$ then for  $\tilde{D} := 6 L + O\left(\kappa \tilde{L} + \sqrt{\frac{d}{\alpha} \ln((2\alpha h)^{-1})   }\right) + \sqrt{(\frac{16}{\alpha} +200 dh) \ln (8 N)} , $ we have
\[d_{TV} (Z_{T}^{\nu_0}, \bar{Z}_T^{\nu_0})^2 \leq   h^2 T \beta^6 \tilde{D}^6  + h T d  \beta^4 \tilde{D}^4 + \epsilon_{\text{score}, 1}^2 T/2 \] 
\end{lemma}
\begin{proof}
To simplify notations, we omit the superscript $\nu_0$ and write $Z_{nh}$ and $\bar{Z}_t$ in the proof instead of $Z_{nh}^{\nu_0}$ and $\bar{Z}_t^{\nu_0}.$
Let $\bar{\nu}_h$ be the distribution of $ \bar{Z}_h.$
First, we bound $ \Renyi_2(\bar{\nu}_h || \mu_S).$ By \cref{lem:continuous initialization},
\[\Renyi_2(\bar{\nu}_h||\mu_S) \leq O(\alpha^{-1}(\beta \tilde{L})^2 + d \ln((2\alpha h)^{-1})   \]
By \cref{prop:cluster bound with small distance betwen centers}, let $u_S = p_S^{-1} \sum_{i\in S} p_i u_i$ then $\mu_S$ satisfies Assumption~\ref{assumption:cluster} so 
\[\mathbb{P}_{\mu_S}[\norm{x-u_S} \geq 1.1 L +t ]  \leq\exp(-\alpha t^2/4). \]
Let $N =T/h.$
By 
the change of measure argument in \cite[Lemma 24]{chewi2021analysis}, with probability $\geq 1-\eta/2$
\begin{align*}
  \max_{k \in [1,N-1]\cap\N } \norm{\bar{Z}_{kh} - u_S} &\leq   1.1 L + \sqrt{\frac{2}{\alpha} \Renyi_2(\bar{\nu}_h||\mu_S) }  + \sqrt{\frac{4}{\alpha}  \ln \frac{8 N}{\eta}}\\
  &\leq 1.1 L +\kappa \tilde{L} + \sqrt{\alpha^{-1} d \ln ((2\alpha h)^{-1})  } + \sqrt{\frac{4}{\alpha}  \ln \frac{8 N}{\eta}}.
\end{align*}
By \cref{prop:drift bound}, this implies that with probability $\geq 1-\eta$, for $\gamma =\frac{16}{\alpha} + 200 dh  $
\[\sup_{t\in [0,T]} \norm{\bar{Z}_t - u_S} \leq \tilde{D} + \sqrt{\gamma \ln (1/\eta) }\]
with $\tilde{D} := 6 L + O(\kappa \tilde{L} + \sqrt{\alpha^{-1} d \ln((2\alpha h)^{-1})   }) + \sqrt{\gamma  \ln (8 N)} . $
By \cref{prop:moment bound for subgaussian concentration}, this implies, for $p=O(1)$
\[\E[\norm{\bar{Z}_t -u_S}^p ]\lesssim (\tilde{D} + \sqrt{\gamma})^p \lesssim \tilde{D}^p \]
where we use the fact that $\sqrt{\gamma} \leq \sqrt{\frac{d+16}{\alpha}}\leq \tilde{D}/50.$

By \cref{prop:pertubation bound helper}, for $t\in [kh, (k+1) h],$
\begin{align*}
    &\E [\norm{\nabla V(\bar{Z}_{k h}) - \nabla V(\bar{Z}_{t}) }^2 ]\\
    &\lesssim  \sqrt{\E[A_{\text{Hess}, 1}^4 (\norm{\bar{Z}_{kh} - u_S}^8 + \norm{\bar{Z}_{t} - u_S}^8 ) + A_{\text{Hess}, 0}^4  ]} \\
&\qquad \times \sqrt{ (t-kh)^3\int_{kh}^t ( A_{\text{grad}, 1}^4  \E[\norm{\bar{Z}_s-u_S}^4] + A_{\text{grad}, 0}^4) ds + d^2 (t-kh)^2   }\\
    &\lesssim  (A_{\text{Hess}, 1}^2 \tilde{D}^4 + A_{\text{Hess}, 0}^2) (h^2  (A_{\text{grad}, 1}^2 \tilde{D}^2 + A_{\text{grad}, 0}^2) + dh)\\
    &\lesssim \beta^4 (\tilde{D}^4 + L^4 ) ( h^2 \beta^2 (\tilde{D}^2 + L^2) +dh) \\
    &\lesssim \beta^4 \tilde{D}^4  (h^2 \beta^2 \tilde{D}^2 + dh)
\end{align*}
where in the second inequality, we use the moment bounds for $\norm{\bar{Z}_s -u_S}$, in the third inequality, we use \cref{prop:cluster bound with small distance betwen centers} to substitute in the parameters $A_{\text{Hess}, 1}, A_{\text{Hess}, 0}, A_{\text{grad}, 1}, A_{\text{grad}, 0}$, and in the final bound, we use $ \tilde{D} \geq 6L.$
Then by Girsanov's theorem (see \cref{lem:approximation argument})
\begin{align*}
&2 d_{TV} (Z_{T}^{\nu_0}, \bar{Z}_T^{\nu_0})^2 \\
&\leq \E[\int_{0}^T \norm{s(\bar{Z}_{\lfloor t/h\rfloor h}) - \nabla V(\bar{Z}_{t}) }^2 dt] \\
&\lesssim \epsilon_{\text{score},1}^2 T +\E[\int_{0}^T \norm{\nabla V(\bar{Z}_{\lfloor t/h\rfloor h}) - \nabla V(\bar{Z}_{t}) }^2 dt] \\
&\lesssim  \epsilon_{\text{score}, 1}^2 T+  h^2 T \beta^6 \tilde{D}^6  + h T d  \beta^4 \tilde{D}^4.
\end{align*}

\end{proof}
\begin{lemma} \label{lem:bound TV between discretize and continuous starting from the muS}
Suppose the score estimate $s$ satisfies \cref{def:eps-score}. Let $u_i$ and $D$ be defined as in \cref{lem:single distribution}. Let $S$ be a connected component of $\mathbb{H}^L$ with $L\geq 10D.$
Let $(X_{nh}^{\mu_S})_{n \in \N}$ be the LMC with score estimate $s$ and step size $h$ initialized at $\mu_S$ and $(\bar{X}_t^{\mu_S})_{t\geq 0}$ be the continuous Langevin diffusion with score $\nabla V_S$ initialized at $\mu_S.$
Let $T= Nh,$ $Q_T$ and $\bar{Q}_T$ denote the distribution of the paths of $(X_{nh}^{\mu_S})_{n\in [0,T/h]\cap\N}$ and $(\bar{X}_t^{\mu_S})_{t\in [0,T]}.$ Then for $\hat{L}=LK,$
\begin{align*}
2d_{TV}(\bar{Q}_T, Q_T)^2 &\leq
   \E\left[\int_{0}^T \norm{s(\bar{X}^{\mu_S}_{\lfloor t/h\rfloor h }) - \nabla V_S(\bar{X}^{\mu_S}_t) }^2 dt \right]\\
   &\lesssim 2 h^2 T \beta^6 \hat{L}^6 +2 h T d \beta^4 \hat{L}^4 + T  \epsilon_{\text{score},0}^2  
 \end{align*}
 with $\epsilon_{\text{score},0}^2 :=3 p_S^{-1} (\epsilon_{\text{score}}^2 + \beta^2 L^2 8 K^3 \exp(-\frac{L^2}{40\kappa}) )  .$

\end{lemma}
\begin{proof}
 By \cref{prop:distance of center in a connected component}, $\norm{u_i-u_j}\leq \hat{L}$ for $i,j\in S$ and $ \norm{u_i-u_j} > L$ for $i\in S,j\not\in S.$
Note that since $\mu_S $ is the stationary distribution of the continuous Langevin diffusion with score $\nabla V_S,$ the law of $\bar{X}^{\mu_S}_t$ is $ \mu_S$ at all time $t.$
Thus, for $t\in [kh, (k+1)h]$

\begin{equation} \label{ineq:one step error bound initialized at stationary}
\begin{split}
    &\E[\norm{s(\bar{X}^{\mu_S}_{k h }) - \nabla V_S(\bar{X}^{\mu_S}_t) }^2]\\
    &\leq 2 (\E[\norm{s(\bar{X}^{\mu_S}_{k h }) - \nabla V_S(\bar{X}^{\mu_S}_{k h }) }^2  ] + \E[\norm{\nabla V_S(\bar{X}^{\mu_S}_{k h }) - \nabla V_S(\bar{X}^{\mu_S}_{t }) }^2 ]\\
    &\leq 2 (\epsilon_{\text{score}, 0}^2 +  \beta^4 \hat{L}^4 ( h^2 \beta^2 \hat{L}^2 + dh))
\end{split}
\end{equation}
where in the second inequality, we use \cref{lem:expected score error} with $R=\emptyset$ to bound the first term and \cref{prop:drift bound} and \cref{prop:moment bound for subgaussian concentration} to bound the second term. The argument is similar to the one in the proof of \cref{lem:lmc with linfity error}. Let $u_S=p_S^{-1} \sum_{i\in S} p_i u_i$ then $ \norm{u_i-u_S}\leq L\forall i\in S.$  For $\tilde{D} = D+\hat{L}\leq 1.1 \hat{L}$ and $\gamma = \frac{4}{\alpha}$, since the law of $\bar{X}^{\mu_S}_{t } $ is $\mu_S,$ by \cref{prop:cluster bound with small distance betwen centers}
 \[\mathbb{P}[ \norm{\bar{X}^{\mu_S}_{t } - u_S} \geq\tilde{D} + \sqrt{\gamma \ln(1/\eta)}  ]\leq \eta \]
thus by \cref{prop:moment bound for subgaussian concentration} and $ \tilde{D} \geq \sqrt{100/\alpha},$ for $p=O(1),$ $\E[\norm{\bar{X}^{\mu_S}_{t } - u_S}^p] \lesssim \tilde{D}^p.$ By \cref{prop:drift bound},
\begin{align*}
    \E[\norm{\nabla V_S(\bar{X}^{\mu_S}_{k h }) - \nabla V_S(\bar{X}^{\mu_S}_{t }) }^2 ]&\leq  \beta^4 (\tilde{D}^4 + \hat{L}^4 ) ( h^2 \beta^2 (\tilde{D}^2 + \hat{L}^2) +dh) \\
    &\lesssim \beta^4 \hat{L}^4  (h^2 \beta^2 \hat{L}^2 + dh)
\end{align*}
The statement follows from integrating \cref{ineq:one step error bound initialized at stationary} from $0$ to $T$ and  Girsanov's theorem (see \cref{lem:approximation argument}).
\end{proof}

This proposition is used in \cref{thm:discrete mixing for cluster of distribution with close centers} to go from a set of samples of fixed size $M_0$ to a set of samples with size $M $ that can be arbitrarily large.

\begin{proposition}\label{prop:boost small sample batch to big batch}
  Fix distributions $\mu_{\text{sample}}, \mu.$
    For a set $U_{\text{sample}}\subseteq \R^d$,  let $ (X_t^{\nu_{\text{sample}} })_t$ be a process initialized at $\nu_{\text{sample}},$ the uniform distribution over
    $U_{\text{sample}}.$
Suppose there exists $T> 0, \epsilon_{TV}\in (0,1)$ 
s.t. with probability $ \geq 1-\epsilon_{TV}/10$ over the choice of  $U_{\text{sample}}$ consisting of $M_0$ i.i.d. samples from $\mu_{\text{sample}} ,$  $d_{TV}(\mathcal{L} (X_T^{\nu_{\text{sample}} } | U_{\text{sample}}), \mu) \leq \epsilon_{TV}/10.$
Then, for $M \geq  6 \epsilon_{TV}^{-2} M_0 \log (\tau^{-1}), $ with probability $ \geq 1-\tau $ over the choice of  $U_{\text{sample}}$ consisting of $M$ i.i.d. samples from $\mu_{\text{sample}} ,$  
\[ d_{TV}( \mathcal{L} (X_T^{\nu_{\text{sample}} } | U_{\text{sample}}), \mu) \leq \epsilon_{TV}/2. \]
\end{proposition}

\begin{proof}
Let $U_{\text{sample}}$ be a set of $M$ i.i.d. samples $x^{(1)},\cdots, x^{(M)}$ from $\mu_{\text{sample}} .$ For $ r \in \set{1, \cdots, \lfloor M/M_0\rfloor}$ Let $U_r = \set{x^{(i)} : (r-1)M_0 +1 \leq r M_0 }$ and $U_{\emptyset} = U_{\text{sample}} \setminus \bigcup_r U_r.$  Let $\nu_r$ be the uniform distribution over $U_r$ and $\nu_{\emptyset}$ be the uniform distribution over $U_{\emptyset}.$ For $m = \lfloor M/M_0\rfloor $ 
\[\nu = \frac{M_0}{M} \sum_r \nu_r + \frac{M - M_0 m }{M} \nu_{\emptyset}\]
Let $ \Omega$ be the set of $U\in (\R^d)^{M_0}$ s.t. $ d_{TV}(X_T^\nu, \mu)\leq \epsilon_{TV}/2 $ with  $\nu$ being the uniform distribution over $U.$ 

Similar to the proof of \cref{prop:process init from sample disjoint case}, if we choose $M/M_0 \geq 6\epsilon_{TV}^{-2} \log (\tau^{-1}), $ then with probability $ \geq 1-\tau,$ $\abs{\set{r: U_r\in\Omega }} \geq m(1-\epsilon_{TV}/5).$ By \cref{prop:tv distance mixture bound}, 
\begin{align*}
   d_{TV} (\mathcal{L} (X_T^{\nu_{\text{sample}} } | U_{\text{sample} }), \mu ) &\leq \sum_{r: U_r \in \Omega } \frac{M_0}{M} d_{TV} (  \mathcal{L} (X_T^{\nu_{r }}), \mu) + \frac{M - M_0 m(1-\epsilon_{TV}/5) }{M} \\
   &\leq \epsilon_{TV}/10 + \epsilon_{TV}^2/6 +\epsilon_{TV}/5 \leq \epsilon_{TV}/2
\end{align*}
where in the penultimate inequality, we use the definition of $ \Omega ,$ $M_0m\leq M$ and  $M-m_0 M \leq M_0 \leq \epsilon_{TV}^2 M/6.$
\end{proof}

The following proposition combined with \cref{thm:discrete mixing for cluster of distribution with close centers} implies \cref{cor:mixing of discrete chain with score error}.
\begin{proposition}\label{prop:from one distribution to mixture of distribution}
For a set $U_{\text{sample}}\subseteq \R^d$,  let $ (X_t^{\nu_{\text{sample}} })_t$ be a process initialized at 
the uniform distribution over
    $U_{\text{sample}}.$ 
  Consider distributions  $\mu_C$ for $C \in \mathcal{C}.$ 
   Let $\mu = \sum p_C \mu_C$ with $p_C> 0$ and $\sum p_C = 1.$  Let $p_*=\min p_C.$
 Suppose there exists $T> 0, \epsilon_{TV}\in (0,1)$ 
s.t. with probability $ \geq 1-\frac{\tau}{10\abs{\mathcal{C}}}$ over the choice of  $U_{C, \text{sample}}$ consisting of $M \geq M_0$ i.i.d. samples from $\mu_C,$ $d_{TV}(\mathcal{L}(X_T^{\nu_{C, \text{sample}}} | {U_{C, \text{sample}}} ), \mu_C) \leq \epsilon_{TV}/10,$ where $\nu_{C, \text{sample}}$ is the uniform distribution over ${U_{C, \text{sample}}} .$
Then, for $M \geq  \min p_*^{-1} \set{M_0, 20 \epsilon_{TV}^{-2} \log (\abs{\mathcal{C} }\tau^{-1})},$ with probability $ \geq 1-\tau$ over the choice of $U_{\text{sample}}$ consisting of $M$ i.i.d. samples from $\mu,$ 
\[ d_{TV} (\mathcal{L} (X_T^{\nu_{\text{sample}} } | U_{ \text{sample}} ), \mu )\leq \epsilon_{TV}. \]
\end{proposition}
\begin{proof}[Proof of \cref{prop:from one distribution to mixture of distribution}]
    Since $\mu=\sum_C p_C\mu_C,$ a sample $x^{(i)}$ from $\mu$ can be drawn by first sampling  $ C^{(i)} \in \mathcal{C}$ from the distribution defined by the weights $ \set{p_C}_{C\in \mathcal{C}},$ then sample from $\mu_{C^{(i)}}.$ Consider $M$ i.i.d. samples $x^{(i)}$ using this procedure, and let $ U_C = \set{x^{(i)}: C^{(i)} = C}.$ Since $M \geq 20 p_*^{-1} \epsilon_{TV}^{-2} ,$ and $\E[\abs{U_C}] = p_C M ,$ by Chernoff's inequality and union bound, with probability $ 1-\tau/2 $ over the randomness of $  U_{\text{sample}},$ the following event $\mathcal{E}_1$ holds
    \[\forall C: \abs{\frac{\abs{U_C} }{M} - p_C } \leq p_C \epsilon_{TV}/2 \]
    Suppose $\mathcal{E}_1$ holds. Then, $\frac{\abs{U_C} }{M} \geq p_C (1-\epsilon_{TV}/2) M\geq M_0.$ Thus by union bound, with probability $ 1-\epsilon_{TV}/10$ over the randomness of $  U_{\text{sample}},$ the following event $\mathcal{E}_2$ holds with $\nu_C$ be the uniform distribution over $U_C$
    \[ \forall C: d_{TV}(\mathcal{L}(X^{\nu_C}_T | U_C), \mu_C) \leq \epsilon_{TV}/10 \]
    then  let $\tilde{\mu} = \sum_C \frac{\abs{U_C} }{M} \mu_C,$ by part 1 of \cref{prop:tv distance mixture bound},
    \begin{align*}
       d_{TV} (\mathcal{L}(X^{\nu_{\text{sample}}}_T | U_{\text{sample}}), \tilde{\mu}) = d_{TV}\left(\sum_C \frac{\abs{U_C} }{M} \mathcal{L}(X^{\nu_C}_T | U_C)   , \tilde{\mu}\right) \leq \sum_C \frac{\abs{U_C}}{M} \epsilon_{TV}/10 =\epsilon_{TV}/10 
    \end{align*}
    and $d_{TV} (\tilde{\mu},\mu) \leq \sum_{C} \abs{\frac{\abs{U_C} }{M} - p_C} \leq \epsilon_{TV}/2.$
    Condition on $\mathcal{E}_1$ and $\mathcal{E}_2$ both hold, which happens with probability $ 1-\tau,$ we have
    \begin{align*}
        d_{TV} (\mathcal{L}(X^{\nu_{\text{sample}}}_T | U_{\text{sample}}), \mu) &\leq d_{TV} (\mathcal{L}(X^{\nu_{\text{sample}}}_T | U_{\text{sample}}), \tilde{\mu})+ d_{TV}(\tilde{\mu}, \mu) \\
        &\leq \epsilon_{TV}/10 + \epsilon_{TV}/2 \leq \epsilon_{TV} 
    \end{align*}
\end{proof}
\section{Removing the dependency on $ p_*=\min_{i\in I} p_i.$} \label{sec:remove minimum weight assumption}
In this section, we remove the dependency on the minimum weight $ p_*=\min_{i\in I} p_i.$ 
The idea is to consider only the components $\mu_i$ with significant weight $p_i$ i.e. $p_i\geq p_{\text{threshold}}$ for some chosen threshold $p_{\text{threshold}}.$ In \cref{lem:well separated cluster without minimum weight lower bound,thm:continuous mixing modified,thm:discrete mixing for cluster of distribution with close centers modified,cor:mixing of discrete chain with score error modified}, we prove analogs of \cref{lem:well separated cluster,thm:continuous mixing,thm:discrete mixing for cluster of distribution with close centers,cor:mixing of discrete chain with score error} respectively with no dependency on $p_*.$

We will need modified versions of \cref{lem:bad set bound} and \cref{prop:gradient error absolute bound}, which are \cref{lem:bad set bound without minimum weight} and \cref{prop:gradient error absolute bound without minimum weight} respectively. 

\begin{definition}[Bad set for partition (modified)]\label{def:bad set for partition modified}
Fix $S\subset I, C_*\subseteq S,S'=S\setminus C_*.$ Suppose we have a partition $ \mathcal{C} = \set{C_1, \dots, C_m}$ of $S'.$  
For $x\in \R^d$, 
let
$i_{\max, S'}(x) = \arg\max_{i\in S'} \mu_i(x)$ and $\mu_{\max,S'}(x) = \mu_{i_{\max, S'}(x) } = \max_{i\in S'} \mu_i(x)$ as in \cref{def:max index}. 
Let $C_{\max,S'}(x) $ is the unique part of the partition $ \mathcal{C} $ containing $i_{\max,S'}(x).$ For $\gamma \in (0, 1), \gamma_* > 0$ let 
\begin{align*}
  &\tilde{B}_{S, C_*,\mathcal{C} , \gamma, \gamma_* } \\
  &= \set{x \lvert \exists j \in S' \setminus C_{\max,S'}(x): \mu_{\max,S'}(x) \leq \gamma^{-1} \mu_j(x) \text{ or } \exists j\in C_*: \mu_{\max,S'}(x) \leq\gamma_*^{-1} \mu_j(x) }  
\end{align*}
Note that if $ C_*=\emptyset$ then $\tilde{B}_{S, C_*,\mathcal{C} , \gamma, \gamma_* }  = B_{S,\mathcal{C},\gamma}  
$ as defined in \cref{def:bad set for partition}.
If they are clear from context, we omit $S, C_*, \mathcal{C} $ in the subscript.
\end{definition}

\begin{lemma}[Bad set bound (generalized version of \cref{lem:bad set bound})] \label{lem:bad set bound without minimum weight}
Fix $S\subseteq I,$ $C_*\subseteq C$, $\mathcal{C}$ be a partition of $S' = S\setminus C_*.$ Let $p_S = \sum_{i \in S} p_i$ and $\bar{p}_i =p_i p_S^{-1}.$ Recall that $ \mu_S = \sum_{i\in S} \bar{p}_i\mu_i.$ 
For $\gamma, \delta \in (0,1)$, define $\tilde{B}_{\gamma} = \tilde{B}_{S, C_*, \mathcal{C}, \gamma, \gamma_*}$ as in \cref{def:bad set for partition modified} with $ \gamma_*^{-1}= \gamma^{-1} \delta K/8$ .
Suppose
\begin{enumerate}
\item If $ i\in C_*$ then $ \bar{p}_i \leq \delta/8 $
    \item
    $ \delta_{ij} \leq\delta$ for $i, j$ which are in $S'$ and are not in the same part of the partition $\mathcal{C}$ of $S'$
\end{enumerate}
then
 $\mu_S (\tilde{B}_{\gamma}) \leq \gamma^{-1} \delta K^2.$
\end{lemma}
\begin{proof}[Proof of \cref{lem:bad set bound,lem:bad set bound without minimum weight}]
We prove \cref{lem:bad set bound without minimum weight}, then \cref{lem:bad set bound} follows immediately by setting $ C_*=\emptyset$ in \cref{def:bad set for partition modified}.

 Consider $ x\in \tilde{B}_{\gamma}$ s.t. $i_{\max,S'}(x) =  i.$ 
 For $j\in S'$, let $C(j)$ denote the unique part of the partition $\mathcal{C}$ containing $j.$
 Let $k = i_{\max 2, S'} (x)  = \arg\max_{j\in S'\setminus C(i)}\mu_j (x
 ).$ 
 If $j\in C(i)$ then by definition of $i_{\max,S'}(x) = i,$ $ \mu_j(x) \leq \mu_{i}(x).$ If $j\in  S'\setminus C(i),$ then by definition of $ k,$
$\mu_j(x) \leq \mu_k(x).$ Let 
\[ B'_{\gamma} = \set{x \mid \exists j \in S' \setminus C_{\max,S'}(x): \mu_{\max,S'}(x) \leq \gamma^{-1} \mu_j(x)} \] and \[ B_* = \set{x\mid \exists j \in C_*: \mu_{\max, S'} \leq \gamma_*^{-1} \mu_j(x) }. \]
Let $\bar{p}_j = p_j p_S^{-1}$ for $j\in S.$
 If $x\in B'_{\gamma}$, $\mu_i(x) \leq \gamma^{-1} \mu_k(x),$ and for 
\begin{align*}
  \mu_S(x) &= \sum \bar{p}_j \mu_j(x)= \sum_{j \in C(i)} p_j \mu_j(x) + \sum_{j \in S'\setminus C(i)} \bar{p}_j \mu_j(x) + \sum_{j\in C_*} \bar{p}_j \mu_j(x) \\
  &\leq \sum_{j \in C(i)} \bar{p}_j \mu_i(x)  +\sum_{j \in S'\setminus C(i)} \bar{p}_j \mu_k (x) + \sum_{j\in C_*} \bar{p}_j \mu_j(x)   \\
  &\leq \sum_{j \in S'} \bar{p}_j \gamma^{-1}\mu_k(x) + \sum_{j\in C_*}\bar{p}_j \mu_j(x)\\
  &\leq  \gamma^{-1} \mu_k(x) + \sum_{j\in C_*} \bar{p}_j \mu_j(x)
\end{align*}
Let $\bar{p}_{C_*} := \sum_{j \in C_*} \bar{p}_j$ then $ \bar{p}_{C_*} \leq K\times  \delta/8 \leq \gamma^{-1}\delta K/8$ since $\gamma^{-1} >1.$
For $i,k \in S'$, let $\Omega_{i,k}$ be the set of $x$ s.t. $i_{\max,S'}(x) = i$ and $i_{\max 2,S'}(x) = k.$ Since $\set{\Omega_{i,k} \lvert i,k\in S' , C(i) \neq C(k)}$ forms a partition  of $\R^d,$ we have
\begin{align*}
    \mu_S(B'_{\gamma}) 
    &= \sum_{i,k\in S': C(i) \neq C(k)} \int_{x\in B_{\gamma}\cap \Omega_{i, k}} \mu_S(x) dx\\
    &\leq \sum_{i,k: C(i) \neq C(k)} \int_{x\in B_{\gamma}\cap \Omega_{i, k}} (\gamma^{-1}\mu_k(x) + \sum_{j\in C_*} \bar{p}_j \mu_j(x))dx \\
    &= \gamma^{-1}\sum_{i< k: C(i) \neq C(k)} \left(\int_{x\in B_{\gamma}\cap \Omega_{i, k}}  \mu_k(x)  dx + \int_{x\in B_{\gamma}\cap \Omega_{k, i}}\mu_i(x) dx \right) \\
    &\qquad + \sum_{j\in C_*} \bar{p}_j \left (\sum_{i,k} \mu_j(B_{\gamma}\cap \Omega_{i,k}) \right)\\
    &= \gamma^{-1}\sum_{i< k: C(i) \neq C(k)} \int_{x\in B_{\gamma}\cap (\Omega_{i, k} \cup \Omega_{k,i}) }  \min\set{\mu_i(x), \mu_k(x)}  dx + \sum_{j\in C_*} \bar{p}_j\\
    &\leq \gamma^{-1}\sum_{i< k: C(i) \neq C(k)}\delta + \gamma^{-1}\delta K /8\\
    &\leq \gamma^{-1} \delta K^2/2 + \gamma^{-1} \delta K/8
\end{align*}
where in the penultimate inequality, we use the fact that $\delta_{ik} \leq \delta$  for $i,k$  which are not in $C_*$ and  not  in the same part of the partition, 
and $p_j \leq\delta K/2 \leq \gamma^{-1}\delta K/2$ for $j\in C_*.$ 

For $i \in C_*,$ let $\Omega^*_{i}$ be the set of $x$ s.t. $i_{\max, C_*} = i.$ If $ x\in \Omega^*_i\cap B_*$ then
\[\mu_S(x)  = \sum_{j\in  C_*} \bar{p}_j \mu_j(x) + \sum_{j\in S'} \bar{p}_j \mu_j(x) \leq \sum_{j\in C_*}   \bar{p}_j\mu_i(x) +  \sum_{j\in S'} \bar{p}_j \gamma_*^{-1}\mu_i(x) = \mu_i(x) (\bar{p}_{C_*} + \gamma_*^{-1}).  \]
Thus
\begin{align*}
    \mu_S(B_*) &= \sum_{i\in C_*} \int_{x\in B_*\cap \Omega^*_i } \mu_S (x) dx \\
    &\leq \sum_{i\in C_*}  \int_{x\in B_*\cap \Omega^*_i }  (\bar{p}_{C_*} +\gamma_*^{-1}) \mu_i(x) dx  \\
    &\leq (\bar{p}_{C_*} +\gamma_*^{-1}) \sum_{i\in C_*}\mu_i(B_* \cap \Omega^*_i)  \leq 
    (\gamma^{-1}\delta K/8 + \gamma^{-1}\delta K/8) K
\end{align*}
where in the last inequality we use the definition of $\gamma_*$ and the fact that $ \mu_i(B_* \cap \Omega^*_i)\leq 1.$ Thus by union bound 
\[\mu_S(\tilde{B}_{S,C_*, \mathcal{C},\gamma, \gamma_*})\leq \mu_S(B'_{\gamma}) + \mu_S(B_*)\leq \gamma^{-1}\delta K^2. \]
\end{proof}

\begin{proposition}[Absolute gradient difference bound (generalized version of \cref{prop:gradient error absolute bound})]
\label{prop:gradient error absolute bound without minimum weight} Fix $S\subseteq I,$ $C_*\subseteq S.$ Let $S'=S\setminus C_*.$ For $i\in S,$ let $\bar{p}_i = p_i p_S^{-1}$ and recall that $\mu_S(x) =\sum_{i\in S} \bar{p}_i \mu_i(S).$ Suppose $\bar{p}_j\leq \frac{ \delta }{8 }$ for $j\in C_*.$ 
Let $i:=i_{\max, S'}(x) =\arg\max_{i'\in S'}\mu_{i'}(x). $
Suppose $i
\in C \subseteq S'$ and
\begin{enumerate}
    \item $\mu_{i}(x) \geq \gamma^{-1} \mu_j(x)\forall j\in S'\setminus C$
    \item $\mu_i(x) \geq\gamma_*^{-1}  \mu_j(x) \forall j \in  C_* $ where $\gamma_*^{-1}  = \gamma^{-1} \delta K/8.  $
\end{enumerate}
 Let $G_S(x)=\max_{i\in S}\norm{ \nabla V_i(x)}$ then 
\[\norm{\nabla V_S(x) - \nabla V_{C}(x)}     \leq \frac{4\gamma }{\bar{p}_i}G_S(x) \]
\end{proposition} 
\begin{proof}[Proof of \cref{prop:gradient error absolute bound without minimum weight} and \cref{prop:gradient error absolute bound}]
We prove \cref{prop:gradient error absolute bound without minimum weight}, then \cref{prop:gradient error absolute bound} follows immediately by setting $ C_*=\emptyset.$ 
For $C'\subseteq S,$ let $ \bar{p}_{C'} = \sum_{i\in C'} \bar{p}_i.$ By \cref{prop:gradient}, we can write
\begin{align*}
     \nabla V_S(x) - \nabla V_{C}(x) &= \frac{\bar{p}_C\mu_C(x) \nabla V_C(x) + \sum_{j \in S \setminus C} \bar{p}_j \mu_j(x) \nabla V_j(x) }{ \mu_S(x) } - \nabla V_C(x)\\
     &= \frac{\bar{p}_C\mu_C(x) \nabla V_C(x) + \sum_{j \in S \setminus C} \bar{p}_j \mu_j(x) \nabla V_j(x) }{\bar{p}_C \mu_C(x) + \sum_{j \in S \setminus C} \bar{p}_j \mu_j(x) } - \nabla V_C(x)\\
     &= \sum_{j \in S \setminus C}\frac{\bar{p}_j\mu_j(x)}{\bar{p}_C \mu_C(x) + \sum_{j \in S \setminus C} \bar{p}_j \mu_j(x) } (\nabla V_j(x) -\nabla V_C(x))
\end{align*}
For $j \in S' \setminus C,$
\[\frac{\bar{p}_C \mu_C(x) + \sum_{j'\in S\setminus C} \bar{p}_j \mu_j(x) }{\bar{p}_j\mu_j(x)} \geq \frac{\bar{p}_i \mu_i(x) }{\bar{p}_j \mu_j(x)}\geq \frac{\bar{p}_i}{\bar{p}_j}\gamma^{-1} \]
and for $ j\in C_*,$ using the upper bound on $p_j$ and the assumption $\mu_i(x) \geq \gamma_*^{-1}\mu_j(x) $
\[\frac{\bar{p}_C \mu_C(x) + \sum_{j' \in S\setminus C} \bar{p}_{j'} \mu_j(x) }{\bar{p}_j\mu_j(x)} \geq \frac{\bar{p}_i \mu_i(x) }{\bar{p}_j \mu_j(x)} \geq   \frac{\bar{p}_i \gamma_*^{-1}}{\bar{p}_j} \geq \bar{p}_i K\gamma^{-1} \]
Next, by \cref{prop:gradient}, $ \norm{\nabla V_C(x)}\leq G_S(x)$ thus,
\[ \norm{ \nabla V_S(x) - \nabla V_{C}(x)}\leq  2G_S(x) \gamma \left(\sum_{j\in S\setminus (C\cup C_*)} \frac{\bar{p}_j}{\bar{p}_i} + \sum_{j\in C_*} \frac{1}{K \bar{p}_i} \right) \leq \frac{4\gamma G_S(x)}{\bar{p}_i}  \]

\end{proof}

The following is a modified version of \cref{lem:well separated cluster}.
\begin{lemma}\label{lem:well separated cluster without minimum weight lower bound}
Fix $ \epsilon_{TV} , \tau \in(0,1/2), \delta \in (0,1].$ Fix $S\subseteq I.$ 
Let $\bar{p}_i=p_i p_S^{-1}$ and recall that $\mu_S = \sum_{i\in S} \bar{p}_i\mu_i.$ 
Suppose for $i\in S,$ $\mu_i$ are $\alpha$-strongly log-concave and $ \beta$-smooth with $\beta \geq 1.$
Let $u_i =\arg\min_x V_i(x)$ and $D\geq 5 \sqrt{\frac{d}{\alpha}}$ be as defined in \cref{lem:single distribution}. Suppose there exists $L\geq 10 D$ such that for any $i,j\in S,$
$ \norm{u_i-u_j}\leq L.$  
Fix $p_*>0.$
Let $ S' = \set{i\in S: \bar{p}_i \geq p_*}$ and $C_*=S\setminus S'.$
Let $\mathbb{G}^{\delta} :=\mathbb{G}^{\delta}(S', E )$ be the graph on $S'$ with an edge between $i,j$ iff $ \delta_{ij}\leq \delta.$ 
Let  
\[ T= \frac{2 C_{p_*,K} }{\delta \alpha }\left( \ln (\frac{\beta^2 L}{\alpha}) + \ln \ln \tau^{-1}   + 2\ln \tilde{\epsilon}_{TV}^{-1}\right). \]
and
\[\delta' =   \frac{\delta^{3/2} \alpha ^{3/2}  p_*^{5/2} \epsilon_{TV}^2 \tau  }{10^5  K^5 d (\beta L)^3 \ln^{3/2} (p_*^{-1})\ln^{3/2} \frac{\beta^2 L \epsilon_{TV}^{-1} \ln \tau^{-1}}{\alpha}  \ln^{2.51} \frac{16 d (\beta L)^2 }{\epsilon_{TV}\tau  \delta \alpha }  }      .\]
Suppose $ \max_{i\in C_*}\bar{p}_i \leq \delta'/8 $ and for all $i, j$ in $S'$ that are not in the same connected component of $\mathbb{G}^{\delta}$, $\delta_{ij} \leq\delta'.$   

For $x\in\R^d$, let $ (\bar{X}_{t}^{\delta_x})_{t\geq 0}$ denote the continuous Langevin diffusion with score $\nabla V_S$ initialized at $\delta_x.$ Let $C_{\max,S'}$ be the unique connected component of $\mathbb{G}^{\delta}$ containing $ i_{\max,S'} (x) =\arg\max_{i'\in S'} \mu_{i'}(x).$
\[\P_{x\sim\mu_S}[d_{TV} (\mathcal{L}(\bar{X}_{T}^{\delta_x} | x) , \mu_{C_{\max, S'}(x)} ) \leq \epsilon_{TV} ] \geq 1-\tau\]

\end{lemma}
\begin{proof}
The proof is same as \cref{lem:well separated cluster}, but we replace \cref{lem:bad set bound} with \cref{lem:bad set bound without minimum weight}, \cref{prop:gradient error absolute bound} with \cref{prop:gradient error absolute bound without minimum weight} and \cref{prop:init partition bound} with \cref{prop:init partition bound modified}. Note that we use $\gamma = \frac{p_* \epsilon_{TV}}{100\tilde{L} \sqrt{T}}$ and $\tilde{B}_{\gamma}$ as defined in \cref{lem:bad set bound without minimum weight}  to ensure that for $ y\not \in \tilde{B}_{\gamma},$ $\norm{\nabla V_{C_{\max,S'} (y)} (y)- \nabla V_S(y)}\leq \frac{4 \gamma (\beta L \sqrt{\ln (\beta L \epsilon_{TV}^{-1} \tau^{-1} T)} )  }{p_*} \leq \frac{\epsilon_{TV}}{10\sqrt{T}}$ so that we can bound the total variation distance between the continuous Langevin diffusions with scores $\nabla V_S$ and $\nabla V_{C_{\max,S'}(x)}$ by $\epsilon_{TV}/10.$ 
\end{proof}
\begin{theorem}\label{thm:continuous mixing modified}
Fix $ \epsilon_{TV} , \tau \in(0,1/2).$ Fix $ S \subseteq I.$
Suppose for $i\in S,$ $\mu_i$ are $\alpha$-strongly log-concave and $ \beta$-smooth with $\beta \geq 1.$ Let $u_i =\arg\min_x V_i(x)$ and $D\geq 5 \sqrt{\frac{d}{\alpha}}$ be as defined in \cref{lem:single distribution}. Suppose there exists $L\geq 10 D$ such that for any $i,j\in S,$
$ \norm{u_i-u_j}\leq L.$  
 Let $U_{\text{sample}}$ be a set of $M$ i.i.d. samples from $\mu_S$ and $ \nu_{\text{sample}}$ be the uniform distribution over  $U_{\text{sample}}.$
Let $(\bar{X}_t^{\nu_{\text{sample}}})_{t\geq 0 }$ be the continuous Langevin diffusion with score $\mu_S$ initialized at $\nu_{\text{sample}}.$ 
 For $M \geq 10^5 (\epsilon_{TV}^3)^{-1} K^3 \log (K \tau^{-1}) $ and
 \[T \geq  \Theta \left(\alpha^{-1}\left(\frac{10^8 d (\beta L)^3 {\exp(K) \ln^{5}  \frac{16 d (\beta L)^2 }{\epsilon_{TV}  \alpha } } }{ \epsilon_{TV}^3 \alpha^{3/2} }\right)^{\exp(20 (K+1) )} \right) \]
 then 
 \[\P_{U_{\text{sample}}} [d_{TV} (\mathcal{L}(\bar{X}_t^{\nu_{\text{sample}}}| U_{\text{sample}}),\mu_S) \leq  \epsilon_{TV} ] \geq 1-\tau \]
\end{theorem}
\begin{proof}
For $i\in S,$ let $\bar{p}_i = p_i p_S^{-1}.$
As in \cref{lem:well separated cluster without minimum weight lower bound}, fix $p_{0,*} = \frac{1}{K}$ and let $S_0'= \set{i\in S: \bar{p}_i \geq p_{0,*}}, C_{0,*}= S\setminus S_0'$ then $S_0'\neq \emptyset,$ since there must be at least one $i$ s.t. $\bar{p}_i \geq \frac{1}{K}.$ 
By the same argument as in proof of \cref{thm:continuous mixing}, we take the sequence $1=\delta_{0,0} >\delta_{0,1} >\cdots > \delta_{0,K}$ where we use the notation $\delta_{0,s}$ to emphasizes its dependency on $p_{0,*}.$ If $ \max_{i\in C_{0,*}} \bar{p}_i <   \frac{\delta_{0,K} }{8} $ then \cref{lem:well separated cluster without minimum weight lower bound} applies. More precisely, we will use \cref{prop:process init from sample disjoint case modified} and the inductive argument on $ \delta_{0,s}$ as in the proof of \cref{thm:continuous mixing} to show that the continuous Langevin diffusion initialized at $M$ samples will converge to $\mu_S$ after a suitable time $T$ defined by  $ \delta_{0,K-1}.$
If this is not the case, then we let $p_{1,*} = \frac{\delta_{0,K} }{8}$ and $S_1'= \set{i\in S: p_i \geq p_{1,*}}$ then $\abs{S'_1}\geq \abs{S'_0}+1. $ In general, we inductively set $ p_{s+1,*} =\frac{\delta_{s,K} }{8}.$ If $ \max_{i\in C_{s, *}}p_i\leq p_{s+1, *}$ for some $s\leq K-2$ then we are done, else $C_{K-1, *} =\emptyset$ thus $\min_{i\in S} \bar{p}_i \geq p_{K-1, *}$ and we can use \cref{thm:continuous mixing}. In all cases, for $p_* = p_{K-1,*}$, the continuous Langevin diffusion initialized at samples converges to $\mu_S$ after time 
\begin{align*}
     T&\geq \Theta \left(\alpha^{-1} K^2 p_*^{-1}\ln(10 p_*^{-1})\delta_{K-1, K-1}^{-1}\right )\\
     &= \Theta (\alpha^{-1} \Xi^{-\exp(20 (K+1) )} )
\end{align*}
To justify the above equation, we lower bound $p_* = p_{K-1,*}$ and $\delta_{K-1, K-1}. $

Let $\tilde{\Gamma}_s =  \frac{p_{s,*}^{7/2} \epsilon_{TV}^3 \alpha^{3/2} }{8000 d (\beta L)^3 {\exp(K) \ln^{3/2} (p_{s,*}^{-1}) \ln^{5}  \frac{16 d (\beta L)^2 }{\epsilon_{TV} \alpha } } }  \geq  \frac{p_{s,*}^{3.51} \epsilon_{TV}^3 \alpha^{3/2} }{10^5 d (\beta L)^3 {\exp(K) \ln^{5}  \frac{16 d (\beta L)^2 }{\epsilon_{TV}  \alpha } } }  $ then \[\delta_{s, K-1} > \delta_{s, K} \geq \tilde{\Gamma}_s^{2 ((3/2)^{K+1}-1)}\geq p_{s,*}^{7.02((3/2)^{K+1}-1)} \Xi \]
with $\Xi = (\frac{ \epsilon_{TV}^3 \alpha^{3/2} }{10^5 d (\beta L)^3 {\exp(K) \ln^{5}  \frac{16 d (\beta L)^2 }{\epsilon_{TV}  \alpha } } } )^{2 ((3/2)^{K+1}-1)}$
and we can prove by induction on $s$ that
\[p_{s,*} \geq K^{-\exp(10 (s+1) )} \Xi^{\exp(2(s+1)) } \geq \Xi^{\exp(4.9 (s+1) ) },\] thus 
\begin{align*}
    \delta_{K-1, K-1}^{-1} &\leq (p_{K-1, *}^{7.02((3/2)^{K+1}-1)} \Xi)^{-1} \\
    &\leq (\Xi^{\exp(4.9 (K+1)  ) })^{-7.02((3/2)^{K+1}-1)} \cdot \Xi^{-1} \\
    &\leq  \Xi^{-\exp(12 (K+1) )} 
\end{align*}

\end{proof}
\begin{theorem}\label{thm:discrete mixing for cluster of distribution with close centers modified}
Suppose each $\mu_i$ is $\alpha$ strongly-log-concave and $\beta$-smooth for all $i$ with $\beta\geq 1.$ Let $u_i=\arg\min V_i(x).$ 
Set
\[L_0 =  \Theta\left( \kappa^2 K \sqrt{d} (\ln(10 \kappa) + \exp(60 K) \ln (d \epsilon_{TV}^{-1}) )\right) .\]
Let $S$ be a connected component of $\mathbb{H}^L$, where there is an edge between $i,j$ 
if $\norm{u_i-u_j}\leq  L:=L_0/(\kappa K).$ 
Let $U_{\text{sample}}$ be a set of $M$ i.i.d. samples from $\mu_S$ and $ \nu_{\text{sample}}$ be the uniform distribution over $U_{\text{sample}}.$
Let $ (X_{nh}^{\nu_{\text{sample}}})_{n \in \N}$ be the LMC with score $s$ and step size $h$ initialized at $ \nu_{\text{sample}}.$
Set
\[T =  \Theta \left(\alpha^{-1}\left(\frac{10^8 d (\beta L_0)^3 {\exp(K) \ln^{5}  \frac{16 d (\beta L_0)^2 }{\epsilon_{TV}  \alpha } } }{ \epsilon_{TV}^3 \alpha^{3/2} }\right)^{\exp(20 (K+1))} \right) \]
Let the step size $ h = \Theta (\frac{\epsilon_{TV}^4 }{(\beta L_0)^4 dT}) .$ Suppose $p_S \geq \frac{\epsilon_{TV}}{K}$ and
 $ s$ satisfies \cref{def:eps-score} with $ \epsilon_{\text{score}} \leq \frac{\epsilon_{TV}^{5/2} \sqrt{h} }{7 \sqrt{K} T } \leq \frac{p_S^{1/2} \epsilon_{TV}^2 \sqrt{h} }{7 T}  .$ 
 Suppose the number of samples $M$
 satisfies $M \geq  10^7 \epsilon_{TV}^{-5} K^3 \log (K \epsilon_{TV}^{-1})  \log(\tau^{-1}) ,$ then
\[\P_{U_{\text{sample}}} [d_{TV} (\mathcal{L}(X_{T}^{\nu_{\text{sample}}} \mid U_{sample}), \mu_S) \leq \epsilon_{TV} ] \geq 1-\tau\]

\end{theorem}
\begin{proof}
The proof is identical to proof of \cref{thm:discrete mixing for cluster of distribution with close centers}, but we plug in $T$ from \cref{thm:continuous mixing modified} instead of \cref{thm:continuous mixing}. With the same setup as in proof of \cref{thm:discrete mixing for cluster of distribution with close centers}, $\epsilon_{\text{score}, 0}^2 =3 p_S^{-1} (\epsilon_{\text{score}}^2 + 8 \beta^2  K \exp(-\frac{L^2}{80\kappa}) ),$ thus as long as we assume $p_S \geq \frac{\epsilon_{TV}}{K},$ we can ensure that with our choice of $L$ and $ \epsilon_{\text{score}},$
$\epsilon_{\text{score}, 0} \leq \frac{p_{S}^{1/2} \epsilon_{TV}^2 \sqrt{h}}{7 T} $ as required.
\end{proof}
\begin{corollary} \label{cor:mixing of discrete chain with score error modified}
Suppose  $\mu_i$ is $\alpha$ strongly-log-concave and $\beta$-smooth for all $i$ with $\beta \geq 1.$ Suppose $ s$ satisfies \cref{def:eps-score}. Let $U_{\text{sample}}$ be a set of $M$ i.i.d. samples from $\mu$ and $ \nu_{\text{sample}}$ be the uniform distribution over $U_{\text{sample}}.$ With $ T, h, \epsilon_{\text{score}}^2$ as in \cref{thm:discrete mixing for cluster of distribution with close centers modified} and $M \geq 10^8 \epsilon_{TV}^{-6} K^4 \log (K \epsilon_{TV}^{-1})  \log(K \tau^{-1})  $. Let $ (X_{nh}^{\nu_{\text{sample}}})_{n \in \N}$ be the LMC with score $s$ and step size $h$ initialized at $ \nu_{\text{sample}},$ then 
\[\P_{U_{\text{sample}}} [d_{TV} (\mathcal{L}(X_{T}^{\nu_{\text{sample}}} \mid U_{\text{sample}}), \mu) \leq \epsilon_{TV} ] \geq 1-\tau\]
\end{corollary}
\begin{proof}
    This is a consequence of \cref{thm:discrete mixing for cluster of distribution with close centers modified} and \cref{prop:from one distribution to mixture of distribution modified}. Here we apply \cref{prop:from one distribution to mixture of distribution modified} 
    with $M_0 = 10^7 \epsilon_{TV}^{-5} K^3 \log (K \epsilon_{TV}^{-1})  \log(K \tau^{-1}) .$
\end{proof}
To remove dependency on $p_*,$ we will use the following variant of \cref{prop:from one distribution to mixture of distribution}.
\begin{proposition}\label{prop:from one distribution to mixture of distribution modified}
For a set $U_{\text{sample}}\subseteq \R^d$,  let $ (X_t^{\nu_{\text{sample}} })_t$ be a process initialized at 
the uniform distribution $\nu_{\text{sample}}$ over
    $U_{\text{sample}}.$ 
  Consider distributions  $\mu_C$ for $C \in \mathcal{C}.$ 
   Let $\mu = \sum p_C \mu_C$ with $p_C> 0$ and $\sum p_C = 1.$ 
  Suppose if $ p_C \geq \frac{\epsilon_{TV}}{8 \abs{\mathcal{C}}},$ there exists $T> 0, \epsilon_{TV}\in (0,1)$ 
s.t. with probability $ \geq 1-\frac{\tau}{10\abs{\mathcal{C}}}$ over the choice of  $U_{C, \text{sample}}$ consisting of $M \geq M_0$ i.i.d. samples from $\mu_C,$ $d_{TV}(\mathcal{L}(X_T^{\nu_{C, \text{sample}}} | {U_{C, \text{sample}}}), \mu_C) \leq \epsilon_{TV}/10$ where $\nu_{C, \text{sample}}$ is the uniform distribution over ${U_{C, \text{sample}}}.$
Then, for $M \geq  (\frac{\epsilon_{TV}}{8 \abs{\mathcal{C}}})^{-1}\min  \set{M_0, 20 \epsilon_{TV}^{-2} \log (\abs{\mathcal{C} }\tau^{-1})},$ with probability $ \geq 1-\tau$ over the choice of $U_{\text{sample}}$ consisting of $M$ i.i.d. samples from $\mu,$ $d_{TV} (\mathcal{L} (X_T^{\nu_{\text{sample}} } | U_{ \text{sample}} ), \mu )\leq \epsilon_{TV} $
\end{proposition}
\begin{proof}[Proof of \cref{prop:from one distribution to mixture of distribution modified}]
    The proof is analogous to \cref{prop:from one distribution to mixture of distribution}. We use the same setup and will spell out the differences between the two proofs. Let $\mathcal{C}' = \set{C\in \mathcal{C}: p_C \geq \frac{\epsilon_{TV}}{8\abs{\mathcal{C}}} }.$ We redefine the event $\mathcal{E}_1$ as 
    \[\forall C \in \mathcal{C}' : \abs{\frac{\abs{U_C} }{M} - p_C } \leq p_C \epsilon_{TV}/8 \] and
    $\mathcal{E}_2$ as, for $\nu_C$ be the uniform distribution over $\mu_C$
     \[ \forall C \in \mathcal{C}': d_{TV}(\mathcal{L}(X^{\nu_C}_T | U_C), \mu_C) \leq \epsilon_{TV}/10 \]
     Let $U_\emptyset = U_{\text{sample}} \setminus \bigcup_{C\in \mathcal{C}'} C$ then \[ \frac{\abs{U_\emptyset}}{M} = \frac{\sum_{C\not\in \mathcal{C}'} \abs{U_C} }{M}\leq 1 - \sum_{C\in \mathcal{C}'} p_C (1- \epsilon_{TV}/8) \leq 1 - (1- \epsilon_{TV}/8)  (1-\epsilon_{TV}/8) \leq \epsilon_{TV}/4.\]
     Suppose $\mathcal{E}_1$ and $\mathcal{E}_2$ both hold, which occur with probability $1 -\tau $ by Chernoff's inequality.
     Let $ \tilde{\mu} =  \sum_{C\in \mathcal{C}} \frac{\abs{U_C}}{M  }\mu_C.$ 
     By part 1 of \cref{prop:tv distance mixture bound}
      \begin{align*}
       d_{TV} (\mathcal{L}(X^{\nu_{\text{sample}}}_T | U_{\text{sample}}), \tilde{\mu}) &= d_{TV}\left(\sum_C \frac{\abs{U_C} }{M} \mathcal{L}(X^{\nu_C}_T | U_C)   , \sum_{C} \frac{\abs{U_C} }{M}  \mu_C \right) \\
       &\leq \sum_{C\in \mathcal{C}'} \frac{\abs{U_C}}{M} \epsilon_{TV}/10+\sum_{C\not \in\mathcal{C}' } \frac{\abs{U_C}}{M} \leq \epsilon_{TV}/10 + \epsilon_{TV}/4 \leq \epsilon_{TV}/2 
    \end{align*}
    By part 2 of \cref{prop:tv distance mixture bound}
    \[\sum_{C} \abs{\frac{\abs{U_C}}{M} - p_C} \leq \sum_{C\in \mathcal{C}'} p_C \epsilon_{TV}/8 + \sum_{C\not \in \mathcal{C}'} \max\set{\frac{\abs{U_C}}{M} 
 ,p_C} \leq \epsilon_{TV}/8 + \epsilon_{TV}/4 \]
 By triangle inequality
    \begin{align*}
        d_{TV} (\mathcal{L}(X^{\nu_{\text{sample}}}_T | U_{\text{sample}}), \mu) &\leq d_{TV} (\mathcal{L}(X^{\nu_{\text{sample}}}_T | U_{\text{sample}}), \tilde{\mu})+ d_{TV}(\tilde{\mu}, \mu) \\
        &\leq \epsilon_{TV}/2 + 3\epsilon_{TV}/8 \leq \epsilon_{TV} 
    \end{align*}
\end{proof}

\section{Additional simulations}\label{sec:additional-details}
In this section we give some additional details about the simulations in the main text as well as a few supplementary ones.

For the simulation in Figure 1 of the main text, the estimated score function was learned from data by running $3 \times 10^5$ steps of stochastic gradient descent without batching, using a fresh sample at each step with learning rate $10^{-5}$. The loss function was the vanilla score matching loss from \cite{hyvarinen2005estimation}. The neural network architecture used had a single hidden layer with tanh nonlinearity and 2048 units. The stationary distribution shown in the rightmost subfigure was computed by numerical integration of the estimated score.

For the 32-dimensional simulation in Figure~\ref{fig:32dimexample} of the main text, to train the network we used ADAM with a batch size of $256$ examples, again generated fresh each time; we used 200 batches per epoch and 300 epochs and we learned the vanilla score function using an equivalent denoising formulation as in \cite{vincent2011connection}. Figure~\ref{fig:bad32dimexample} is the same but the network was trained for only 30 epochs. 
In Figure~\ref{fig:cdexample}, we performed the same experiment as Figure~\ref{fig:32dimexample} but we used Contrastive Divergence (CD) training \cite{hinton2012practical}, which has been used by numerous experimental papers in the literature, instead of score matching as the mechanism to learn the approximate gradient.   More precisely, we used CD (again trained over 300 epochs) to learn a distribution of the form $\exp(f(x))$ where the potential $f$ was parameterized by a 8192 unit one-hidden-layer neural network with tanh activations. Once this network is learned, $\nabla f$ was used as the approximate score function since this is the score function of the learned distribution. We also observed in Figure~\ref{fig:cd-scorematching} that the score matching loss, which was explicitly trained in the other figures, is also monotonically decreasing over time under CD training. The fact that the behavior is somewhat similar under CD and score matching is morally in agreement with theoretical connections between the two observed by \cite{hyvarinen2007some}. Note that in all three of these figures, the same random seeds were used so that colored trajectories will correspond to each other.

\begin{figure}
    \centering
    \begin{subfigure}{0.32\textwidth}
    \includegraphics[width=\textwidth]{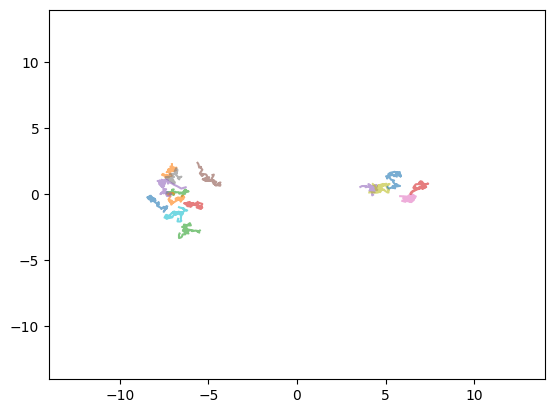}
    \caption{$T = 300$}
    \end{subfigure}
    \hfill
    \begin{subfigure}{0.32\textwidth}
    \includegraphics[width=\textwidth]{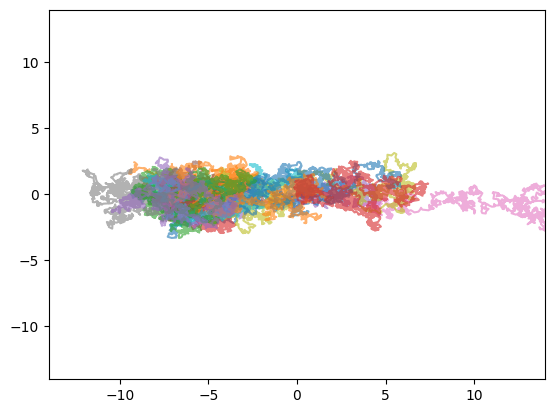}
    \caption{$T = 12000$}
    \end{subfigure}
    \hfill
    \begin{subfigure}{0.32\textwidth}
    \includegraphics[width=\textwidth]{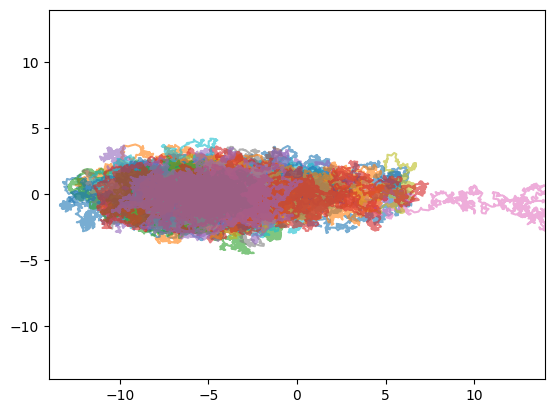}
    \caption{$T = 120000$ }
    \end{subfigure}
    \caption{Failure to approximate the ground truth with a less accurate score function. This is exact same simulation as Figure~\ref{fig:32dimexample}, except that the network estimating the score function was trained for 30 rather than 300 epochs. 
    We see that while the short time evolution is similar, at moderate times (Figure (b)) the output of the dynamics have drifted away from the true distribution due to accumulation of errors and in particular one trajectory has escaped far right of the rightmost component.}
    \label{fig:bad32dimexample}
\end{figure}

\begin{figure}
    \centering
    \begin{subfigure}{0.32\textwidth}
    \includegraphics[width=\textwidth]{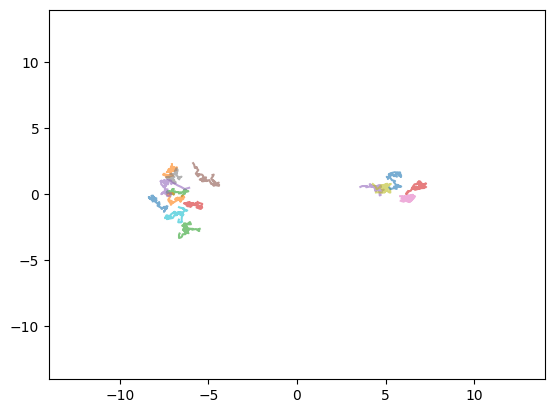}
    \caption{$T = 300$}
    \end{subfigure}
    \hfill
    \begin{subfigure}{0.32\textwidth}
    \includegraphics[width=\textwidth]{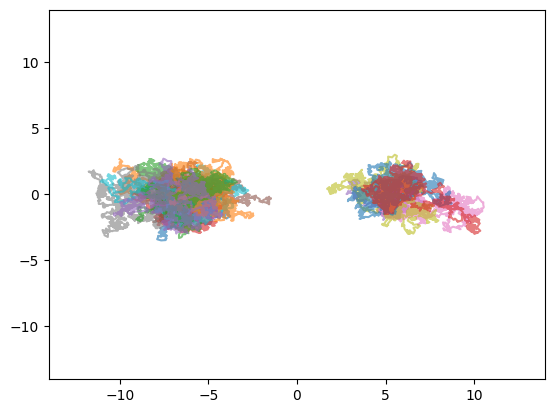}
    \caption{$T = 12000$}
    \end{subfigure}
    \hfill
    \begin{subfigure}{0.32\textwidth}
    \includegraphics[width=\textwidth]{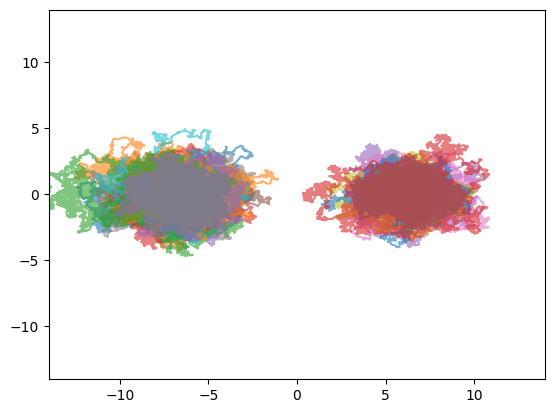}
    \caption{$T = 120000$ }
    \end{subfigure}
    \caption{Variant of Figure~\ref{fig:32dimexample} where the approximate score function is learned via Contrastive Divergence (CD) instead of directly trying to match the score function. We used the most basic/efficient version of CD, with only a single step of Langevin dynamics, and we used a larger step size of $0.05$ when sampling in the training loop to compensate for only taking a single step.  Qualitatively, the behavior seems similar to Figure~\ref{fig:32dimexample}; at large times, while none of these particular trajectories crossed between components, one trajectory escaped into a low-density region left of the leftmost component.}
    \label{fig:cdexample}
\end{figure}

\begin{figure}
    \centering
    \includegraphics[scale=0.55]{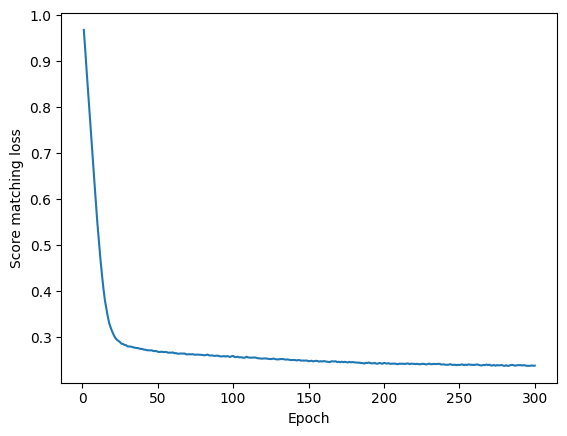}
    \caption{Score matching training loss (precisely the same loss used to train the models in Figures~\ref{fig:32dimexample} and \ref{fig:bad32dimexample}) curve for the CD-trained model in Figure~\ref{fig:32dimexample}. Although the score matching loss is not being explicitly optimized, we see it goes down monotonically over the epochs of CD training nonetheless.}
    \label{fig:cd-scorematching}
\end{figure}
\end{document}